\documentclass[11pt,twoside]{article}

\usepackage[utf8]{inputenc}

\usepackage{fullpage}

\usepackage{epsf}
\usepackage{fancyhdr}
\usepackage{graphics}
\usepackage{graphicx}
\usepackage{psfrag}
\usepackage{algorithm}
\usepackage{algorithmic}

\usepackage{color}

\usepackage{amsthm}
\usepackage{amsfonts}
\usepackage{amsmath}
\usepackage{amssymb,bbm}
\usepackage{natbib}
\usepackage[colorlinks,linkcolor=magenta,citecolor=blue]{hyperref}
\usepackage{url}
\usepackage{nicefrac}

\usepackage{chngpage}

 \usepackage{tabularx}%

\usepackage{enumitem}
\usepackage{booktabs}
\usepackage{pbox}

\usepackage{caption,subcaption}

\usepackage{mathtools}

\usepackage{fullpage}

\newcommand{\bracket}[1]{\left[ #1 \right]}
\newcommand{\parenth}[1]{\left( #1 \right)}

\newcommand{\braces}[1]{\left\{ #1 \right \}}
\newcommand{\abss}[1]{\left| #1 \right |}
\newcommand{\angles}[1]{\left\langle #1 \right \rangle}

\newcommand{\ceil}[1]{\left\lceil #1 \right\rceil}

\newcommand{\stepsize}{\eta}

\newcommand{\grad}{\nabla}

\newcommand{\Rspace}{\ensuremath{\mathbb{R}}}

\newcommand{\vecnorm}[2]{\left\| #1\right\|_{#2}}

\newcommand{\inprod}[2]{\ensuremath{\langle #1 , \, #2 \rangle}}

\newcommand{\Exs}{\ensuremath{{\mathbb{E}}}}
\newcommand{\Prob}{\ensuremath{{\mathbb{P}}}}

\newtheoremstyle{named}{}{}{\itshape}{}{\bfseries}{.}{.5em}{\thmnote{#3's }#1}
\theoremstyle{named}

\theoremstyle{plain}

\newtheorem{theorem}{Theorem}

\newtheorem{lemma}{Lemma}

\newtheorem{corollary}{Corollary}
\newtheorem{definition}{Definition}

\newtheorem{assumption}{Assumption}

\newlength{\widebarargwidth}
\newlength{\widebarargheight}
\newlength{\widebarargdepth}

\makeatletter
\long\def\@makecaption#1#2{
        \vskip 0.8ex
        \setbox\@tempboxa\hbox{\small {\bf #1:} #2}
        \parindent 1.5em  
        \dimen0=\hsize
        \advance\dimen0 by -3em
        \ifdim \wd\@tempboxa >\dimen0
                \hbox to \hsize{
                        \parindent 0em
                        \hfil
                        \parbox{\dimen0}{\def\baselinestretch{0.96}\small
                                {\bf #1.} #2
                                }
                        \hfil}
        \else \hbox to \hsize{\hfil \box\@tempboxa \hfil}
        \fi
        }
\makeatother

\long\def\comment#1{}
\definecolor{battleshipgrey}{rgb}{0.52, 0.52, 0.51}
\definecolor{darkgray}{rgb}{0.66, 0.66, 0.66}
\definecolor{darkgreen}{rgb}{0.0, 0.2, 0.13}
\definecolor{darkspringgreen}{rgb}{0.09, 0.45, 0.27}
\definecolor{dukeblue}{rgb}{0.0, 0.0, 0.61}
\definecolor{olivedrab7}{rgb}{0.24, 0.2, 0.12}
\definecolor{darkblue}{rgb}{0.0, 0.0, 0.55}
\definecolor{darkscarlet}{rgb}{0.34, 0.01, 0.1}
\definecolor{candyapplered}{rgb}{1.0, 0.03, 0.0}
\definecolor{ao(english)}{rgb}{0.0, 0.5, 0.0}
\definecolor{applegreen}{rgb}{0.55, 0.71, 0.0}

\newcommand{\trueMat}{\ensuremath{\mathbf{X}^*}}
\newcommand{\fitMat}{\ensuremath{\mathbf{F}}}
\newcommand{\senseMat}{\ensuremath{\mathbf{A}}}
\newcommand{\matA}{\ensuremath{\mathbf{A}}}
\newcommand{\matB}{\ensuremath{\mathbf{B}}}

\newcommand{\Umat}{\ensuremath{\mathbf{U}}}
\newcommand{\Vmat}{\ensuremath{\mathbf{V}}}
\newcommand{\Ucoeff}{\ensuremath{\mathbf{S}}}
\newcommand{\Vcoeff}{\ensuremath{\mathbf{T}}}
\newcommand{\trueUcoeff}{\ensuremath{\mathbf{D_S^*}}}
\newcommand{\trueVcoeff}{\ensuremath{\mathbf{D_T^*}}}
\newcommand{\objFun}{\ensuremath{\mathcal{L}}}
\newcommand{\Id}{\ensuremath{\mathbf{I}}}

\newcommand{\popUup}{\ensuremath{\mathcal{M}_{\Ucoeff}}}
\newcommand{\popVup}{\ensuremath{\mathcal{M}_{\Vcoeff}}}

\newcommand{\bA}{\ensuremath{\mathbf{A}}}
\newcommand{\bB}{\ensuremath{\mathbf{B}}}
\newcommand{\bF}{\ensuremath{\mathbf{F}}}
\newcommand{\bG}{\ensuremath{\mathbf{G}}}
\newcommand{\bU}{\ensuremath{\mathbf{U}}}
\newcommand{\bu}{\ensuremath{\mathbf{u}}}
\newcommand{\bX}{\ensuremath{\mathbf{X}}}
\newcommand{\bx}{\ensuremath{\mathbf{x}}}
\newcommand{\bz}{\ensuremath{\mathbf{z}}}
\newcommand{\Vcal}{\ensuremath{\mathcal{V}}}

\begin{document}
\begin{center}
{\bf{\LARGE{On the computational and statistical complexity of over-parameterized matrix sensing}}} 

\vspace*{.2in}
{\large{
\begin{tabular}{cccc}
Jiacheng Zhuo$^{\dagger}$ & Jeongyeol Kwon$^{\flat}$ & Nhat Ho$^{\diamond}$ & Constantine Caramanis$^{\flat}$ \\
\end{tabular}
}}

\vspace*{.2in}
\begin{tabular}{c}
Department of Computer Science, University of Texas at Austin$^\dagger$, \\
Department of Electrical and Computer Engineering, University of Texas at Austin$^\flat$ \\
Department of Statistics and Data Sciences, University of Texas at Austin$^\diamond$
\end{tabular}

\today

\vspace*{.2in}

\begin{abstract}
We consider solving the low rank matrix sensing problem with Factorized Gradient Descend (FGD) method when the true rank is unknown and over-specified, which we refer to as over-parameterized matrix sensing.
If the ground truth signal $\bX^* \in \mathbb{R}^{d*d}$ is of rank $r$, but we try to recover it using $\fitMat \fitMat^\top$ where $\fitMat \in \mathbb{R}^{d*k}$ and $k>r$, the existing statistical analysis falls short, due to a flat local curvature of the loss function around the global maxima.
By decomposing the factorized matrix $\fitMat$ into separate column spaces to capture the effect of extra ranks, we show that $\vecnorm{\fitMat_t \fitMat_t - \trueMat}{F}^2$ converges to a statistical error of $\tilde{\mathcal{O}} \parenth{k d \sigma^2/n}$ after $\tilde{\mathcal{O}}(\frac{\sigma_{r}}{\sigma}\sqrt{\frac{n}{d}})$ number of iterations where $\fitMat_t$ is the output of FGD after $t$ iterations, $\sigma^2$ is the variance of the observation noise, $\sigma_{r}$ is the $r$-th largest eigenvalue of $\trueMat$, and $n$ is the number of sample.
Our results, therefore, offer a comprehensive picture of the statistical and computational complexity of FGD for the over-parameterized matrix sensing problem. 
\end{abstract}
\end{center}

\section{Introduction} 
\label{sec:introduction}

We consider the low rank matrix sensing problem: we are given $n$ i.i.d. observations $\{\matA_i, y_i\}_{i=1}^n$ from the data generating model $y_i = \angles{\matA_i, \trueMat} + \epsilon_i$, where $\matA_i \in \mathbb{R}^{d*d}$ is a symmetric random sensing matrix, $\trueMat \in \mathbb{R}^{d*d}$ is the target rank $r$ symmetric matrix we want to recover, and $\epsilon_i$ is a zero-mean sub-Gaussian noise with variance proxy $\sigma^2$. The low rank matrix sensing problem has found applications in various scenarios, such as multi-task regression, vector auto-regressive process, image processing, metric embedding, quantum tomography, and so on~\citep{candes2011tight, negahban2011estimation, recht2010guaranteed, jain2013low, gross2010quantum, candes2011robust, waters2011sparcs, kalev2015quantum}.
One common approach to recover a low-rank matrix $\bX \in \mathbb{R}^{d * d}$ is to solve the following optimization problem:
\begin{align}
\label{formulation:convex}
    \underset{\bX: \bX \succeq 0, \text{rank}(\bX) \leq k}{\arg \min} \frac{1}{4n} \sum_{i=1}^n \left( y_i - \angles{\matA_i, \bX}  \right)^2,
\end{align}
where $k$ is a chosen rank based on domain knowledge of the data. This problem can be solved by relaxing the rank constraint to nuclear norm constraint \citep{recht2010guaranteed, candes2011tight}.
However for computational benefits, it is common to reformulate this as a non-convex problem by introducing $\bF\in \mathbb{R}^{d*k}$ such that $\bX = \bF \bF^\top$ and solving the transformed problem \citep{bhojanapalli2016dropping, chen2015fast, jain2013low, hardt2014understanding}
\begin{align}
\label{formulation:non-convex}
    \underset{\bF: \bF \in \mathbb{R}^{d*k}}{\arg \min} \quad \objFun(\fitMat) := \frac{1}{4n} \sum_{i=1}^n \left( y_i - \angles{\matA_i, \bF \bF^\top}  \right)^2.
\end{align}
Solving this formulation directly with gradient descent method on the matrix $\bF$ is usually referred to as the Factorized Gradient Descent (FGD) method, which is given by:
\begin{align}
\label{eqn:factored-gradient-method}
    \fitMat_{t+1} = \fitMat_t - \eta \bG_t^n, \quad \text{where} \quad
    \mathbf{G}_t^n = \grad \objFun(\fitMat_{t}) = \frac{1}{n} \sum_{i=1}^n \left(  \angles{\matA_i, \fitMat_t \fitMat_t^{\top} - y_i} \right)\matA_i  \fitMat_t,
\end{align}
where $\eta$ is the step size and $\bG_t^n$ denotes the gradient evaluated at iteration $t$ with $n$ i.i.d. samples.

When the specified rank $k$ matches the ground truth rank $r$, namely, the true rank $r$ is known, FGD converges linearly to a statistical error \citep{chen2015fast}, and the statistical error is minimax optimal up to log factors \citep{candes2011tight}.
However, in the real world applications, it is often a big challenge to correctly identify the true rank $r$, and hence the practitioners tend to over-specify the rank. When the rank is over-specified (i.e. $k>r$), we refer to that setting as the \emph{over-parameterized matrix sensing} problem. 





The over-parameterized matrix sensing comes with many challenges, and to the best of our knowledge, none of the existing works offer a complete understanding about the computational and statistical performance of FGD under this setting.
First and foremost, we are faced a degenerate Hessian around the global maxima caused by the over-specification of the rank. 
Hence previous works with known rank settings~\citep{bhojanapalli2016dropping, zheng2016convergence, tu2016low} are no longer applicable since they rely on local strong convexity around the global maxima.
The analysis of~\citet{chen2015fast} is also void, because with over-parameterization the ratio of the first and the $k$-th eigenvalue of $\trueMat$ is infinity. 
\citet{li2018algorithmic} focus on the implicit regularization effect with early stopping, and their analysis is limited to the setting where there is no observation noise ($\epsilon_i = 0$), $k = d$ and they can only guarantee recovery within a lower and upper bounded iteration range (as in their Theorem 1).
In summary, despite the current progress on the matrix sensing problem, the following questions remain unclear: \\
\emph{If we solve the over-parameterized matrix sensing problem with FGD, (1) what is the achievable statistical error? (2) and how fast we can recover a target matrix $X^*$?} 

\vspace{0.5 em}
\noindent
\textbf{Contribution.}
This paper offers a comprehensive analysis of over-parameterized low-rank matrix sensing with the FGD method. We show that $\vecnorm{\fitMat_t\fitMat_t^\top - \trueMat}{F}^2$ converges to a final statistical error of $\tilde{\mathcal{O}} \parenth{{kd \sigma^2 } /n}$ after $\tilde{\mathcal{O}}(\frac{\sigma_{r}}{\sigma}\sqrt{\frac{n}{d}})$ number of iterations where $\sigma_r$ and $\sigma$ respectively the $r$-th largest eigenvalue of $\trueMat$ and the standard deviation of the observation noise. It is different from the computational and statistical behavior of FGD when the true rank is known, namely, the FGD converges to a radius of convergence $\tilde{\mathcal{O}}(r d \sigma^2/ n)$ around the true matrix $\trueMat$ after $\mathcal{O}(\log(\frac{\sigma_{r}}{\sigma_{1}} \cdot \frac{n}{d}))$ iterates~\citep{chen2015fast} where $\sigma_{1}$ is the largest eigenvalue of $\trueMat$. Since we assume no {\it a priori} knowledge of true rank $r$, the statistical error $\tilde{\mathcal{O}} \parenth{{kd \sigma^2 } /n}$ is also minimax optimal up to logarithmic factors \citep{candes2011tight}. Furthermore, the number of iterations $\tilde{\mathcal{O}}(\frac{\sigma_{r}}{\sigma}\sqrt{\frac{n}{d}})$ is needed in the over-parameterized setting as the local curvature of the loss function~\eqref{formulation:convex} around the global maxima is not quadratic and therefore the FGD only converges sub-linearly to the global maxima; see the simulations in Figure~\ref{fig:motivation-sim} for an illustration. Finally, when $\sigma=0$, i.e., in the noiseless case, we can guarantee the exact recovery similar to when we correctly specify the rank~\citep{chen2015fast}.

\subsection{Related Work}

\paragraph{Works related to Matrix Sensing.}
Early works on matrix sensing often perform a semidefinite programming (SDP) relaxation, and replace the nonconvex rank constraint with a convex constraint based on the trace norm or nuclear norm; see for example~\citep{candes2011tight, recht2010guaranteed, negahban2011estimation, chen2013low} and the references therein.
\citet{candes2011tight} show that for any estimator $\hat{\bX}$ based on $\braces{\matA_i, y_i}_{i=1}^{n}$ observations, $\vecnorm{\hat{\bX} - \bX^*}{F}^2 \geq \frac{dr}{n}\sigma^2$,
where $\bX^*$ is the ground truth rank $r$ matrix that we want to recover, and $\sigma$ is the standard deviation of the (sub)-Gaussian observation noise (see Section \ref{section:notations} for details).
This convex relaxation approach is nearly optimal in this sense.
Although we can theoretically solve this convex problem in polynomial time, the computational cost is often prohibitively high for large scale problems.
For example, if we solve this SDP problem with the classical interior point method, the computational cost is roughly $\mathcal{O}(d^6)$ \citep{boyd2004convex, chen2015fast}
Although recently some tailored algorithms \citep{zheng2015convergent, tu2016low} are developed to solve this convex problem, their computational complexity is at least $\mathcal{O}(d^3)$ since this SDP involves multiplication of two matrices in $\mathbb{R}^{d*d}$.
This computational overhead motivates the study of FGD method.
The low rank matrix sensing problem is tightly connected to the low rank matrix completion problem, since they have the same population update when solved by (factorized) gradient method, and they can often be analyzed by very similar techniques \citep{negahban2012restricted, koltchinskii2011nuclear, chi2019nonconvex}.

\paragraph{Works related to FGD.}
The idea of factorizing the low rank matrix dates back to~\cite{burer2003nonlinear, burer2005local}. 
\citet{bhojanapalli2016dropping} characterize the computational convergence behavior of FGD method for general convex and strongly convex function using the restricted strong convexity argument.
However, such analysis cannot be converted into statistical analysis.
\citet{chen2015fast} offer a general theoretical framework for understanding FGD method from both computational and statistical perspective. Specifically, they show that with suitable initialization, FGD converges geometrically up to a statistical precision. However, their analysis only works when we know the ground truth rank ($k = r$).

In this work we focus on local convergence as this is the crux in statistical analysis (see~\citep{chen2015fast}).
Initialization condition can be achieved via spectral methods (see~\citep{bhojanapalli2016dropping, tu2016low, zheng2016convergence}).
Moreover, the works by \cite{bhojanapalli2016global}, \cite{ge2016matrix}, and~\cite{Zhang_Sharp_2019} show that reformulation~\eqref{formulation:non-convex} does not have any spurious local minima from optimization's perspective, indicating that it is possible to extend our analysis to random initialization.

Recently,~\cite{li2018algorithmic} look into the implicit regularization effect in the learning of over-parameterized matrix factorization with FGD.
They show that if there is no observation noise ($\epsilon_i = 0$) and $k = d$, FGD tends to first recover the majority part of the true signal (that is of rank $r$) due to the implicit regularization effect of the FGD method.
However their analysis can only address the noiseless case, and can not be extended to the more realistic setting when the observation is noisy, i.e., $\epsilon_i \neq 0$.
Moreover, they only guarantee recovery within an iteration lower bound and upper bound (e.g., as in the Theorem 1 in \citet{li2018algorithmic}, the number of iterations to reach the target accuracy has an upper bound and lower bound), which is not in line with the common notion of convergence and statistical rate.
We focus on the statistical rate, which means we want to understand the algorithm behavior if run the algorithm for infinitely long.
(Further discussion can be found in Section \ref{sec:discussion}).


\paragraph{Localized analysis for degenerate landscape.} When the curvature around the local optimum degenerates, first-order methods such as gradient descent slow down due to vanishing gradients as the estimator gets closer to the local optimum. This phenomenon is reported in various optimization problems with degenerate landscapes in weakly separated mixture of distributions \citep{Raaz_Ho_Koulik_2018, kwon2020minimax}. We can observe the same phenomenon when the rank is over-specified for low-rank matrix factorization problems.


The localization technique is a powerful analysis tool to handle degenerate landscapes with a tight statistical rate. This technique has been used widely in the empirical process theory literature~\citep{Vaart_Wellner_2000}. We now see how the localization argument can be applied for a low-rank matrix sensing when we over-specify the rank. In result we obtain a tight statistical rate of FGD which matches the known information-theoretical lower bound for this problem even if we over-specify the rank.

\subsection{Motivating Simulations}
\label{sec:motivating-sim}

\begin{figure}[t]
\centering
\begin{subfigure}[t]{0.48\textwidth}
\includegraphics[width=1\textwidth]{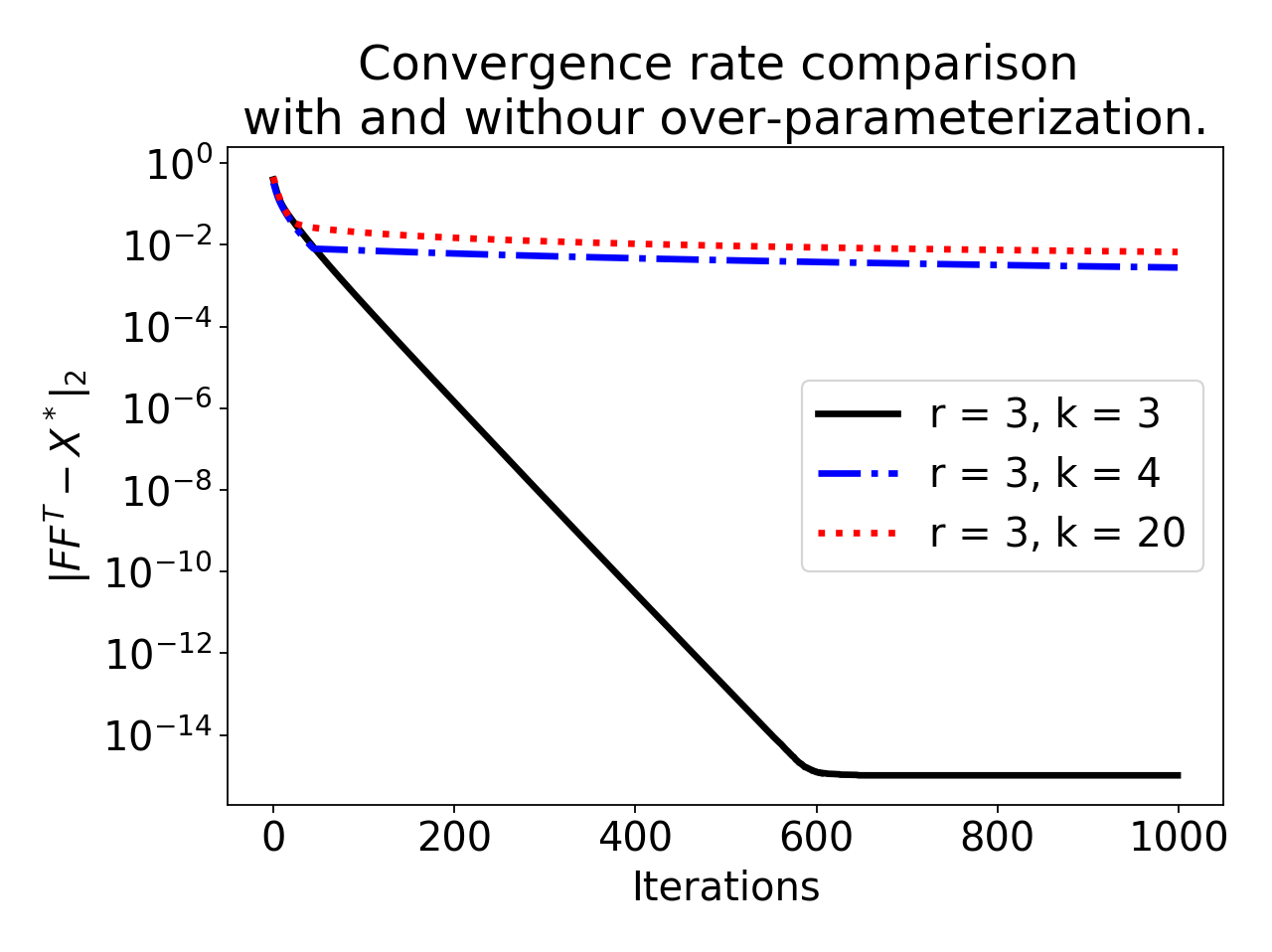}
\caption{}
\label{fig:motivation-sim-a}
\end{subfigure}
\begin{subfigure}[t]{0.48\textwidth}
\includegraphics[width=1\textwidth]{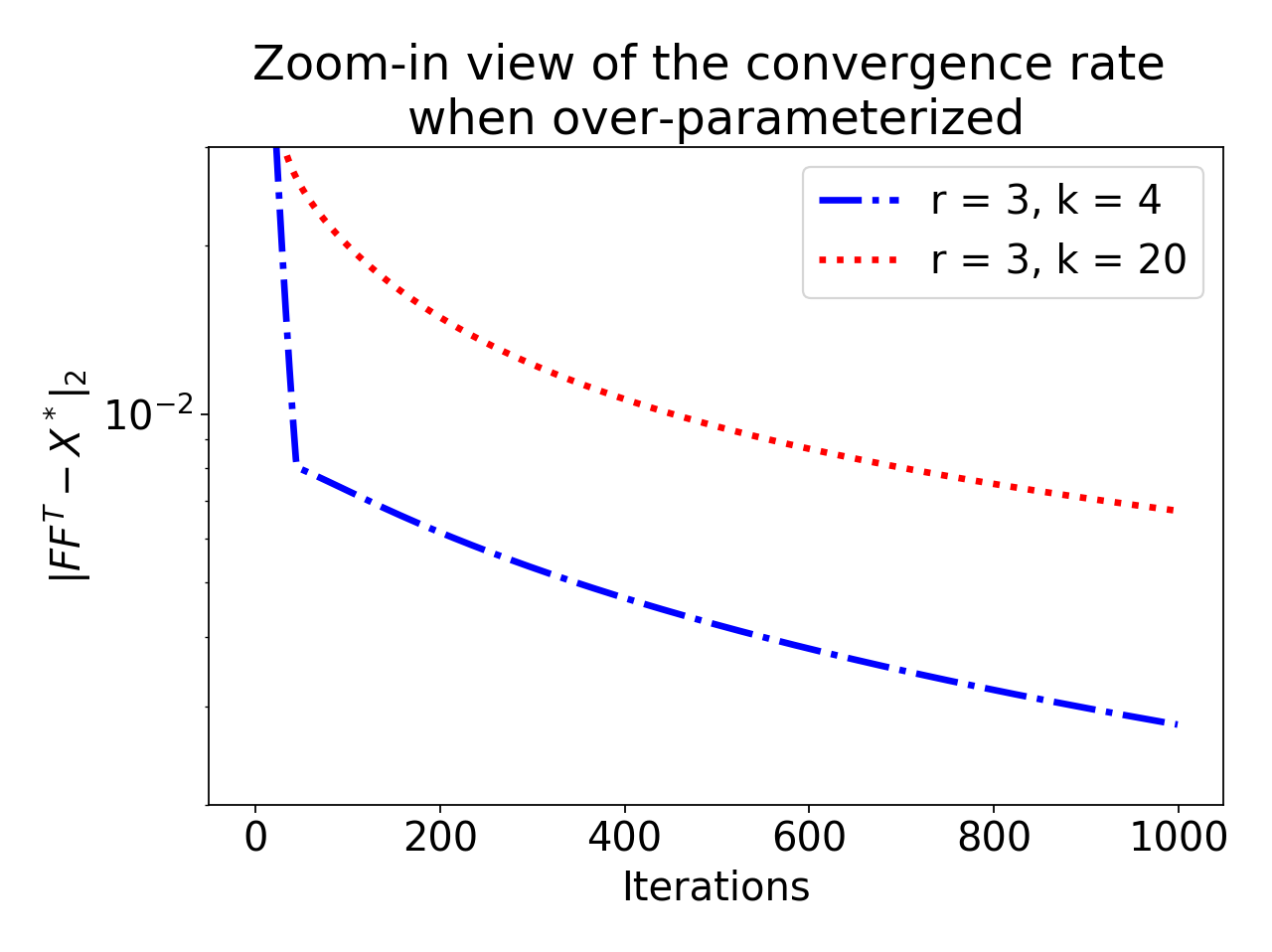}
\caption{}
\label{fig:motivation-sim-b}
\end{subfigure}
\caption{
\textbf{The motivating simulations.}
(a) When we correctly specify the rank (i.e., $k=r=3$), the FGD method converges geometrically towards machine precision. But when $k>r$, FGD only converges sub-linearly.
(b) A zoom-in view of the convergence rate shows that, FGD might first converge geometrically, and then converge sub-linearly.
}
\label{fig:motivation-sim}
\end{figure}


In the simulations, we consider the dimension $d = 20$,  the true rank $r = 3$, and the number of samples $n = 200$.
We first generate random orthonormal matrices $\Umat$ and $\Vmat$ such that the union of their column spaces is $\mathbb{R}^d$.
We set $\trueUcoeff$ to be a diagonal matrix, with its $(1,1), (2,2), (3,3)$ entries be $1, 0.9, 0.8$ respectively, and zero elsewhere.
Hence $\trueMat = \Umat \trueUcoeff \Umat^\top$.
The upper triangle entries of the sensing matrices $\bA_i$ are sampled from standard Gaussian distribution, and we fill the lower triangle entries accordingly such that $\bA_i$ are symmetric.
We further assume that there is no observation noise, so that we have a better understanding of the convergence behavior of the algorithm.

Let $\braces{\fitMat_t}_t$ be the sequence generated by the FGD method as in equation~\eqref{eqn:factored-gradient-method} with $\eta = 0.1$.
The simulation results are shown in Figure \ref{fig:motivation-sim}.
When we correctly specify the rank (i.e. $k=r=3$), the FGD method converge geometrically towards machine precision.
However, even if we increase the specified rank by $1$, FGD will end up with a much slower convergence rate.
A zoom-in view of the convergence rate shows that, FGD might first converge geometrically, and then converge sub-linearly.
This phenomenon is not captured by the recent works about FGD \citep{li2018algorithmic, chen2015fast}.
What exactly is the convergence rate? And what about the statistical error? These are the questions that we want to answer in this work.
\subsection{Organization}
The remainder of the paper is organized as follows. In Section~\ref{sec:main_result}, we present the convergence rate of the FGD iterates under the over-parameterized matrix sensing setting. Then, we present the proof sketch of the results in Section~\ref{sec:proof_main_result}. The detailed proofs of the main results are deferred to the Appendices while we conclude the paper with a few discussions in Section~\ref{sec:discussion}.
\subsection{Notations}
\label{section:notations}
In the paper, 
we use bold lower case letters to represent vectors, such as $\bx$, and bold upper case letters to represent matrices, such as $\bX$.
When $\bX$ is a matrix, we use $X_{ij}$ to represent the element on the $i$-th row and $j$-th column of $\bX$, unless otherwise specified.
We use $\angles{\cdot, \cdot}$ for matrix inner product. For example $\angles{\matA, \bX} = \sum_{ij} A_{ij} X_{ij} $.
We denote $\ceil{ x}$ as the smallest integer greater than or equal to $x$ for any $x \in \Rspace$.
We write $\bA \succ \bB$ (respectively $\bA \succeq \bB$) if $\bA - \bB$ is positive definite (respectively positive semidefinite) for square matrices  $\bA$ and $\bB$.
We write $\braces{\bA_i}_{i=1}^t$ to represent the sequence $\braces{\bA_1, \bA_2, ..., \bA_t}$.
We also use the short hand $\braces{\bA_i}_{i}$ to represent $\braces{\bA_1, \bA_2, ...}$
We use $\sigma_1$ and $\sigma_r$ to denote the first eigenvalue and the $r$-th eigenvalue of $\trueMat$ respectively, which is the ground truth rank $r$ matrix that we want to recover. And we use $\kappa$ to denote the conditional number: $\kappa := \sigma_1 / \sigma_r$.

We also use the standard asymptotic complexity notation.
Specifically, $f(x) = \mathcal{O}(g(x))$ implies $\abss{f(x)} \leq C \abss{g(x)}$ for some constant $C$ and for large enough $x$, $f(x) = \Omega(g(x))$ implies $\abss{f(x)} \geq C \abss{g(x)}$ for some constant $C$ and for large enough $x$, and $f(x) = \Theta(g(x))$ implies $C_1\abss{g(x)} \leq \abss{f(x)} \leq C_2\abss{g(x)}$ for some constant $C_1, C_2$ and for large enough $x$.
When  $\log$ factors are omitted, we use $\tilde{\mathcal{O}}$, $\tilde{\Omega}$, $\tilde{\Theta}$ to represent ${\mathcal{O}}$, ${\Omega}$, ${\Theta}$ respectively.
\begin{definition}
\textbf{(Sub-Gaussian Random Variable).} We call a random variable $X$ with mean $\mu$ sub-Gaussian with variance proxy $\sigma > 0$ if $\forall \lambda \in \mathbb{R}$,
$
    \Exs \bracket{\exp\parenth{\lambda\parenth{X - \mu}}} \leq e^{\parenth{\sigma^2\lambda^2 / 2}}
$.
\end{definition}
\begin{definition}
\textbf{(Sub-Gaussian Sensing Matrix).} We call a matrix $\bA$ a sub-Gaussian sensing matrix if it is sampled as follow:
each upper triangle entry ($i < j$) $A_{ij}$ is sampled i.i.d. from a zero-mean sub-Gaussian distribution with variance proxy $1$, each lower triangle entry ($i > j$) $A_{ij} = A_{ji}$, and the diagonal entries are sample from i.i.d. from a zero-mean sub-Gaussian distribution with variance proxy $1$.
\end{definition}


\section{Main Result}
\label{sec:main_result}


Before we present our main result, we formally introduce the decomposition notation for $\trueMat$.
Let the eigen-decomposition of $\trueMat$ (eigenvalues ordered by the absolute values) be given by
\begin{align*}
    \mathbf{X}^* = \left[\Umat \;\; \Vmat \right] \begin{bmatrix} \trueUcoeff & 0 \\ 0 & \trueVcoeff \end{bmatrix} \left[\Umat \;\; \Vmat \right]^\top = \Umat \trueUcoeff \Umat^\top + \Vmat \trueVcoeff \Vmat^\top,
\end{align*}
where $\Umat \in \mathbb{R}^{d*r}$, $\Vmat \in \mathbb{R}^{d*(d-r)}$, $\trueUcoeff \in \mathbb{R}^{r *r}$, $\trueVcoeff \in \mathbb{R}^{(d-r)*(d-r)}$.
Without loss of generality we assume that the both $\Umat$ and $\Vmat$ are orthonormal and $\Umat^\top \Vmat = 0$ (i.e. $\Umat$ and $\Vmat$ together span the entire $\mathbb{R}^d$).
Denote $\sigma_1$ be the largest value in  $\trueUcoeff$, $\sigma_r$ be the smallest value in $\trueUcoeff$, and $\sigma_{r+1}$ be the largest value in $\trueVcoeff$.
Since we assume $\trueMat$ is of approximately rank $r$, there is a non-trivial gap between $\sigma_r$ and $\sigma_{r+1}$. In this section, we assume that $\sigma_{r+1} \ll \sigma_r$. 
Since the union of the column space of $\Umat$ and $\Vmat$ spans the entire $\mathbb{R}^d$, then for any $\fitMat_t \in \mathbb{R}^{d*k}$, there exits  matrices $\Ucoeff_{t} \in \mathbb{R}^{r*k}$ and $\Vcoeff_{t}\in \mathbb{R}^{(d-r)*k}$ such that 
\begin{align*}
  	\fitMat_{t} = \Umat \Ucoeff_{t} + \Vmat \Vcoeff_{t}.
\end{align*} 
As $t$ goes to infinity, we hope that $\Ucoeff_{t} \Ucoeff_{t}^\top$ converges to $\trueUcoeff$, $\Vcoeff_{t} \Vcoeff_{t}^\top$ converges to $\trueVcoeff$, and $\Ucoeff_t \Vcoeff_t^\top$ and $\Vcoeff_t \Ucoeff_t^\top$ converges to zero, and hence $\fitMat_t \fitMat_t^\top = \Umat \Ucoeff_{t} \Ucoeff_{t}^\top \Umat^\top + \Vmat \Vcoeff_{t} \Vcoeff_{t}^\top \Vmat^\top + \Umat \Ucoeff_t \Vcoeff_t^\top \Vmat^\top + \Vmat\Vcoeff_t \Ucoeff_t^\top \Umat^\top$ converges to $\trueMat$.

We introduce the decomposition and study the convergence of $\Ucoeff_{t} \Ucoeff_{t}^\top$, $\Vcoeff_{t} \Vcoeff_{t}^\top$, and $\Ucoeff_t \Vcoeff_t^\top$ separately.
This decomposition technique is essential, since we can then bypass some technical difficulties when we over-specify the rank. For example we do not have to establish the uniqueness (up to rotational ambiguity) of the optimal solution as in the Lemma 1 in \cite{chen2015fast}.
Moreover, this gives more insights about which part is the computational and/or statistical bottleneck. As we will see shortly (both in Theorem~\ref{theorem:sample-convergence-T-small} and Lemma \ref{lemma:sample-contraction-T-small}), it is the convergence of $\braces{\vecnorm{\Vcoeff_t \Vcoeff_t^\top - \trueVcoeff}{2}}_t$ that slows down the entire process of the convergence.
Similar decomposition technique is also employed in the work of \cite{li2018algorithmic}.

Here we focus on the local convergence of FGD method within the following basin of attraction:
\begin{assumption}
\label{assumption:init}
\textbf{(Initialization assumption)}
\begin{align}
    \vecnorm{\fitMat_0 \fitMat_0^\top - \trueMat}{2} \leq \rho \sigma_r, \quad \text{for} \quad \rho \leq 0.07.
\end{align}
\end{assumption}
Note that $0.07$ is a universal constant and is chosen for the ease of presentation.
Note that one can use spectrum method to achieve this initialization \citep{chen2015fast, bhojanapalli2016dropping, tu2016low}. Connecting the initialization condition to our decomposition strategy, we need to control $\max \left\{ \vecnorm{\trueVcoeff - \Vcoeff_0 \Vcoeff_0^\top}{2} ,  \vecnorm{\trueUcoeff - \Ucoeff_0 \Ucoeff_0^\top}{2} , \vecnorm{ \Ucoeff_0 \Vcoeff_0^\top}{2} \right\}$ in our analysis.
The following Lemma establish the connection between what we need in the analysis and Assumption \ref{assumption:init}.
\begin{lemma}
\label{lemma:init}
If $\vecnorm{\fitMat_0 \fitMat_0^\top - \trueMat}{2} \le 0.7 \rho\sigma_r$, then 
$$
    \max \left\{ \vecnorm{\trueVcoeff - \Vcoeff_0 \Vcoeff_0^\top}{2} ,  \vecnorm{\trueUcoeff - \Ucoeff_0 \Ucoeff_0^\top}{2} , \vecnorm{ \Ucoeff_0 \Vcoeff_0^\top}{2} \right\} \leq \rho \sigma_r.
$$
\end{lemma}
\noindent
We leave the proof of Lemma~\ref{lemma:init} to Appendix~\ref{appendix:proof-init}.
Now we are ready to present our main result.


\begin{theorem}
\label{theorem:sample-convergence-T-small}
\textbf{(Main result)}
Assume the following settings: (1) $\vecnorm{\trueVcoeff}{2} < \sqrt{\frac{d \log d}{n}} \sigma $;
(2) we have good initialization as in Assumption~\ref{assumption:init}; 
(3) the sample size $n > C_1 k \kappa^2 d \log^3 d \cdot \max(1, \sigma^2/\sigma_r^2)$ for some universal constant $C_1$;
(4)  the step size $\stepsize = \frac{1}{100\sigma_1}$,
(5) $\bA_i$s are sub-Gaussian sensing matrices.
Let $\braces{\fitMat_t}_t$ be the sequence generated by the FGD algorithm as in Equation \ref{eqn:factored-gradient-method}.
Then, the following holds:
\begin{enumerate}
    \item[(a)] After $t >  \ceil{2\log \frac{\sigma_r}{\epsilon_{comp}}}$ steps,
    $\max \braces{ \vecnorm{\Ucoeff_{t}\Ucoeff_{t}^\top - \trueUcoeff}{2}, \vecnorm{\Ucoeff_t \Vcoeff_t^{\top}}{2}} < C \epsilon_{comp}$ for some universal constant $C$, where $\epsilon_{comp} = \sqrt{\frac{k \kappa^2 d \log d}{n} }  \sigma_r$.
    \item[(b)] After $t \geq \Theta\parenth{\frac{\sigma_1}{\epsilon_{stat}}}$ steps,
    $\max \braces{ \vecnorm{\Ucoeff_{t}\Ucoeff_{t}^\top - \trueUcoeff}{2}, \vecnorm{\Ucoeff_t \Vcoeff_t^{\top}}{2}, \vecnorm{\Vcoeff_t \Vcoeff_t^\top - \trueVcoeff}{2}} < C_1 \epsilon_{stat}$, and
    $\vecnorm{\fitMat_t \fitMat_t^\top - \trueMat}{2} \leq C_2 \epsilon_{stat}$ for some constants $C_1$ and $C_2$, where $\epsilon_{stat} : = \kappa \sqrt{\frac{ d \log d }{n}}\sigma$.
\end{enumerate}
\end{theorem}
\noindent
The proof of Theorem~\ref{theorem:sample-convergence-T-small} is in Appendix~\ref{sec:proof-main-theorem}. We now have a few remarks with these results.

\vspace{0.5 em}
\noindent
\textit{(1) The sequences $\braces{\vecnorm{\Ucoeff_{t}\Ucoeff_{t}^\top - \trueUcoeff}{2}}_t$ and $\braces{\vecnorm{\Ucoeff_t \Vcoeff_t^{\top}}{2}}_t$ first converge linearly and then sub-linearly.} Theorem~\ref{theorem:sample-convergence-T-small} indicates that the sequences $\braces{\vecnorm{\Ucoeff_{t}\Ucoeff_{t}^\top - \trueUcoeff}{2}}_t$ and $\braces{\vecnorm{\Ucoeff_t \Vcoeff_t^{\top}}{2}}_t$ first converge linearly from $0.1\sigma_r$ to  $\epsilon_{comp} $, and then converge sub-linearly to $\Omega \parenth{\epsilon_{stat}}$.
Furthermore, the sequence $\braces{\vecnorm{\fitMat_t \fitMat_t^\top - \trueMat}{2}}_t$ always converges sublinearly towards $\Omega \parenth{\epsilon_{stat}}$.
This is consistent with our simulations in Figure \ref{fig:sim-verification}.
As we will see later in Lemma~\ref{lemma:sample-contraction-T-small}, it is the convergence of $\vecnorm{\trueVcoeff - \Vcoeff_t \Vcoeff_t^\top}{2}$ that slows down the convergence of $\braces{\vecnorm{\fitMat_t \fitMat_t^\top - \trueMat}{2}}_t$, and incurring the sublinear convergence of $\braces{\vecnorm{\Ucoeff_{t}\Ucoeff_{t}^\top - \trueUcoeff}{2}}_t$ and $\braces{\vecnorm{\Ucoeff_t \Vcoeff_t^{\top}}{2}}$.

\vspace{0.5 em}
\noindent
\textit{(2) There is a convergence rate discrepancy between the population and finite-sample versions}.
It is often believed that the convergence rate is consistent even if we go from finite $n$ to infinitely large $n$ (i.e., from finite sample scenario to the scenario when we have access to the population gradient).
However this is not the case in our setting.
As we will show shortly in Lemma~\ref{lemma:pop-contraction}, if we have access to the population gradient, the convergence rates of the sequences $\braces{\vecnorm{\Ucoeff_{t}\Ucoeff_{t}^\top - \trueUcoeff}{2}}_t$ and $\braces{\vecnorm{\Ucoeff_t \Vcoeff_t^T}{2}}_t$ are linear all the way until zero.
In our setting, going from population to finite-sample creates an unusual tangling factor, causing the convergence rate discrepancy between the finite-sample and population sequences.

\vspace{0.5 em}
\noindent
\textit{(3) This achieves nearly minimax-optimal statistical error.}
At a glance the statistical error seems too good to be true compared to Yudong's work, and even better than the minimax rate~\citep{candes2011tight}.
In fact the guarantees we offer are in spectral norm, while the typical rate in the related work is in Frobenius norm. Translating the spectral norm to Frobenius norm will introduce an extra $\sqrt{k}$ factor. 
That is, the statistical error is $\kappa\sqrt{\frac{ kd \log d }{n}}\sigma$ if we evaluate $\vecnorm{\fitMat_t \fitMat_t^\top - \trueMat}{F}$. This statistical error is similar to the results in \citet{chen2015fast} when the rank is known, i.e., $k = r$. 
Furthermore, we are able to cover both the noisy and noiseless matrix sensing settings.
Given that we assume no {\it a priori} knowledge of true rank $r$, the statistical error in Theorem~\ref{theorem:sample-convergence-T-small} is minimax optimal up to log factors~\citep{candes2011tight}.

\subsection{Simulation verification of the main result}

\begin{figure}[t]
\centering
\begin{subfigure}[t]{0.48\textwidth}
\includegraphics[width=1\textwidth]{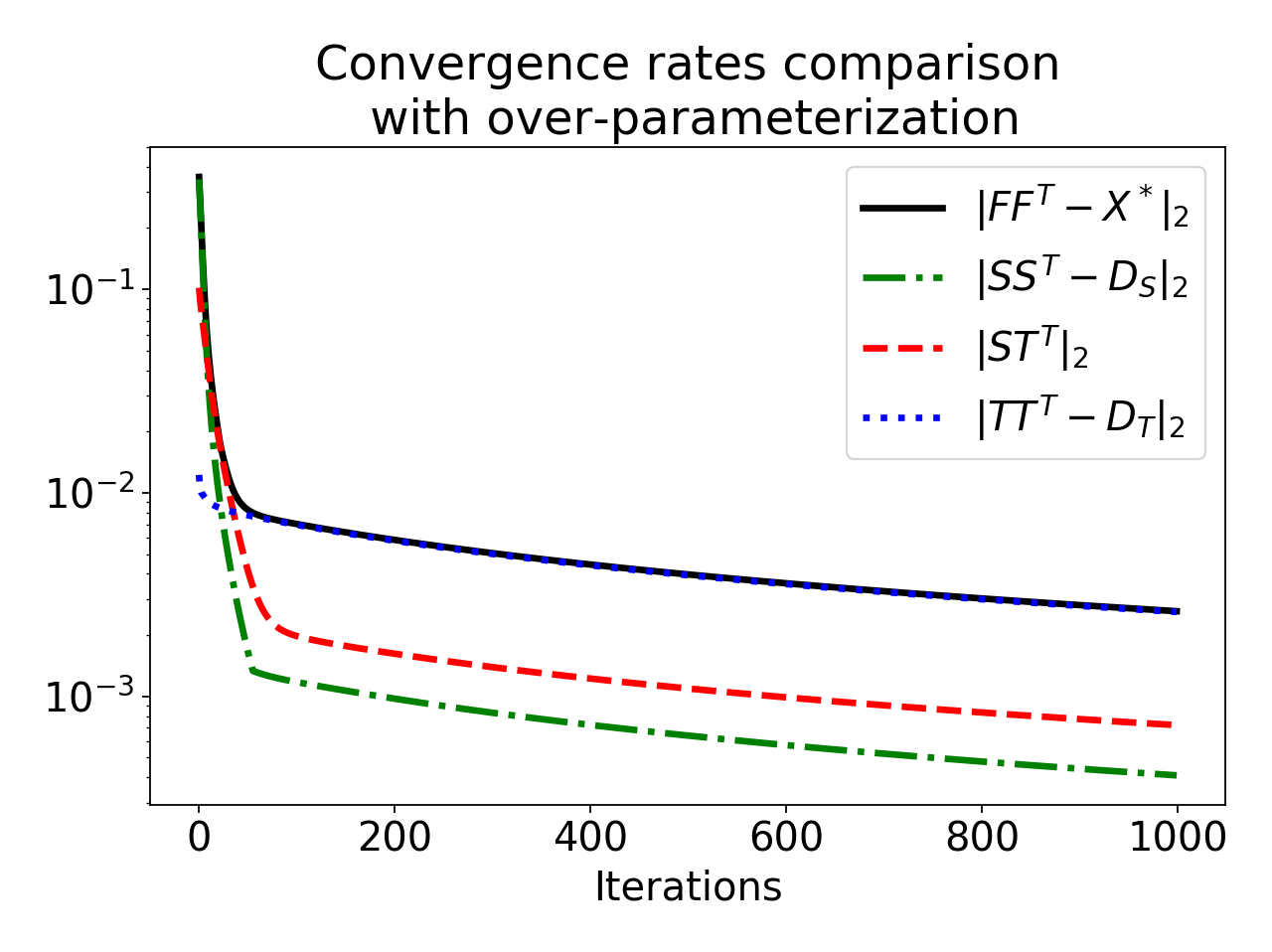}
\caption{}
\label{fig:sim-verification-a}
\end{subfigure}
\begin{subfigure}[t]{0.48\textwidth}
\includegraphics[width=1\textwidth]{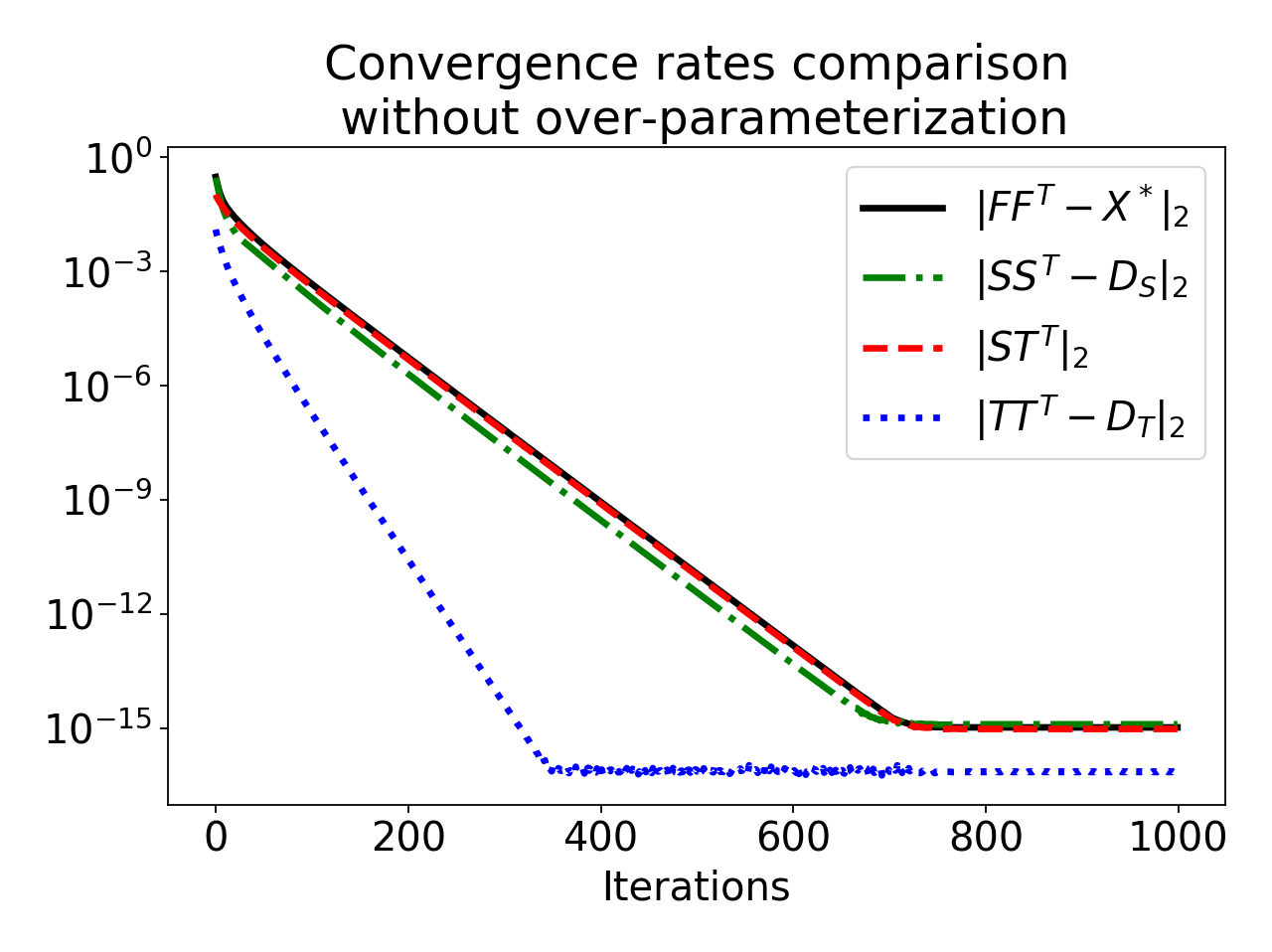}
\caption{}
\label{fig:sim-verification-b}
\end{subfigure}
\vspace{1 em}
\begin{subfigure}[t]{0.48\textwidth}
\includegraphics[width=1\textwidth]{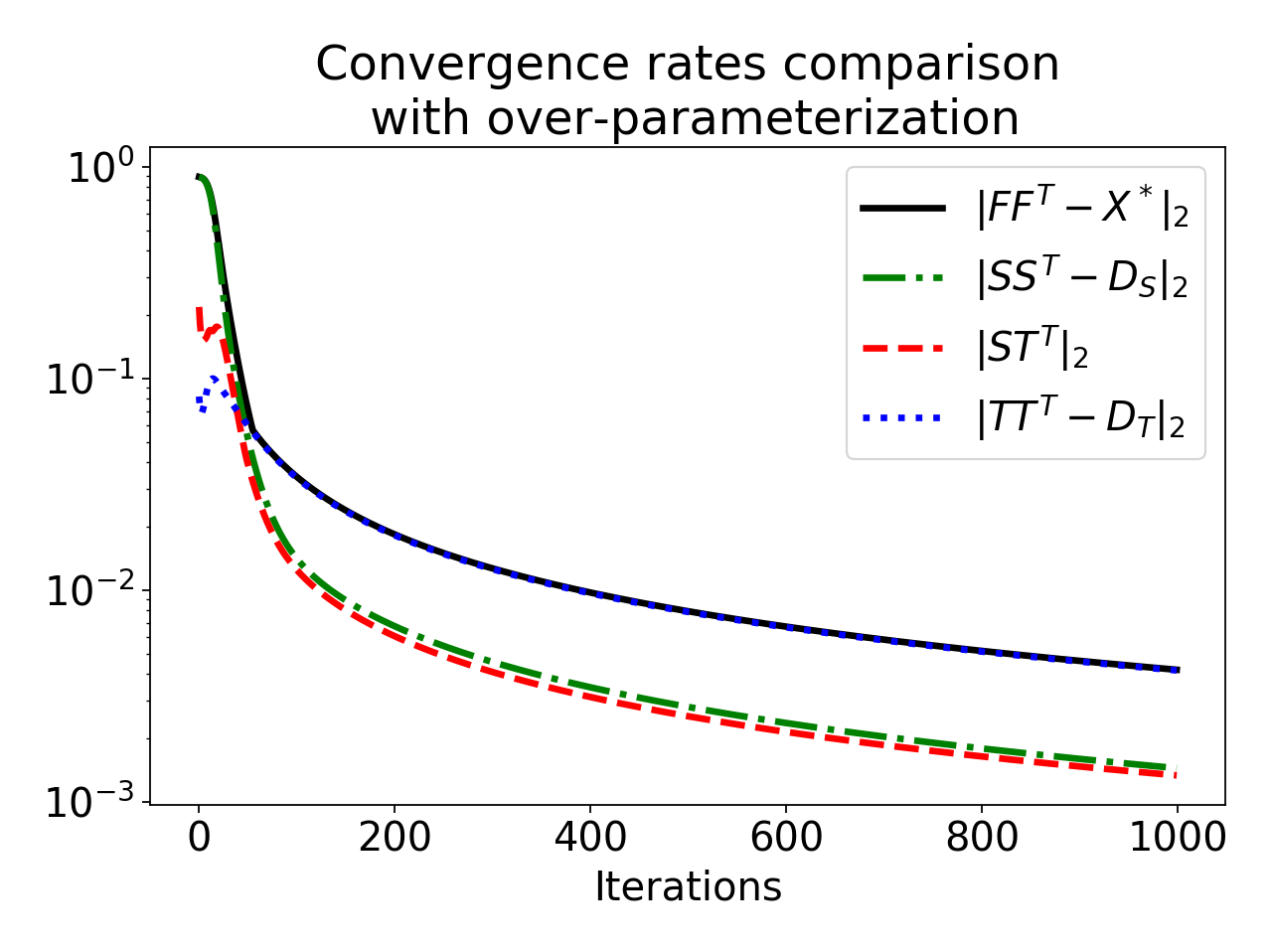}
\caption{}
\label{fig:sim-verification-c}
\end{subfigure}
\begin{subfigure}[t]{0.48\textwidth}
\includegraphics[width=1\textwidth]{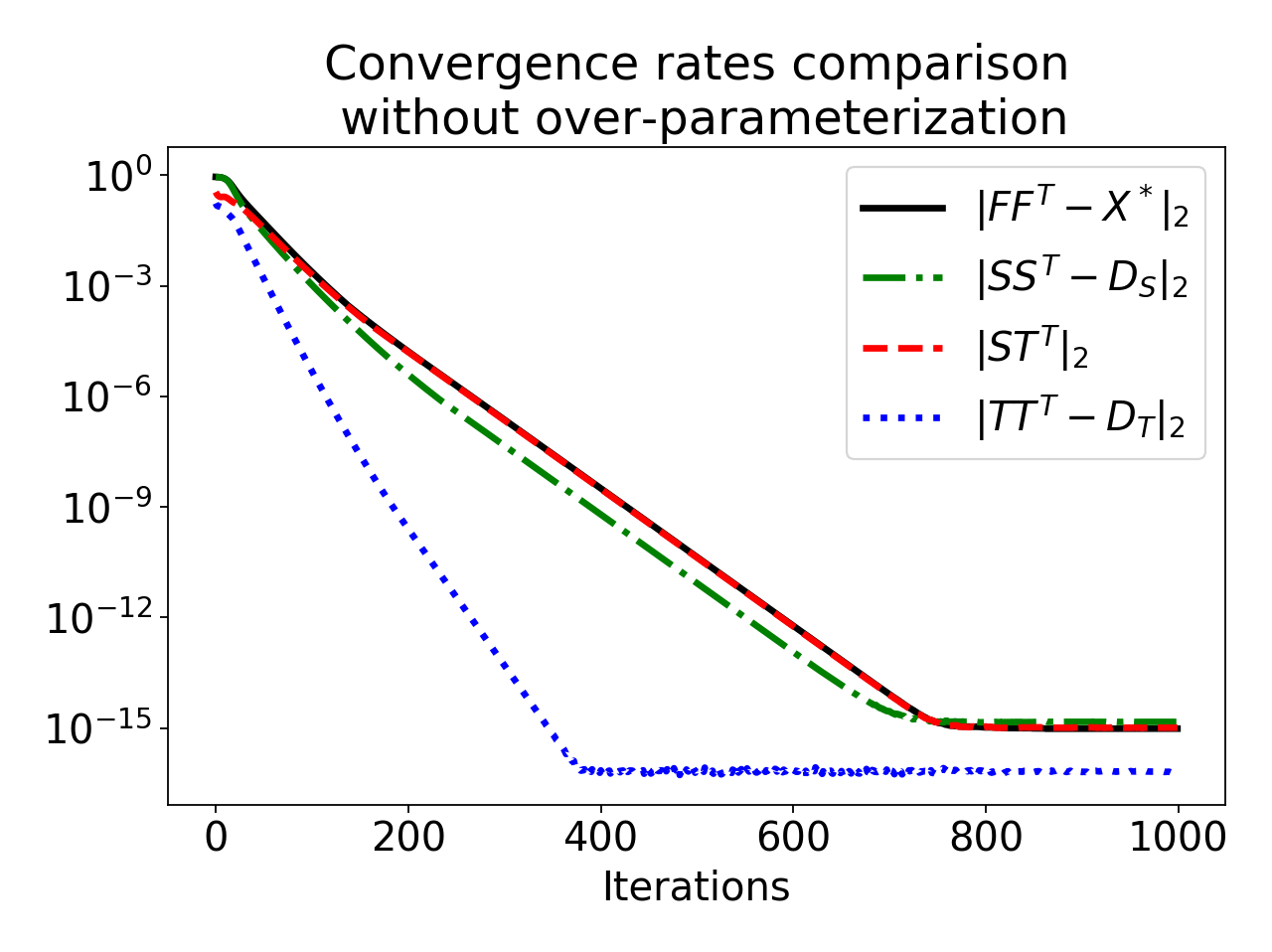}
\caption{}
\label{fig:sim-verification-d}
\end{subfigure}
\caption{
\textbf{Simulations that verify the main result.}
(a) Convergence rates of the FGD iterates when we over-specify the rank ($r=3, k=4$).
(b) Convergence rates of the FGD method when we correctly specify the rank ($r = k = 3$).
The Figures in (c) and (d) are executed in the same setting as those in (a) and (b) respectively, except with random initialization around the origin, instead of using Assumption~\ref{assumption:init}. 
}
\label{fig:sim-verification}
\end{figure}


In this subsection we use the same simulation setup as in Section \ref{sec:motivating-sim}.
Let $\braces{\fitMat_t}_t$ be the sequence generated by the FGD method as in Equation \ref{eqn:factored-gradient-method}, and let $\Ucoeff, \Vcoeff$ be defined as in the previous subsection.

The simulation results are shown in Figure \ref{fig:sim-verification}.
In Figure \ref{fig:sim-verification-a}, we plot $\vecnorm{\fitMat \fitMat^\top - \trueMat}{2}$, $\vecnorm{\Ucoeff \Ucoeff^\top - \trueUcoeff}{2}$, $\vecnorm{\Ucoeff \Vcoeff^\top}{2}$, and $\vecnorm{\Vcoeff \Vcoeff^\top - \trueVcoeff}{2}$ against the algorithm iterations.
The simulation results are aligned with our theory.
As said in Theorem \ref{theorem:sample-convergence-T-small}, $\vecnorm{\Ucoeff \Ucoeff^\top - \trueUcoeff}{2}$ and $\vecnorm{\Ucoeff \Vcoeff^\top}{2}$ first converge linearly, and then sublinearly. Furthermore,
$\vecnorm{\fitMat \fitMat^\top - \trueMat}{2}$ is soon dominated by $\vecnorm{\Vcoeff \Vcoeff^\top - \trueVcoeff}{2}$, which converges sub-linearly all the time.
Note that these phenomena are when the true rank is $3$ and we set $k = 4$.
If we correctly specify the rank ($k=r=3$), the convergences will be linear, as shown in Figure \ref{fig:sim-verification-b}.
In Figures~\ref{fig:sim-verification-c} and~\ref{fig:sim-verification-d}, we re-produce the result as in Figures~\ref{fig:sim-verification-a} and~\ref{fig:sim-verification-b} with random initialization.
This indicates that our assumption of initialization could possibly be waived using recent insights about the global landscape of the matrix sensing problem~\citep{zhang2020many}.

\section{Proof of the main result}
\label{sec:proof_main_result}
The proof of the main result follows the typical population-sample analysis~\citep{Siva_2017}.
We first analyze the convergence behavior of the algorithm when we have access to the population gradient.
Then in the finite sample setting, we quantify the difference between the population gradient and the finite sample gradient using concentration arguments, and use this difference plus the convergence result in population analysis, to characterize the convergence behavior in the finite sample setting.

While it is common to use the Restricted Isometric Property (RIP) as the building block to encapsulate the concentration requirement \citep{chen2015fast, chi2019nonconvex, li2018algorithmic}, 
we build our results directly based on the concentration of sub-Gaussian sensing matrices for technical convenience.
While it is possible to control the Frobenius norm directly, we find it technically easier and more reader friendly to show that the sequence $\vecnorm{\bF_t \bF_t^\top - \trueMat}{2}$ converges, and the resulting statistical rate is tight.
However, RIP is defined in Frobenius norm since it was first developed for vector and then extended to matrix \citep{recht2010guaranteed, candes2011tight}.
Translating the Frobenius norm directly to spectral norm will incur a $\Theta \parenth{\sqrt{k}}$ factor of sub-optimality.
That being said, we believe that it is possible to establish similar results for $\vecnorm{\bF_t \bF_t^\top - \trueMat}{F}$ directly, and hence we can use the general RIP notion.
We leave this for future work.

\subsection{Population analysis}
\label{sec:popu-analysis}
The first step of our analysis is to understand the contraction if we have access to the population gradient. 
One can check that
$\Exs \left[ \angles{\matA_i, \matB} \matA_i \right] = \matB$
for any matrix $\matB$ with appropriate dimensions.
Combined with the fact that $y_i = \angles{\matA_i, \mathbf{X}^*} + \epsilon$,  the population gradient (taking expectation over the observation noise $\epsilon$ and the observation matrices $\matA_i $) is
\begin{align*}
    \mathbf{G}_t := \Exs [ \bG_t^n] = \left(\fitMat_t \fitMat_t^{\top} - \mathbf{X}^* \right)  \fitMat_t.
\end{align*}
A closer look at the update in the Factored Gradient Method (Equation \ref{eqn:factored-gradient-method}) with population gradient reveals that 
at each iteration, the update only changes the coefficient
matrices $\Ucoeff$ and $\Vcoeff$. Simple algebra using the last observation yields:
\begin{align*}
  & \hspace{- 4 em} \fitMat_{t}  - \stepsize \bG_t   \\
  =&   \fitMat_t - \stepsize \parenth{\fitMat_t \fitMat_t^{\top} \fitMat_t - 
   \mathbf{X}^* \fitMat_t } \\
  =&  \Umat \Ucoeff_t + \Vmat \Vcoeff_t   - \stepsize \bracket{ \left( \Umat \Ucoeff_t + \Vmat \Vcoeff_t \right) \left( \Ucoeff_t^\top \Ucoeff_t + \Vcoeff_t^\top \Vcoeff_t \right) - (\Umat \trueUcoeff \Ucoeff_t + \Vmat \trueVcoeff \Vcoeff_t) } \\
  =& \Umat \popUup\parenth{\Ucoeff_t} + \Vmat \popVup\parenth{\Vcoeff_t}
\end{align*}
where we define the following operators:
\begin{align*}
\popUup(\Ucoeff) 
&= \Ucoeff - \stepsize \parenth{ \Ucoeff \Ucoeff^\top \Ucoeff
 	 + \Ucoeff \Vcoeff^\top \Vcoeff - \trueUcoeff \Ucoeff }, \\
\popVup(\Vcoeff) 
&= \Vcoeff - \stepsize \parenth{ \Vcoeff \Vcoeff^\top \Vcoeff
 	 + \Vcoeff \Ucoeff^\top \Ucoeff - \trueVcoeff \Vcoeff }.
\end{align*}

\begin{lemma}
\label{lemma:pop-contraction}
\textbf{(Contraction per iteration with access to the population gradient.)}
Set $\eta = \frac{1}{100 \sigma_1}$ .
We assume good initialization as in Assumption \ref{assumption:init}. 
Then we have:
\begin{enumerate}
    \item[(a)] $\vecnorm{\trueUcoeff - \popUup(\Ucoeff)  \popUup(\Ucoeff)^\top}{2} \leq \parenth{1 - \eta \sigma_r}\vecnorm{\trueUcoeff - \Ucoeff \Ucoeff^\top}{2} + 3 \eta \vecnorm{\Ucoeff \Vcoeff^\top}{2}^2$,
    \item[(b)] $\vecnorm{\popUup(\Ucoeff)\popVup(\Vcoeff)^\top }{2} \leq \vecnorm{\Ucoeff \Vcoeff^\top}{2} \parenth{1 - \eta \sigma_r}$,
    \item[(c)] $\vecnorm{\popVup(\Vcoeff) \popVup(\Vcoeff)^\top }{2} \le \vecnorm{\Vcoeff \Vcoeff^\top}{2} \parenth{1 - \eta \vecnorm{\Vcoeff \Vcoeff^\top}{2} + 2\eta \vecnorm{\trueVcoeff}{2}}  $, 
    \item[(d)] $ \vecnorm{\popVup(\Vcoeff) \popVup(\Vcoeff)^\top - \trueVcoeff}{2} \leq \vecnorm{\Vcoeff \Vcoeff^\top - \trueVcoeff}{2} \vecnorm{\Id - 2\eta \Vcoeff \Vcoeff^\top}{2} + 3 \eta \vecnorm{\Ucoeff\Vcoeff^\top}{2}^2$.
\end{enumerate}

\end{lemma}
\noindent
The proof of Lemma~\ref{lemma:pop-contraction}  can be found in Appendix~\ref{subsec:proof:lemma:pop_contraction}.

According to Lemma~\ref{lemma:pop-contraction} above, we have fast convergence in estimating $\Ucoeff\Ucoeff^\top$, $\Ucoeff\Vcoeff^\top$, but slow convergence in estimating $\Vcoeff\Vcoeff^\top$.
Intuitively, $\Vcoeff \Vcoeff^\top$ is slow because the local curvature of the population version of the loss function~\eqref{formulation:non-convex} is flat, namely, the Hessian matrix around the global maxima $\trueVcoeff$ is degenerate.
We know that when the curvature of the target matrix is undesirable, we can only guarantee sub-linear convergence rate \citep{bhojanapalli2016dropping}.

Note that, we assume that $k > r$ for the above analysis.
The case when $k \leq r$ is already covered by various existing works (see \citet{chen2015fast, tu2016low, bhojanapalli2016dropping} and the references therein); therefore, we will not focus on this setting in our analysis.

\subsection{Finite sample analysis}
\label{sec:finite_sample}





On top of our population analysis result, we consider the case when we only have access to the gradient evaluated with finitely many samples.
Consider the deviation of the population and sample gradient:
\begin{align*}
    \bG_t^n - \bG_t =& \frac{1}{n} \sum_{i=1}^n \parenth{ \angles{\matA_i, \fitMat_t \fitMat_t^{\top}} - y_i} \matA_i  \fitMat_t - \parenth{\fitMat_t \fitMat_t^{\top} - \mathbf{X}^* }  \fitMat_t \\
    =& \frac{1}{n} \sum_{i=1}^n \parenth{ \angles{\matA_i, \fitMat_t \fitMat_t^{\top} - \trueMat} + \epsilon_i } \matA_i  \fitMat_t - \parenth{\fitMat_t \fitMat_t^{\top} - \mathbf{X}^* }  \fitMat_t.
\end{align*}
We define $\Delta_t$ to quantify this deviation:
\begin{align*}
    \Delta_t = \frac{1}{n} \sum_i^n  \parenth{\angles{\senseMat_i, \fitMat_t \fitMat_t^\top - \mathbf{X}^*} + \epsilon_i} \senseMat_i   - (\fitMat_t \fitMat_t^\top - \mathbf{X}^*),
\end{align*}
and hence $\bG_t^n - \bG_t = \Delta_t \fitMat_t$.
If we can control $\Delta_t$, we can have contraction per-iteration, as shown in the lemma below.
Note that we make no attempts to optimize the constants.

\begin{lemma}
\label{lemma:sample-contraction-T-small}
\textbf{(Contraction per iteration.)}
Assume that we have the same setting as Theorem \ref{theorem:sample-convergence-T-small}.
Denote $D_t = \max \{ \vecnorm{\Ucoeff_t \Ucoeff_t^{\top} - \trueUcoeff}{2},  \vecnorm{\Vcoeff_t \Vcoeff_t^{\top}}{2}, \vecnorm{\Ucoeff_t \Vcoeff_t^{\top}}{2}\}$, and assume that $D_t$ is still sub-optimal to the statistical error: $D_t > 50\kappa \sqrt{\frac{ d \log d }{n}}\sigma$.
Suppose
\begin{align}
\label{eqn:tight-delta_bound}
\vecnorm{\Delta_t}{2} \leq 10 \sqrt{\frac{k d \log d}{n} }  D_t + 4 \sqrt{\frac{d \log d}{n}} \sigma,
\end{align}
Then, we find that
\begin{align}
\label{eqn:contraction-SS-main-text}
    \vecnorm{\Ucoeff_{t+1}\Ucoeff_{t+1}^\top - \trueUcoeff}{2} \leq&
    \parenth{1 - \frac{7}{10}\eta\sigma_r} \vecnorm{\Ucoeff_{t}\Ucoeff_{t}^\top - \trueUcoeff}{2}
    + \sqrt{\frac{k d \log d}{n} }  D_t 
    + \frac{4}{10} \sqrt{\frac{d \log d }{n}}\sigma,\\
\label{eqn:contraction-ST-main-text}
    \vecnorm{\Ucoeff_{t+1}\Vcoeff_{t+1}^\top }{2} \leq& 
    \parenth{1 - \eta \sigma_r } \vecnorm{\Ucoeff_t \Vcoeff_t^\top}{2}
    + \sqrt{\frac{k d \log d}{n} }  D_t 
    + \frac{4}{10} \sqrt{\frac{d \log d }{n}}\sigma.
\end{align}
Moreover, denote $\epsilon_{stat} = \kappa \sqrt{\frac{d \log d }{n}}\sigma$. Then we have
\begin{align}
\label{eqn:sub-linear-contraction-in-lemma}
    \parenth{D_{t+1} -  50 \epsilon_{stat}}  \leq \bracket{1 - \frac{1}{2} \eta \parenth{D_t -  50 \epsilon_{stat}}}\parenth{D_t -  50 \epsilon_{stat}}.
\end{align}
\end{lemma}
\noindent
The proof of Lemma~\ref{lemma:sample-contraction-T-small} can be found in Appendix~\ref{section:proof-sample-contraction-T-small}.

\paragraph{Implication of equations~\eqref{eqn:contraction-SS-main-text} and~\eqref{eqn:contraction-ST-main-text}:} 
Firstly, when $n$ goes to infinity, the sequence of $\braces{\vecnorm{\Ucoeff_{t}\Ucoeff_{t}^\top - \trueUcoeff}{2}}_t$ has constant contraction at each step, and hence achieves a linear convergence after all. This matches our population results in Lemma~\ref{lemma:pop-contraction}.
Secondly, if $n$ is finite, the sequence of $\braces{\vecnorm{\Ucoeff_{t}\Ucoeff_{t}^\top - \trueUcoeff}{2}}_t$ still has constant contraction, until roughly the magnitude of $\vecnorm{\Ucoeff_{t}\Ucoeff_{t}^\top - \trueUcoeff}{2}$ reaches $5 \sqrt{\frac{k d \log d}{n} }  D_t$. 
This indicates that we will have a linear convergence behavior in the beginning, and then sublinear convergence, as is indicated in Theorem \ref{theorem:sample-convergence-T-small}.


\subsection{Proof sketch for the main theorem }
\label{sec:proof-sketch}
In this subsection we offer a proof sketch for Theorem \ref{theorem:sample-convergence-T-small}.
Detailed proof can be found in Appendix~\ref{sec:proof-main-theorem}.


Lemma \ref{lemma:sample-contraction-T-small} is our key building block towards the main theorem.
However there are two missing pieces.
(1) Firstly we have to establish equation~\eqref{eqn:tight-delta_bound} so that Lemma~\ref{lemma:sample-contraction-T-small} can be invoked for one iteration.
(2) Secondly we have to find a way to correctly invoke Lemma \ref{lemma:sample-contraction-T-small} for all iterations and obtain the correct statistical rate.

We resolve the first point by bounding $\vecnorm{\Delta_t}{2}$ using matrix Bernstein concentration bound \citep{troppuserfriendly} together with the $\epsilon$-net discretization techniques.

\begin{lemma}
\label{lemma:concentration-Aepsilon-main-text}
Let $\bA_i$ be symmetric random matrices in $\mathbb{R}^{d*d}$, with the upper triangle entries ($i\geq j$) being independently sampled from an identical  sub-Gaussian distribution whose mean is $0$ and variance proxy is $1$. Let $\epsilon_i$ follows $N(0, \sigma)$. Then
\begin{align*}
\Prob \parenth{\vecnorm{\frac{1}{n}\sum_i^n \bA_i \epsilon_i}{2} \geq C \sqrt{\frac{d \sigma^2}{n}}} \leq \exp(-C).
\end{align*}
\end{lemma}

\begin{lemma}
\label{lemma:concentration-uniform-main-text}
Let $\bA_i$ be a symmetric random matrix of dimension $d$ by $d$. Its upper triangle entries ($i\geq j$) are independently sampled from an identical  sub-Gaussian distribution whose mean is $0$ and variance proxy is $1$. If $\bU$ is of rank $k$ and is in a bounded spectral norm ball of radius $R$ (i.e., $\|\bU\|_2 \leq R$),
\begin{align*}
\Prob \parenth{\sup_{\bU: \vecnorm{\bU}{2} \leq R}\vecnorm{\frac{1}{n}\sum_i^n \parenth{\langle \bA_i, \bU \rangle \bA_i - \bU}}{2} \leq \sqrt{\frac{d \log d}{n}} \sqrt{k}R} 
\geq  1 - \exp \parenth{ - C_2 \log d} .
\end{align*}
\end{lemma}
The proof of the above two concentration results can be found in the Appendix~\ref{appendix:concentration-proof}.
If we invoke these lemmas for $\Delta_t$, then we can immediately have
\begin{align*}
\vecnorm{\Delta_t}{2} &= \vecnorm{\frac{1}{n} \sum_i^n  \angles{\senseMat_i, \fitMat_t \fitMat_t^\top - \mathbf{X}^*} \senseMat_i   - (\fitMat_t \fitMat_t^\top - \mathbf{X}^*) + \frac{1}{n} \sum_i^n \epsilon_i \senseMat_i}{2}\\
&\leq 5\sqrt{\frac{k d \log d}{n} }  D_t + \sqrt{\frac{d \log d}{n}} \sigma.
\end{align*}

\paragraph{The linear convergence part.}
Claim (a) in the main theorem is about linear convergence.
We mention in the remark that equations~\eqref{eqn:contraction-SS-main-text} and~\eqref{eqn:contraction-ST-main-text} imply constant contractions for $\vecnorm{\Ucoeff_{t}\Ucoeff_{t}^\top - \trueUcoeff}{2}$ and $\vecnorm{\Ucoeff_{t}\Vcoeff_{t}^\top}{2}$ respectively.
To make the argument more precise,
we consider $\vecnorm{\Ucoeff_{t}\Ucoeff_{t}^\top - \trueUcoeff}{2} > 1000 \sqrt{\frac{k \kappa^2 d \log d}{n} }  \sigma_r > 1000 \sqrt{\frac{k \kappa^2 d \log d}{n} }  D_t$. 
Then, we find that $\sqrt{\frac{k d \log d}{n} }  D_t \leq 0.1\eta \sigma_r \vecnorm{\Ucoeff_{t}\Ucoeff_{t}^\top - \trueUcoeff}{2}$ since $\eta \sigma_r = 0.01/\kappa$.
Also, $\frac{4}{10} \sqrt{\frac{d \log d }{n}}\sigma < 0.1\eta \sigma_r \vecnorm{\Ucoeff_{t}\Ucoeff_{t}^\top - \trueUcoeff}{2}$ by the choice of the constants in the lower bound of $n$.
Hence, when $\vecnorm{\Ucoeff_{t}\Ucoeff_{t}^\top - \trueUcoeff}{2} > 1000 \sqrt{\frac{k \kappa^2 d \log d}{n} }  \sigma_r$, we find that
\begin{align*}
    \vecnorm{\Ucoeff_{t+1}\Ucoeff_{t+1}^\top - \trueUcoeff}{2} 
    \leq &
    \parenth{1 - \frac{7}{10}\eta\sigma_r} \vecnorm{\Ucoeff_{t}\Ucoeff_{t}^\top - \trueUcoeff}{2}
    + \sqrt{\frac{k d \log d}{n} }  D_t 
    + \frac{4}{10} \sqrt{\frac{d \log d }{n}}\sigma \\
    \leq &
    \parenth{1 - \frac{5}{10}\eta\sigma_r} \vecnorm{\Ucoeff_{t}\Ucoeff_{t}^\top - \trueUcoeff}{2}.
\end{align*}
The same arguments hold for $\vecnorm{\Ucoeff_{t}\Vcoeff_{t}^\top}{2}$.
Therefore to obtain the linear convergence result as the part (a) in the main theorem, we can just invoke concentration lemmas for each iteration to obtain constant contraction, and then take union bound over all the iterations.

\paragraph{The sub-linear convergence part.}
Claim (b) of the main theorem is about sublinear convergence, and is build upon equation~\eqref{eqn:sub-linear-contraction-in-lemma}.

Before we discuss how equation~\eqref{eqn:tight-delta_bound} holds in this sub-linear convergence case for all iteration $t$,
we briefly illustrate how equation~\eqref{eqn:sub-linear-contraction-in-lemma} implies convergence to $\Theta\parenth{\epsilon_{stat}}$ after $\mathcal{O}(1/\epsilon_{stat})$ iterations.
By equation~\eqref{eqn:sub-linear-contraction-in-lemma}, we know that $A_{t+1} \leq \parenth{1 - \frac{1}{2} \eta A_t} A_t $ where $A_t = D_t -  50 \epsilon_{stat}$. Hence
\begin{align*}
     A_{t+1} \leq \parenth{1 - \frac{1}{2} \eta A_t} A_t 
     \overset{(1)}{\leq} &\parenth{1 -   \frac{2}{ t + \frac{4}{\eta A_0}}} \frac{4}{\eta t + \frac{4}{A_0}} 
     = \frac{\parenth{t + \frac{4}{\eta A_0}}-2}{t + \frac{4}{\eta A_0}} \frac{4}{\eta \parenth{t + \frac{4}{\eta A_0}}}\\
     \overset{(2)}{\leq} & \frac{4}{\eta \parenth{t + 1+ \frac{4}{\eta A_0}}},
\end{align*}
where inequality $(1)$ holds because $\parenth{1 - \frac{1}{2} \eta A_t} A_t$ is quadratic with respect to $A_t$ and we plug-in the optimal $A_t$; inequality $(2)$ holds because $\frac{\parenth{t + \frac{4}{\eta A_0}}-2}{\parenth{t + \frac{4}{\eta A_0}}^2} \leq \frac{1}{\parenth{t + \frac{4}{\eta A_0}} + 1}$.
Therefore, after $t \geq \Theta\parenth{\frac{1}{\eta \epsilon_{stat}}}$ number of iterations, $A_t = D_t - 50 \kappa \sqrt{\frac{d \log d }{n}}\sigma \leq \Theta\parenth{\epsilon_{stat}}$.

We still need to show that equation~\eqref{eqn:tight-delta_bound} holds (with probability at least $1 - d^{-c}$) in this sub-linear convergence case for all iteration $t$ with high probability.
To do so, we need to use the localization technique \citep{kwon2020minimax, dwivedi2019challenges, Raaz_Ho_Koulik_2018}.
Without the localization technique, the statistical error will be proportional to $n^{-1/4}$ which is not tight.
With the localization argument, we can improve it to $n^{-1/2}$, making the result nearly minimax optimal.
We leave the details of this argument to Appendix~\ref{sec:proof-main-theorem}.

\section{Discussion}
\label{sec:discussion}
In the paper, we provide a comprehensive analysis of the computational and statistical complexity of the Factorized Gradient Descent method under the over-parameterized matrix sensing problem, namely, when the true rank is unknown and over-specified. We show that $\vecnorm{\fitMat_t \fitMat_t - \trueMat}{F}^2$ converges to a minimax optimal radius of convergence $\tilde{\mathcal{O}} \parenth{k d/n}$ after $\tilde{\mathcal{O}}(\sqrt{\frac{n}{d}})$ number of iterations where $\fitMat_t$ is the output of FGD after $t$ iterations. We now discuss a few natural questions with this work.

\vspace{0.5 em}
\noindent
\textbf{Can the results in~\cite{li2018algorithmic} imply this work?} We would like to explain the difference between our results and those in~\cite{li2018algorithmic}.
If we choose the specified rank $k$ as $d$, we have the same problem setting, and use the same algorithm. However, the results are different.
The key difference here is the sample complexity.
As~\cite{li2018algorithmic} focus on over-parameterization, their analysis requires $\tilde{\mathcal{O}} (d r)$ samples, where $r$ is the rank of the ground truth matrix $\bX^*$, while our analysis requires $\tilde{\mathcal{O}} (d k) = \tilde{\mathcal{O}} (d^2)$ samples when $k = d$.
Since they only require $\tilde{\mathcal{O}} (d r)$ samples, they cannot control the error of the over-parameterization part (equivalent to our $\Vcoeff \Vcoeff^\top$ part). 
In fact in their analysis, they only show that in a limited number of step, this error does not blow up.
While with $\tilde{\mathcal{O}} (d^2)$ we can show that the over-parameterization part also converges, although with a slower convergence rate. Therefore, their results cannot imply ours.

\vspace{0.5 em}
\noindent
\textbf{Extensions to general convex function of low rank problem.}
One natural question to ask is, if our analysis can be extended to general convex function with respect to a low rank matrix. In particular, we consider minimizing a convex function $f \parenth{\bX}$, where $\bX \in \mathbb{R}^{d*d}$ is PSD. Let $\bX^*$ be the ground truth solution and it is of rank $r$.
We can as well reformulate our problem as minimizing $f \parenth{\fitMat \fitMat^\top}$ where $\fitMat \in \mathbb{R}^{d*k}$.
Then as long as the population gradient with respect to $\bX$ is $\bX - \bX^*$, and the sample gradients have good concentration around the population gradient, our analysis techniques should be applicable.
For example, matrix completion and principle component analysis fall into this category~\citep{chen2015fast}.
However, extension the current results with over-parameterized matrix sensing to general convex function as in~\cite{chen2015fast}, or in~\cite{bhojanapalli2016dropping} would require more refined analysis. We leave this question for the future work.



\section{Acknowledgements}
We would like to thank Raaz Dwivedi, Koulik Khamaru, and Martin Wainwright for helpful discussion with this work.


\appendix

\section{Proofs for population analysis}
\label{appendix:population-proof}
In this appendix, we provide all the proofs for population analysis of matrix sensing problem.
\subsection{Proof of Lemma \ref{lemma:pop-contraction}}
\label{subsec:proof:lemma:pop_contraction}
We prove the four contraction results separately. 
To simplify the ensuing presentation, we drop all subscripts $t$ associated with the iteration counter.
\paragraph{Proof of the contraction result $(a)$ in Lemma \ref{lemma:pop-contraction}:} We would like to prove the following inequality:
\begin{align*}
    \vecnorm{\trueUcoeff - \popUup(\Ucoeff)  \popUup(\Ucoeff)^\top}{2} \leq \parenth{1 - \eta \sigma_r}\vecnorm{\trueUcoeff - \Ucoeff \Ucoeff^\top}{2} + 3 \eta \vecnorm{\Ucoeff \Vcoeff^\top}{2}^2.
\end{align*}
Indeed, from the formulation of $\popUup(\Ucoeff)$, we have
\begin{align*}
    & \hspace{- 3 em} \trueUcoeff - \popUup(\Ucoeff)  \popUup(\Ucoeff) ^\top  \\
    =&\trueUcoeff - 
    \parenth{\Ucoeff - \stepsize \Ucoeff \Vcoeff^\top \Vcoeff + \stepsize \parenth{\trueUcoeff - \Ucoeff \Ucoeff^\top} \Ucoeff }
    \parenth{\Ucoeff - \stepsize \Ucoeff \Vcoeff^\top \Vcoeff + \stepsize \parenth{\trueUcoeff - \Ucoeff \Ucoeff^\top} \Ucoeff }^\top.
\end{align*}
We can group the terms in the RHS of the above equation according to whether they contain $\trueUcoeff - \Ucoeff \Ucoeff^\top$ or not, namely, we find that
\begin{align*}
    \trueUcoeff - \popUup(\Ucoeff)  \popUup(\Ucoeff) ^\top  &= I + II\\
    \text{where}, \quad \quad 
    I &= \parenth{\trueUcoeff - \Ucoeff \Ucoeff^\top} - \eta \parenth{\trueUcoeff - \Ucoeff \Ucoeff^\top} \Ucoeff \Ucoeff^\top - \eta \Ucoeff \Ucoeff^\top \parenth{\trueUcoeff - \Ucoeff \Ucoeff^\top} \\
    & \quad - \eta^2 \parenth{\trueUcoeff - \Ucoeff \Ucoeff^\top} \Ucoeff \Ucoeff^\top \parenth{\trueUcoeff - \Ucoeff \Ucoeff^\top} \\
    & \quad + \eta^2 \parenth{\trueUcoeff - \Ucoeff \Ucoeff^\top} \Ucoeff \Vcoeff^\top \Vcoeff \Ucoeff^\top + \eta^2 \Ucoeff \Vcoeff^\top \Vcoeff \Ucoeff^\top \parenth{\trueUcoeff - \Ucoeff \Ucoeff^\top}, \\
    II &= 2 \eta \Ucoeff \Vcoeff^\top \Vcoeff \Ucoeff^\top - \eta^2 \Ucoeff \Vcoeff^\top \Vcoeff \Vcoeff^\top \Vcoeff \Ucoeff^\top.
\end{align*}

We first deal with the $I$ term. A direct application of inequality with operator norm leads to
\begin{align*}
    \vecnorm{I}{2} &\le \vecnorm{\trueUcoeff - \Ucoeff \Ucoeff^\top}{2} \vecnorm{\Id - 2 \eta \Ucoeff \Ucoeff^\top - \eta^2 \Ucoeff \Ucoeff^\top \parenth{\trueUcoeff - \Ucoeff \Ucoeff^\top} + 2\eta^2 \Ucoeff \Vcoeff^\top \Vcoeff \Ucoeff^\top}{2}.
\end{align*}
From the choice of the step size and the initialization condition, the term $\Id - 2 \eta \Ucoeff \Ucoeff^\top - \eta^2 \Ucoeff \Ucoeff^\top \parenth{\trueUcoeff - \Ucoeff \Ucoeff^\top} + 2\eta^2 \Ucoeff \Vcoeff^\top \Vcoeff \Ucoeff^\top$ is PSD matrix. Furthermore, for any $\|\bz\| = 1$, we have
\begin{align*}
    & \hspace{- 6 em} \bz^\top \parenth{\Id - 2 \eta \Ucoeff \Ucoeff^\top - \eta^2 \Ucoeff \Ucoeff^\top \parenth{\trueUcoeff - \Ucoeff \Ucoeff^\top} + 2\eta^2 \Ucoeff \Vcoeff^\top \Vcoeff \Ucoeff^\top} \bz  \\
    \le & 1 - 2 \eta \vecnorm{\Ucoeff \bz}{2}^2  + \eta^2 \vecnorm{\Ucoeff \Ucoeff^\top}{2} \vecnorm{ \trueUcoeff - \Ucoeff \Ucoeff^\top}{2} + 2 \eta^2 \vecnorm{\Vcoeff \Vcoeff^\top}{2} \vecnorm{\Ucoeff \bz}{2}^2 \\
    \stackrel{(i)}{\le} &1 - 2 \eta \sigma_r + 3 \eta^2 \sigma_r \sigma_1 \\
    \stackrel{(ii)}{\le} & 1 - \eta \sigma_r,
\end{align*}
where in step~(i) we used $0.9 \sigma_r \Id \preceq \Ucoeff \Ucoeff^\top \preceq (\sigma_1+0.1\sigma_r) \Id$,  $\vecnorm{\trueUcoeff - \Ucoeff \Ucoeff^\top}{2} \le 0.1 \sigma_r$ and $\vecnorm{\Vcoeff \Vcoeff^{\top}}{2}\leq 1.1 \sigma_r$ by initialization condition and triangular inequality; step~(ii) follows from choice of step size $\eta = \frac{1}{100\sigma_1}$, and definition of the conditional number $\kappa = \sigma_1 / \sigma_r$. 
Therefore, we arrive at the following inequality: 
\begin{align}
\vecnorm{\Id - 2 \eta \Ucoeff \Ucoeff^\top - \eta^2 \Ucoeff \Ucoeff^\top \parenth{\trueUcoeff - \Ucoeff \Ucoeff^\top} + 2\eta^2 \Ucoeff \Vcoeff^\top \Vcoeff \Ucoeff^\top}{2} \leq 1 - \eta \sigma_r. \label{eq:population_analysis_first_term_I}
\end{align}

To deal with the $II$ term, we have to establish the connection between $\trueUcoeff - \Ucoeff \Ucoeff^\top$ and $\Ucoeff \Vcoeff^\top$. Note that, $\vecnorm{\eta^2 \Ucoeff \Vcoeff^\top \Vcoeff \Vcoeff^\top \Vcoeff \Ucoeff^\top}{2} \leq \eta^2 \vecnorm{\Vcoeff\Vcoeff^\top}{2} \vecnorm{\Ucoeff\Vcoeff^\top\Vcoeff\Ucoeff^\top}{2} \leq \eta \vecnorm{\Ucoeff\Vcoeff^\top\Vcoeff\Ucoeff^\top}{2}$ since $\eta \leq 1/\sigma_r$. Hence, we have
\begin{align}
    \vecnorm{II}{2}  \leq \vecnorm{2 \eta \Ucoeff \Vcoeff^\top \Vcoeff \Ucoeff^\top}{2} + \vecnorm{ \eta^2 \Ucoeff \Vcoeff^\top \Vcoeff \Vcoeff^\top \Vcoeff \Ucoeff^\top }{2}
    \leq 3 \eta \vecnorm{\Ucoeff \Vcoeff^\top \Vcoeff \Ucoeff^\top}{2}. \label{eq:population_analysis_first_term_II}
\end{align}
Collecting the results from equations~\eqref{eq:population_analysis_first_term_I} and~\eqref{eq:population_analysis_first_term_II}, we obtain 
\begin{align*}
    \vecnorm{\trueUcoeff - \popUup(\Ucoeff)  \popUup(\Ucoeff)^\top}{2} \leq \parenth{1 - \eta \sigma_r}\vecnorm{\trueUcoeff - \Ucoeff \Ucoeff^\top}{2} + 3 \eta \vecnorm{\Ucoeff \Vcoeff^\top}{2}^2.
\end{align*}
Therefore, we reach the conclusion with claim (a) in Lemma~\ref{lemma:pop-contraction}.
\paragraph{Proof of the contraction result $(b)$ in Lemma \ref{lemma:pop-contraction}:} Recall that we want to demonstrate that
\begin{align*}
    \vecnorm{\popUup(\Ucoeff)\popVup(\Vcoeff)^\top }{2} \leq \vecnorm{\Ucoeff \Vcoeff^\top}{2} \parenth{1 - \eta \sigma_r}.
\end{align*}
Firstly, from the formulations of $\popUup(\Ucoeff)$ and $\popVup(\Vcoeff)$, we have the following equations:
\begin{align}
    & \hspace{- 3 em} \popUup(\Ucoeff)\popVup(\Vcoeff)^\top \nonumber \\ 
    = & \parenth{\Ucoeff + \stepsize \parenth{\trueUcoeff -\Ucoeff \Ucoeff^\top } \Ucoeff
 	 - \stepsize \Ucoeff \Vcoeff^\top \Vcoeff }
     \parenth{\Vcoeff + \stepsize \parenth{\trueVcoeff -\Vcoeff \Vcoeff^\top } \Vcoeff
 	 - \stepsize \Vcoeff \Ucoeff^\top \Ucoeff }^\top \nonumber \\
    =& \frac{1}{2} \parenth{\Id - 2 \eta \Ucoeff \Ucoeff^\top + 2\eta \parenth{\trueUcoeff - \Ucoeff \Ucoeff^\top }  - 2 \eta^2  \parenth{\trueUcoeff - \Ucoeff \Ucoeff^\top } \Ucoeff \Ucoeff^\top + 2 \eta^2 \Ucoeff \Vcoeff^\top \Vcoeff \Ucoeff^\top }  \Ucoeff \Vcoeff^\top \nonumber \\
    & + \frac{1}{2} \Ucoeff \Vcoeff^\top \parenth{\Id + 2 \eta \parenth{\trueVcoeff - \Vcoeff \Vcoeff^\top} - 2 \eta \Vcoeff \Vcoeff^\top  - 2 \eta^2  \Vcoeff \Vcoeff^\top\parenth{\trueVcoeff - \Vcoeff \Vcoeff^\top}}  \nonumber \\
    & + \eta^2 \parenth{\trueUcoeff - \Ucoeff \Ucoeff^\top} \Ucoeff \Vcoeff^\top \parenth{\trueVcoeff - \Vcoeff \Vcoeff^\top} . \label{eq:population_analysis_second_term}
\end{align}
Recall that, we have  $0.9 \sigma_r \Id \preceq \Ucoeff \Ucoeff^\top \preceq (\sigma_1+0.1\sigma_r) \Id$,  $\vecnorm{\trueUcoeff - \Ucoeff \Ucoeff^\top}{2} \le 0.1 \sigma_r$ and $\vecnorm{\Vcoeff \Vcoeff^{\top}}{2}\leq 1.1 \sigma_r$ by initialization condition and triangular inequality, and we choose $\eta = \frac{1}{100\sigma_1}$.

For the term in the first line of the RHS of equation~\eqref{eq:population_analysis_second_term} we have 
\begin{align*}
    & \vecnorm{\frac{1}{2} \parenth{\Id - 2 \eta \Ucoeff \Ucoeff^\top + 2\eta \parenth{\trueUcoeff - \Ucoeff \Ucoeff^\top }  - 2 \eta^2  \parenth{\trueUcoeff - \Ucoeff \Ucoeff^\top } \Ucoeff \Ucoeff^\top + 2 \eta^2 \Ucoeff \Vcoeff^\top \Vcoeff \Ucoeff^\top }  \Ucoeff \Vcoeff^\top}{2} \\
    \leq & \frac{1}{2} \vecnorm{\Ucoeff \Vcoeff^\top}{2} \parenth{ \vecnorm{1 - 2 \eta \Ucoeff\Ucoeff^\top}{2} + 2  \eta \vecnorm{\trueUcoeff - \Ucoeff \Ucoeff^\top}{2} + 2 \eta^2 \vecnorm{\trueUcoeff - \Ucoeff \Ucoeff^\top}{2}\vecnorm{\Ucoeff \Ucoeff^\top}{2} + 2 \eta^2 \vecnorm{\Ucoeff \Vcoeff^\top \Vcoeff \Ucoeff^\top}{2}} \\
    \leq & \frac{1}{2} \vecnorm{\Ucoeff \Vcoeff^\top}{2} \parenth{1 - 1.8 \eta \sigma_r + 0.2  \eta \sigma_r + 0.0022 \eta \sigma_r + 0.02 \eta^2 \sigma_r^2} \\
    \leq & \frac{1}{2} \vecnorm{\Ucoeff \Vcoeff^\top}{2} \parenth{1 - 1.5 \eta \sigma_r}.
\end{align*}
For the term in the second line of the RHS of equation~\eqref{eq:population_analysis_second_term}, direct calculation yields that 
\begin{align*}
    & \hspace{- 6 em}  \vecnorm{\frac{1}{2} \Ucoeff \Vcoeff^\top \parenth{\Id + 2 \eta \parenth{\trueVcoeff - \Vcoeff \Vcoeff^\top} - 2 \eta \Vcoeff \Vcoeff^\top  - 2 \eta^2  \Vcoeff \Vcoeff^\top\parenth{\trueVcoeff - \Vcoeff \Vcoeff^\top}}}{2} \\
    \leq & \frac{1}{2} \vecnorm{\Ucoeff \Vcoeff^\top}{2} \parenth{\vecnorm{\Id - 2 \eta \Vcoeff \Vcoeff^\top}{2} + 0.2 \eta  \sigma_r + 2.2 \eta^2  \sigma_r^2}\\
    \leq & \frac{1}{2} \vecnorm{\Ucoeff \Vcoeff^\top}{2} \parenth{1 + 0.3 \eta \rho \sigma_r}.
\end{align*}
Lastly, for the second order term in the third line of the RHS of equation~\eqref{eq:population_analysis_second_term} we have
\begin{align*}
    \vecnorm{\eta^2 \parenth{\trueUcoeff - \Ucoeff \Ucoeff^\top} \Ucoeff \Vcoeff^\top \parenth{\trueVcoeff - \Vcoeff \Vcoeff^\top}}{2} \leq \eta^2 \rho^2 \sigma_r^2 \vecnorm{\Ucoeff \Vcoeff^\top}{2} = \frac{1}{1000} \eta \sigma_r\vecnorm{\Ucoeff \Vcoeff^\top}{2}.
\end{align*}
Plugging the above results into equation~\eqref{eq:population_analysis_second_term} leads to
\begin{align*}
    \vecnorm{\popUup(\Ucoeff)\popVup(\Vcoeff)^\top }{2} \leq \vecnorm{\Ucoeff \Vcoeff^\top}{2} \parenth{1 - \eta \sigma_r}.
\end{align*}
Hence, we obtain the conclusion of claim (b) in Lemma~\ref{lemma:pop-contraction}.
\paragraph{Proof of the contraction result (c) in Lemma \ref{lemma:pop-contraction}:} We would like to establish that
\begin{align*}
    \vecnorm{\popVup(\Vcoeff) \popVup(\Vcoeff)^\top }{2} \le \vecnorm{\Vcoeff \Vcoeff^\top}{2} \parenth{1 - \eta \vecnorm{\Vcoeff \Vcoeff^\top}{2} + 2\eta \vecnorm{\trueVcoeff}{2}}.  
\end{align*}
To check the convergence in low SNR, i.e., with small singular values, we assume that $\|\trueVcoeff \| \ll \sigma_r$. It suggests that the focus is how fast $\Vcoeff \Vcoeff^\top$ converges to 0 when $\|\Vcoeff \Vcoeff^\top\|\gg \|\trueVcoeff \|$. Indeed, simple algebra indicates that 
\begin{align*}
    &\popVup(\Vcoeff) \popVup(\Vcoeff)^\top  \\
    = & \parenth{\Vcoeff + \stepsize \parenth{\trueVcoeff -\Vcoeff \Vcoeff^\top } \Vcoeff
 	 - \stepsize \Vcoeff \Ucoeff^\top \Ucoeff } 
 	 \parenth{\Vcoeff + \stepsize \parenth{\trueVcoeff -\Vcoeff \Vcoeff^\top } \Vcoeff
 	 - \stepsize \Vcoeff \Ucoeff^\top \Ucoeff }^\top \\
    = & \Vcoeff \Vcoeff^\top + \eta \Vcoeff \Vcoeff^\top  \parenth{\trueVcoeff -\Vcoeff \Vcoeff^\top } - \eta \Vcoeff \Ucoeff^\top \Ucoeff \Vcoeff^\top\\
    &+ \eta \parenth{\trueVcoeff -\Vcoeff \Vcoeff^\top } \Vcoeff \Vcoeff^\top + \eta^2 \parenth{\trueVcoeff -\Vcoeff \Vcoeff^\top } \Vcoeff \Vcoeff^\top \parenth{\trueVcoeff -\Vcoeff \Vcoeff^\top } - \eta^2 \parenth{\trueVcoeff -\Vcoeff \Vcoeff^\top } \Vcoeff \Ucoeff^\top \Ucoeff \Vcoeff^\top\\
    &- \eta \Vcoeff \Ucoeff^\top \Ucoeff \Vcoeff^\top - \eta^2 \Vcoeff \Ucoeff^\top \Ucoeff\Vcoeff^\top \parenth{\trueVcoeff -\Vcoeff \Vcoeff^\top }^\top + \eta^2 \Vcoeff \Ucoeff^\top \Ucoeff \Ucoeff^\top \Ucoeff \Vcoeff^\top\\
    =& III + IV + V,
\end{align*}
where we use the following shorthand notation:
\begin{align*}
    III = & \parenth{\Vcoeff \Vcoeff^\top - 2\eta \parenth{\Vcoeff \Vcoeff^\top}^2 + \eta^2 \parenth{\Vcoeff \Vcoeff^\top}^3},\\
    IV = &\eta \parenth{\trueVcoeff \Vcoeff \Vcoeff^\top + \Vcoeff \Vcoeff^\top \trueVcoeff}- \parenth{\eta^2 \trueVcoeff \parenth{\Vcoeff \Vcoeff^\top }^2 + \eta^2  \parenth{\Vcoeff \Vcoeff^\top }^2 \trueVcoeff} + \eta^2 \trueVcoeff \parenth{\Vcoeff \Vcoeff^\top } \trueVcoeff, \\
    V = &- 2 \eta \Vcoeff \Ucoeff^\top \Ucoeff \Vcoeff^\top  
    - \eta^2 \Vcoeff \Ucoeff^\top \Ucoeff\Vcoeff^\top \parenth{\trueVcoeff -\Vcoeff \Vcoeff^\top }^\top
    - \eta^2 \parenth{\trueVcoeff -\Vcoeff \Vcoeff^\top }\Vcoeff \Ucoeff^\top \Ucoeff\Vcoeff^\top  \\
    &+ \eta^2 \Vcoeff \Ucoeff^\top \Ucoeff \Ucoeff^\top \Ucoeff \Vcoeff^\top.
\end{align*} 
We first bound the $IV$ term. Inequalities with operator norm show that
\begin{align*}
    \vecnorm{\trueVcoeff \Vcoeff \Vcoeff^\top}{2} &\le \vecnorm{\trueVcoeff}{2} \vecnorm{ \Vcoeff \Vcoeff^\top}{2}, \\
    \vecnorm{\trueVcoeff \parenth{ \Vcoeff \Vcoeff^\top}^2}{2} &\le \sigma_r \vecnorm{\trueVcoeff}{2} \vecnorm{ \Vcoeff \Vcoeff^\top}{2}, \\
    \vecnorm{\trueVcoeff \parenth{ \Vcoeff \Vcoeff^\top} \trueVcoeff}{2} &\le \sigma_r \vecnorm{\trueVcoeff}{2} \vecnorm{ \Vcoeff \Vcoeff^\top}{2}.
\end{align*}
Given these bounds, we find that
\begin{align}
    \vecnorm{IV}{2} \leq \parenth{\eta + 3 \eta^2 \sigma_r} \vecnorm{\trueVcoeff}{2} \vecnorm{ \Vcoeff \Vcoeff^\top}{2}. \label{eq:population_analysis_third_term_IV}
\end{align}
Now, we move to bound the $V$ term. Indeed, we have
\begin{align}
    &- 2 \eta \Vcoeff \Ucoeff^\top \Ucoeff \Vcoeff^\top  
    - \eta^2 \Vcoeff \Ucoeff^\top \Ucoeff\Vcoeff^\top \parenth{\trueVcoeff -\Vcoeff \Vcoeff^\top }^\top
    - \eta^2 \parenth{\trueVcoeff -\Vcoeff \Vcoeff^\top }\Vcoeff \Ucoeff^\top \Ucoeff\Vcoeff^\top  
    + \eta^2 \Vcoeff \Ucoeff^\top \Ucoeff \Ucoeff^\top \Ucoeff \Vcoeff^\top \nonumber \\
    & \preceq \parenth{-2\eta + 2\eta^2 \rho \sigma_r + \eta^2 \sigma_1} \Vcoeff \Ucoeff^\top \Ucoeff \Vcoeff^\top \preceq 0. \label{eq:population_analysis_third_term_V}
\end{align}
Since $\popVup(\Vcoeff) \popVup (\Vcoeff)^\top$ is PSD, we can just relax this term to zero.
Finally, we bound the $III$ term. Observe that,
\begin{align*}
    \Vcoeff \Vcoeff^\top - 2\eta \parenth{\Vcoeff \Vcoeff^\top }^2 + \eta^2 \parenth{\Vcoeff \Vcoeff^\top }^3 \preceq \Vcoeff \Vcoeff^\top - \eta \parenth{\Vcoeff \Vcoeff^\top }^2,
\end{align*}
since $\eta < 1/\sigma_1$ and $\| \Vcoeff \Vcoeff^\top \| \le \rho \sigma_r$. The remaining task is to bound $\Vcoeff \Vcoeff^\top - \eta \parenth{\Vcoeff \Vcoeff^\top }^2$. Let the singular value decomposition of $\Vcoeff \Vcoeff^\top$ as $Q D Q^\top$. Note that $D$ is a diagonal matrix with diagonal entries less than $(1 + \rho) \sigma_r$. We can proceed as
\begin{align*}
    \|\Vcoeff \Vcoeff^\top - \eta \parenth{\Vcoeff \Vcoeff^\top }^2\| &= \max_{\|z\|=1} \parenth{ z^\top \Vcoeff \Vcoeff^\top z - \eta z^\top \parenth{\Vcoeff \Vcoeff^\top }^2 z } \\
    &= \max_{\|z\|=1} \parenth{ z^\top Q D Q^\top z - \eta z^\top Q D^2 Q^\top z } \\
    &= \max_{\|z'\|=1} \parenth{ z'^\top D z' - \eta z'^\top D^2 z' } \\
    &= \max_{\|z'\|=1} \sum_i \parenth{d_i - \eta d_i^2} {z'}_i^2.
\end{align*}
Since $d_i < \sigma_r$ and $1/2\eta \gg \sigma_1$, the above maximum is obtained at the largest singular value of $\Vcoeff \Vcoeff^\top$. That is, we have
\begin{align}
    \vecnorm{\Vcoeff \Vcoeff^\top - \eta \parenth{\Vcoeff \Vcoeff^\top }^2}{2} \le \vecnorm{\Vcoeff \Vcoeff^\top}{2} \parenth{1 - \eta \vecnorm{\Vcoeff \Vcoeff^\top}{2}}. \label{eq:population_analysis_third_term_III}
\end{align}
Now combining every pieces from equations~\eqref{eq:population_analysis_third_term_IV},~\eqref{eq:population_analysis_third_term_V}, and~\eqref{eq:population_analysis_third_term_III}, we arrive at the following inequality:
\begin{align*}
    \vecnorm{\popVup(\Vcoeff) \popVup(\Vcoeff)^\top }{2} \le \vecnorm{\Vcoeff \Vcoeff^\top}{2} \parenth{1 - \eta \vecnorm{\Vcoeff \Vcoeff^\top}{2} + 2\eta \vecnorm{\trueVcoeff}{2}}. 
\end{align*}
As long as $\vecnorm{\trueVcoeff}{2} \ll \vecnorm{\Vcoeff \Vcoeff^\top}{2}$, the contraction rate is roughly $(1 - \eta \vecnorm{\Vcoeff \Vcoeff^\top}{2})$. Therefore, we obtain the conclusion of claim (c) in Lemma~\ref{lemma:pop-contraction}.


\paragraph{Proof of the contraction result (d) in Lemma \ref{lemma:pop-contraction}:}
Direct calculation shows that
\begin{align*}
    & \hspace{-2 em} \popVup(\Vcoeff) \popVup(\Vcoeff)^\top - \trueVcoeff \\
    =&\parenth{\Vcoeff - \stepsize \parenth{ \Vcoeff \Ucoeff^\top \Ucoeff + \parenth{\Vcoeff \Vcoeff^\top  - \trueVcoeff} \Vcoeff } }
    \parenth{\Vcoeff - \stepsize \parenth{ \Vcoeff \Ucoeff^\top \Ucoeff + \parenth{\Vcoeff \Vcoeff^\top  - \trueVcoeff} \Vcoeff } }^\top 
    - \trueVcoeff\\
    = & VI + VII,
\end{align*}
where we denote VI and VII as follows:
\begin{align*}
    VI = &(\Vcoeff \Vcoeff^\top - \trueVcoeff) - \eta ((\Vcoeff \Vcoeff^\top - \trueVcoeff) \Vcoeff \Vcoeff^\top + \Vcoeff \Vcoeff^\top (\Vcoeff \Vcoeff^\top - \trueVcoeff)) \\
    &+ \eta^2 (\Vcoeff \Vcoeff^\top - \trueVcoeff) \Vcoeff \Vcoeff^\top (\Vcoeff \Vcoeff^\top - \trueVcoeff), \\
    VII = &  \eta^2 (\Vcoeff \Vcoeff^\top - \trueVcoeff) \Vcoeff \Ucoeff^\top \Ucoeff \Vcoeff^\top + \eta^2 \Vcoeff \Ucoeff^\top \Ucoeff \Vcoeff^\top (\Vcoeff \Vcoeff^\top - \trueVcoeff) \\
    & -2 \eta \Vcoeff \Ucoeff^\top \Ucoeff \Vcoeff^\top + \eta^2 \Vcoeff \Ucoeff^\top \Ucoeff \Ucoeff^\top \Ucoeff \Vcoeff^\top.
\end{align*}
We first show that the $\vecnorm{VII}{2} \leq 3 \eta \vecnorm{\Ucoeff\Vcoeff^\top}{2}^2$. 
Firstly, since $\eta \leq \frac{1}{10\sigma_1}$ and the initialization condition $ \vecnorm{\Vcoeff \Vcoeff^\top - \trueVcoeff}{2}\leq \rho \sigma_r$, we have
\begin{align*}
    \eta^2 \vecnorm{(\Vcoeff \Vcoeff^\top - \trueVcoeff) \Vcoeff \Ucoeff^\top \Ucoeff \Vcoeff^\top}{2} \leq \frac{1}{10} \eta \vecnorm{\Vcoeff \Ucoeff^\top \Ucoeff \Vcoeff^\top}{2}.
\end{align*}
Furthermore, by the choice of $\eta$ and the fact that $\vecnorm{\Ucoeff \Ucoeff^\top}{2} \leq \sigma_1 $, we find that
\begin{align*}
    \eta^2 \vecnorm{\Vcoeff \Ucoeff^\top \Ucoeff \Ucoeff^\top \Ucoeff \Vcoeff^\top}{2} 
    \leq \frac{1}{10} \eta \vecnorm{\Vcoeff \Ucoeff^\top \Ucoeff \Vcoeff^\top}{2}.
\end{align*}
Putting these results together we have $\vecnorm{VII}{2} \leq 3 \eta \vecnorm{\Ucoeff\Vcoeff^\top}{2}^2$.

Now for the $VI$ term, direct calculation shows that
\begin{align*}
    VI 
    = & \parenth{\Vcoeff \Vcoeff^\top - \trueVcoeff} - \eta \parenth{\Vcoeff \Vcoeff^\top - \trueVcoeff} \Vcoeff \Vcoeff^\top \parenth{I - \frac{\eta}{2} \parenth{\Vcoeff \Vcoeff^\top - \trueVcoeff}}\\
    &- \eta \parenth{I - \frac{\eta}{2} \parenth{\Vcoeff \Vcoeff^\top - \trueVcoeff}}\Vcoeff \Vcoeff^\top \parenth{\Vcoeff \Vcoeff^\top - \trueVcoeff} \\
    = & \parenth{\Vcoeff \Vcoeff^\top - \trueVcoeff} \parenth{\frac{\Id}{2} - \eta \Vcoeff \Vcoeff^\top \parenth{I - \frac{\eta}{2} \parenth{\Vcoeff \Vcoeff^\top - \trueVcoeff}}} \\
    &+  \parenth{\frac{\Id}{2} - \eta \parenth{I - \frac{\eta}{2} \parenth{\Vcoeff \Vcoeff^\top - \trueVcoeff}}\Vcoeff \Vcoeff^\top } \parenth{\Vcoeff \Vcoeff^\top - \trueVcoeff}.
\end{align*}
Note that, since $\vecnorm{\Vcoeff \Vcoeff^\top}{2} \leq \rho \sigma_r$, $\vecnorm{\Vcoeff \Vcoeff^\top - \trueVcoeff}{2} \leq \rho \sigma_r$, and $\eta = \frac{1}{C\sigma_1}$, we obtain that
\begin{align*}
    \vecnorm{\frac{\Id}{2} - \eta \Vcoeff \Vcoeff^\top \parenth{I - \frac{\eta}{2} \parenth{\Vcoeff \Vcoeff^\top - \trueVcoeff}}}{2} \leq \vecnorm{\frac{\Id}{2} - \eta \Vcoeff \Vcoeff^\top}{2}.
\end{align*}
Collecting the above upper bounds with VI and VII, we arrive at
\begin{align*}
     \vecnorm{\popVup(\Vcoeff) \popVup(\Vcoeff)^\top - \trueVcoeff}{2} \leq \vecnorm{\Vcoeff \Vcoeff^\top - \trueVcoeff}{2} \vecnorm{\Id - 2\eta \Vcoeff \Vcoeff^\top}{2} + 3 \eta \vecnorm{\Ucoeff\Vcoeff^\top}{2}^2.
\end{align*}
As a consequence, we reach the conclusion of claim (d) in Lemma~\ref{lemma:pop-contraction}.
\subsection{Additional contraction results for population operators}
In this appendix, we offer more population contraction (non-expansion) results, which are useful in showing the contraction results in finite sample setting.
\begin{lemma}
\label{lemma:more-pop-contraction}
Under the same settings as Lemma \ref{lemma:pop-contraction}, we have
\begin{enumerate}
    \item[(a)] $\vecnorm{\trueUcoeff - \popUup(\Ucoeff)\Ucoeff^\top}{2} \leq \parenth{1 - \eta \sigma_r}\vecnorm{\trueUcoeff - \Ucoeff \Ucoeff^\top}{2} + \eta \vecnorm{\Ucoeff \Vcoeff^\top}{2}^2$,
    \item[(b)] $\vecnorm{\popUup(\Ucoeff)\Vcoeff^\top}{2} \leq \vecnorm{\Ucoeff \Vcoeff^\top}{2} $,
    \item[(c)] $\vecnorm{\popVup(\Vcoeff)\Ucoeff^\top}{2} \leq \vecnorm{\Ucoeff \Vcoeff^\top}{2}$,
    \item[(d)] $\vecnorm{\popVup(\Vcoeff)\Vcoeff^\top }{2} \le \vecnorm{\Vcoeff \Vcoeff^\top}{2} + \eta \vecnorm{\Ucoeff \Vcoeff^\top}{2}^2$. 
\end{enumerate}
\end{lemma}

\begin{proof}
With
$
\popUup(\Ucoeff) = \Ucoeff - \stepsize \parenth{ \Ucoeff \Ucoeff^\top \Ucoeff
 	 + \Ucoeff \Vcoeff^\top \Vcoeff - \trueUcoeff \Ucoeff }
$
and simple algebraic manipulations, we obtain
\begin{align*}
    \popUup(\Ucoeff)\Ucoeff^\top - \trueUcoeff = &\parenth{\Ucoeff \Ucoeff^\top - \trueUcoeff} - \stepsize \parenth{ \Ucoeff \Ucoeff^\top \Ucoeff \Ucoeff^\top+ \Ucoeff \Vcoeff^\top \Vcoeff \Ucoeff^\top- \trueUcoeff \Ucoeff \Ucoeff^\top} \\
    = &\parenth{\Ucoeff \Ucoeff^\top - \trueUcoeff} \parenth{\Id - \eta\Ucoeff \Ucoeff^\top } - \eta \Ucoeff \Vcoeff^\top \Vcoeff \Ucoeff^\top.
\end{align*}
Since $\vecnorm{\Ucoeff \Ucoeff^\top}{2} \geq 0.9 \sigma_r$ by initialization condition and triangular inequality, we know that $\vecnorm{\trueUcoeff - \popUup(\Ucoeff)\Ucoeff^\top}{2} \leq \parenth{1 - 0.9 \eta \sigma_r}\vecnorm{\trueUcoeff - \Ucoeff \Ucoeff^\top}{2} + \eta \vecnorm{\Ucoeff \Vcoeff^\top}{2}^2$. Therefore, we obtain the conclusion of claim (a).

Move to claim (b), with simple algebraic manipulations, we can show that
\begin{align*}
    \popUup(\Ucoeff)\Vcoeff^\top  = & \Ucoeff\Vcoeff^\top - \stepsize \parenth{ \Ucoeff \Ucoeff^\top \Ucoeff\Vcoeff^\top + \Ucoeff \Vcoeff^\top \Vcoeff\Vcoeff^\top - \trueUcoeff \Ucoeff\Vcoeff^\top } \\
    = & \parenth{\frac{1}{2}\Id - \eta \parenth{\Ucoeff \Ucoeff^\top - \trueUcoeff} }\Ucoeff\Vcoeff^\top - \Ucoeff \Vcoeff^\top\parenth{ \frac{1}{2}\Id - \eta \Vcoeff \Vcoeff^\top}.
\end{align*}
By initialization condition and triangular inequality, we know that $0 \leq \vecnorm{\parenth{\Ucoeff \Ucoeff^\top - \trueUcoeff}}{2} \leq \rho \sigma_r$ and $0 \leq \vecnorm{\Vcoeff \Vcoeff^\top}{2} \leq 1.1 \sigma_r$, and hence $\vecnorm{\popUup(\Ucoeff)\Vcoeff^\top}{2} \leq \vecnorm{ \Ucoeff\Vcoeff^\top}{2}$. Hence, we reach the conclusion of claim (b).

With
$
\popVup(\Vcoeff) = \Vcoeff - \stepsize \parenth{ \Vcoeff \Vcoeff^\top \Vcoeff
 	 + \Vcoeff \Ucoeff^\top \Ucoeff - \trueVcoeff \Vcoeff  }
$
and direct calculation, we find that
\begin{align*}
    \popVup(\Vcoeff)\Ucoeff^\top  = & \Vcoeff\Ucoeff^\top - \stepsize \parenth{ \Vcoeff \Vcoeff^\top \Vcoeff\Ucoeff^\top + \Vcoeff \Ucoeff^\top \Ucoeff\Ucoeff^\top - \trueVcoeff \Vcoeff\Ucoeff^\top } \\
    = & \parenth{\frac{1}{2}\Id - \eta \parenth{\Vcoeff \Vcoeff^\top - \trueVcoeff} }\Vcoeff\Ucoeff^\top - \Vcoeff\Ucoeff^\top\parenth{ \frac{1}{2}\Id - \eta \Ucoeff \Ucoeff^\top}.
\end{align*}
By initialization condition and triangular inequality, we know that $0 \leq \vecnorm{\parenth{\Vcoeff \Vcoeff^\top - \trueVcoeff}}{2} \leq \rho \sigma_r$ and $0.9 \sigma_r \leq \vecnorm{\Ucoeff \Ucoeff^\top}{2} \leq 0.1 \sigma_r + \sigma_1 $, and hence $\vecnorm{\popUup(\Ucoeff)\Vcoeff^\top}{2} \leq \vecnorm{ \Ucoeff\Vcoeff^\top}{2}$. It leads to the conclusion of claim (c).

Finally, moving to claim (d), simple algebra shows that
\begin{align*}
    \popVup(\Vcoeff)\Vcoeff^\top  = & \Vcoeff\Vcoeff^\top - \stepsize \parenth{ \Vcoeff \Vcoeff^\top \Vcoeff\Vcoeff^\top + \Vcoeff \Ucoeff^\top \Ucoeff\Vcoeff^\top - \trueVcoeff \Vcoeff\Vcoeff^\top } \\
    = & \parenth{\Id - \eta \parenth{\Vcoeff \Vcoeff^\top - \trueVcoeff} }\Vcoeff\Vcoeff^\top - \eta \Vcoeff \Ucoeff^\top \Ucoeff\Vcoeff^\top.
\end{align*}
By initialization condition and triangular inequality, we know that  $\vecnorm{\popVup(\Vcoeff)\Vcoeff^\top}{2} \leq \vecnorm{ \Vcoeff\Vcoeff^\top}{2} + \eta \vecnorm{\Ucoeff\Vcoeff^\top}{2}^2$. As a consequence, we obtain the conclusion of claim (d).

\end{proof}
\section{Proofs for the finite sample analysis}
\label{sec:finite_sample_analysis}
Recall that, we denote $\bG_t$ as the population gradient at iteration $t$ and denote $\bG_t^n$ as the corresponding sample gradient with sample size $n$:
\begin{align*}
    \bG_t &= \left(\fitMat_t \fitMat_t^{\top} - \mathbf{X}^* \right)  \fitMat_t, \\
    \bG_t^n &= \frac{1}{n} \sum_{i=1}^n \left( \angles{\senseMat_i, \fitMat_t \fitMat_t^{\top} - \mathbf{X}^*} + \epsilon_i \right)\senseMat_i  \fitMat_t.
\end{align*}    
Then, we can write our update as follows:
\begin{align*}
    \fitMat_{t+1} = \fitMat_t - \stepsize \bG_t + \stepsize \bG_t - \stepsize \bG_t^n.
\end{align*}
We assume the following decomposition by notations: $\fitMat = \Umat \Ucoeff + \Vmat \Vcoeff$. Therefore, we find that
\begin{equation}
\label{eqn:St+1St+1}
\begin{split}
\Ucoeff_{t+1} \left(\Ucoeff_{t+1}\right)^{\top} =&
    \Umat^{\top} \fitMat_{t+1}  \left(\Umat^{\top} \fitMat_{t+1}\right)^{\top} \\
    =& \Umat^{\top} \parenth{\fitMat_t - \stepsize \bG_t}\parenth{\fitMat_t - \stepsize \bG_t}^{\top} \Umat + \stepsize^2 \Umat^{\top} \parenth{ \bG_t -  \bG_t^n}\parenth{ \bG_t -  \bG_t^n}^{\top} \Umat\\
    & + \stepsize \Umat^{\top} \parenth{\fitMat_t - \stepsize \bG_t}\parenth{ \bG_t -  \bG_t^n}^{\top} \Umat + \stepsize\Umat^{\top} \parenth{ \bG_t -  \bG_t^n} \parenth{\fitMat_t - \stepsize \bG_t}^{\top} \Umat \\
    =& \popUup(\Ucoeff_t) \popUup(\Ucoeff_t)^{\top}
    + \stepsize^2 \Umat^{\top} \parenth{ \bG_t -  \bG_t^n}\parenth{ \bG_t -  \bG_t^n}^{\top} \Umat \\
    & + \stepsize\popUup(\Ucoeff_t)\parenth{ \bG_t -  \bG_t^n}^{\top} \Umat + \stepsize\Umat^{\top} \parenth{ \bG_t -  \bG_t^n} \popUup(\Ucoeff_t)^{\top},
\end{split}
\end{equation}
where we define $\popUup(\Ucoeff_t)$ as follows:
$$\popUup(\Ucoeff_t) = \Umat^{\top} \parenth{\fitMat_t - \stepsize \bG_t}
= \Ucoeff_t - \stepsize \left( \Ucoeff_t \Ucoeff_t^\top \Ucoeff_t
 	 + \Ucoeff_t \Vcoeff_t^\top \Vcoeff_t - \trueUcoeff \Ucoeff_t \right).$$
Furthermore, direct calculation shows that
\begin{equation}
\label{eqn:St+1Tt+1}
\begin{split}
\Ucoeff_{t+1} \left(\Vcoeff_{t+1}\right)^{\top} =&
    \Umat^{\top} \fitMat_{t+1}  \left(\Vmat^T \fitMat_{t+1}\right)^{\top} \\
    =& \Umat^{\top} \parenth{\fitMat_t - \stepsize \bG_t}\parenth{\fitMat_t - \stepsize \bG_t}^{\top} \Vmat + \stepsize^2 \Umat^{\top} \parenth{ \bG_t -  \bG_t^n}\parenth{ \bG_t -  \bG_t^n}^{\top} \Vmat\\
    & + \stepsize \Umat^{\top} \parenth{\fitMat_t - \stepsize \bG_t}\parenth{ \bG_t -  \bG_t^n}^{\top} \Vmat + \stepsize\Umat^{\top} \parenth{ \bG_t -  \bG_t^n} \parenth{\fitMat_t - \stepsize \bG_t}^{\top} \Vmat \\
    =& \popUup(\Ucoeff_t) \popVup(\Vcoeff_t)^{\top}
    + \stepsize^2 \Umat^{\top} \parenth{ \bG_t -  \bG_t^n}\parenth{ \bG_t -  \bG_t^n}^{\top} \Vmat \\
    & + \stepsize\popUup(\Ucoeff_t)\parenth{ \bG_t -  \bG_t^n}^{\top} \Vmat + \stepsize\Umat^{\top} \parenth{ \bG_t -  \bG_t^n} \popVup(\Vcoeff_t)^{\top},
\end{split}
\end{equation}
where $\popVup(\Vcoeff_t)$ is given by:
$$\popVup(\Vcoeff_t) = \Vmat^{\top} \parenth{\fitMat_t - \stepsize \bG_t}
= \Vcoeff_{t} - \stepsize \parenth{ \Vcoeff_{t} \Vcoeff_{t}^\top \Vcoeff_{t}
 	 + \Vcoeff_{t} \Ucoeff_{t}^\top \Ucoeff_{t} - \trueVcoeff \Vcoeff_{t} }.$$
Similarly, we also have
\begin{equation}
\label{eqn:Tt+1Tt+1}
\begin{split}
\Vcoeff_{t+1} \left(\Vcoeff_{t+1}\right)^{\top} =&
    \Vmat^{\top} \fitMat_{t+1}  \left(\Vmat^{\top} \fitMat_{t+1}\right)^{\top} \\
    =& \Vmat^{\top} \parenth{\fitMat_t - \stepsize \bG_t}\parenth{\fitMat_t - \stepsize \bG_t}^{\top} \Vmat + \stepsize^2 \Vmat^{\top} \parenth{ \bG_t -  \bG_t^n}\parenth{ \bG_t -  \bG_t^n}^{\top} \Vmat\\
    & + \stepsize \Vmat^{\top} \parenth{\fitMat_t - \stepsize \bG_t}\parenth{ \bG_t -  \bG_t^n}^{\top} \Vmat + \stepsize\Vmat^T \parenth{ \bG_t -  \bG_t^n} \parenth{\fitMat_t - \stepsize \bG_t}^{\top} \Vmat \\
    =& \popVup(\Vcoeff_t) \popVup(\Vcoeff_t)^{\top}
    + \stepsize^2 \Vmat^{\top} \parenth{ \bG_t -  \bG_t^n}\parenth{ \bG_t -  \bG_t^n}^{\top} \Vmat \\
    & + \stepsize\popVup(\Vcoeff_t)\parenth{ \bG_t -  \bG_t^n}^{\top} \Vmat + \stepsize\Vmat^{\top} \parenth{ \bG_t -  \bG_t^n} \popVup(\Vcoeff_t)^{\top}. 
\end{split}
\end{equation}
Note that, in the above equations, $\popUup(\Ucoeff_t)$ and $\popVup(\Vcoeff_t)$ are the updates of the coefficients when we update $S$ and $T$ using the population gradient. Furthermore, $\stepsize^2 \Umat^{\top} \parenth{ \bG_t -  \bG_t^n}\parenth{ \bG_t -  \bG_t^n}^{\top}$, $\stepsize^2 \Umat^{\top} \parenth{ \bG_t -  \bG_t^n}\parenth{ \bG_t -  \bG_t^n}^{\top} \Vmat$, and $\stepsize^2 \Vmat^{\top} \parenth{ \bG_t -  \bG_t^n}\parenth{ \bG_t -  \bG_t^n}^{\top} \Vmat$ are second order terms and are relatively small. To facilitate the proof argument, we denote 
\begin{align*}
    \Delta_t : = \frac{1}{n} \sum_{i = 1}^n \left( \angles{\senseMat_i, \fitMat_t \fitMat_t^{\top} - \mathbf{X}^*} + \epsilon_i \right)\senseMat_i   - (\fitMat_t \fitMat_t^{\top} - \mathbf{X}^*).
\end{align*}
We can see that $\Delta_t$ is symmetric matrix, and 
\begin{align*}
    \bG_t^m - \bG_t = \Delta_t \fitMat_t.
\end{align*}

\subsection{Proof for Lemma \ref{lemma:sample-contraction-T-small}}
\begin{proof}
\label{section:proof-sample-contraction-T-small}
By Lemma \ref{lemma:pop-contraction} and Lemma \ref{lemma:more-pop-contraction}, we have the following contraction results:
\begin{equation}
\begin{split}
    \label{eqn:pop_contraction}
    \vecnorm{\popUup(\Ucoeff_t) \popUup(\Ucoeff_t)^\top - \trueUcoeff}{2}  &\leq \parenth{1 - \eta \sigma_r } \vecnorm{\Ucoeff_t \Ucoeff_t^\top - \trueUcoeff}{2} + 3\eta \vecnorm{\Ucoeff_t \Vcoeff_t^\top}{2}^2, \\
    \vecnorm{\popUup(\Ucoeff_t) \popVup(\Vcoeff_t)^\top}{2} & \leq \parenth{1 - \eta \sigma_r} \vecnorm{\Ucoeff_t \Vcoeff_t^\top}{2},\\
    \vecnorm{\popVup(\Vcoeff_t) \popVup(\Vcoeff_t)^\top }{2} &\leq \vecnorm{\Vcoeff_t \Vcoeff_t^\top}{2} \parenth{1 - \eta \vecnorm{\Vcoeff_t \Vcoeff_t^\top}{2} + 2\eta \vecnorm{\trueVcoeff}{2}},
\end{split}
\end{equation}
and the following non-expansion results:
\begin{equation}
\begin{split}
    \label{eqn:pop-non-expansion1}
    \vecnorm{\popUup(\Ucoeff_t) \Ucoeff_t^\top - \trueUcoeff}{2} &\leq \parenth{1 - \eta \sigma_r } \vecnorm{\Ucoeff_t \Ucoeff_t^\top - \trueUcoeff}{2} + \eta \vecnorm{\Ucoeff_t \Vcoeff_t^\top}{2}^2,   \\
    \vecnorm{\popUup(\Ucoeff_t) \Vcoeff_t^\top}{2} & \leq  \vecnorm{\Ucoeff_t \Vcoeff_t^\top}{2}, \\
    \vecnorm{\Ucoeff_t \popVup(\Vcoeff_t)^\top}{2} & \leq  \vecnorm{\Ucoeff_t \Vcoeff_t^\top}{2}, \\
    \vecnorm{\popVup(\Vcoeff_t) \Vcoeff_t^\top }{2} &\leq \vecnorm{\Vcoeff_t \Vcoeff_t^\top}{2} + \eta \vecnorm{\Ucoeff_t \Vcoeff_t^\top}{2}^2.
\end{split}
\end{equation}
For notation simplicity, let $D_t = \max \{ \vecnorm{\Ucoeff_t \Ucoeff_t^{\top} - \trueUcoeff}{2},  \vecnorm{\Vcoeff_t \Vcoeff_t^{\top}}{2}, \vecnorm{\Ucoeff_t \Vcoeff_t^{\top}}{2}\}$, and denote the statistical error $\epsilon_{stat} = \sqrt{\frac{ d \log d }{n}}\sigma$.
Since Assumption \ref{assumption:init} is satisfied, and $\vecnorm{\trueVcoeff}{2}\leq \epsilon_{stat}$, we have $D_t \leq \sigma_r$ by triangular inequality.
Since $\eta \vecnorm{\Ucoeff_t \Vcoeff_t^\top}{2} \leq \frac{1}{10} \eta \sigma_r $ by initialization, and $\vecnorm{\Ucoeff_t \Vcoeff_t^\top}{2} \leq D_t$, we have $\eta \vecnorm{\Ucoeff_t \Vcoeff_t^\top}{2}^2 \leq 0.1 \eta \sigma D_t$. Putting these results together, we have
\begin{equation}
\begin{split}
    \label{eqn:pop-non-expansion2}
    \vecnorm{\popUup(\Ucoeff_t) \popUup(\Ucoeff_t)^\top - \trueUcoeff}{2}  &\leq \parenth{1 - \frac{7}{10} \eta \sigma_r } D_t, \\
    \vecnorm{\popUup(\Ucoeff_t) \Ucoeff_t^\top - \trueUcoeff}{2} &\leq \parenth{1 - \frac{9}{10} \eta \sigma_r } D_t, \\
    \vecnorm{\popVup(\Vcoeff_{t}) \Vcoeff_{t}^\top }{2} &\leq \parenth{1 + \frac{1}{10}\eta \sigma_r} D_t.
\end{split}
\end{equation}

For the ease of the presentation, we assign a value to the constant $C_1$ as in the number of sample $n$.
From the requirements of the lemma we know that 
$
\vecnorm{\Delta_t}{2} \leq 10 \sqrt{\frac{k d \log d}{n} }  D_t + 4 \sqrt{\frac{ d \log d}{n}} \sigma.
$
We need to connect $\sqrt{\frac{k d \log d}{n} }$ with $\eta \sigma_r$ for the development of the proof.
Since $\eta = \frac{1}{100 \sigma_1}$ and $n > C_1 k \kappa^2 d \log^3 d \cdot \max(1, \sigma^2/\sigma_r^2)$, by choosing $C_1 \geq 1000^2$, we have
\begin{align}
\label{equation:condition-of-kdlogd/n}
    \sqrt{\frac{k d \log d}{n} }  \leq 0.1 \eta \sigma_r.
\end{align}
\paragraph{Upper bound for $\vecnorm{\Ucoeff_t \Ucoeff_t^{\top} - \trueUcoeff}{2}$:} According to equation \eqref{eqn:St+1St+1}, we have
\begin{align*}
    \Ucoeff_{t+1}\Ucoeff_{t+1}^{\top} - \trueUcoeff
    = & \underbrace{\popUup(\Ucoeff_t) \popUup(\Ucoeff_t)^{\top} - \trueUcoeff}_{\text{I}}
    + \underbrace{\stepsize^2 \Umat^{\top} \parenth{ \bG_t -  \bG_t^n}\parenth{ \bG_t -  \bG_t^n}^{\top} \Umat}_{\text{II}} \\
    & + \stepsize\popUup(\Ucoeff_t)\parenth{ \bG_t -  \bG_t^n}^{\top} \Umat + \stepsize\Umat^{\top} \parenth{ \bG_t -  \bG_t^n} \popUup(\Ucoeff_t)^{\top},
\end{align*}
where we can further expand $\stepsize\popUup(\Ucoeff_t)\parenth{ \bG_t -  \bG_t^n}^{\top} \Umat$ and $\stepsize\Umat^{\top} \parenth{ \bG_t -  \bG_t^n} \popUup(\Ucoeff_t)^{\top}$ as follows:
\begin{align*}
    & \hspace{-4 em} \stepsize\popUup(\Ucoeff_t)\parenth{ \bG_t -  \bG_t^n}^{\top} \Umat \\
    = & \stepsize\popUup(\Ucoeff_t)\fitMat_t^{\top} \Delta_t \Umat =  \stepsize\popUup(\Ucoeff_t)\parenth{\Umat \Ucoeff_t + \Vmat \Vcoeff_t}^{\top} \Delta_t \Umat \\
    = & \stepsize\popUup(\Ucoeff_t)\Ucoeff_t^{\top} \Umat^{\top} \Delta_t \Umat 
    + \stepsize\popUup(\Ucoeff_t)\Vcoeff_t^{\top} \Vmat^{\top} \Delta_t \Umat \\
    = & \underbrace{\stepsize\popUup(\Ucoeff_t)\Ucoeff_t^{\top} \Umat^{\top} \Delta_t \Umat - \stepsize \trueUcoeff \Umat^{\top} \Delta_t \Umat}_{\text{III}}
    + \underbrace{\stepsize\popUup(\Ucoeff_t)\Vcoeff_t^{\top} \Vmat^{\top} \Delta_t \Umat}_{\text{IV}}
    + \underbrace{\stepsize \trueUcoeff \Umat^{\top} \Delta_t \Umat}_{\text{V}},
\end{align*}
and
\begin{align*}
    & \hspace{- 4 em} \stepsize\Umat^{\top} \parenth{ \bG_t -  \bG_t^n} \popUup(\Ucoeff_t)^{\top}  \\
    = & \underbrace{\stepsize\Umat^{\top}  \Delta_t \Umat \Ucoeff_t \popUup(\Ucoeff_t)^{\top} - \stepsize \Umat^{\top} \Delta_t \Umat \trueUcoeff}_{\text{VI}}
    + \underbrace{\stepsize\Umat^{\top}  \Delta_t \Vmat \Vcoeff_t \popUup(\Ucoeff_t)^{\top}}_{\text{VII}}
    + \underbrace{\stepsize \Umat^{\top} \Delta_t \Umat \trueUcoeff}_{\text{VIII}}.
\end{align*}
Clearly, our target can be bounded by bounding the eight terms, marked from I to VIII. Note that the spectral norms of the terms (1) III and VI are the same,  (2) IV and VII are the same, and (3)V and VIII are the same, which can be upper bounded as follows:
\begin{align*}
    \text{III \& VI:}& \quad \quad \quad \vecnorm{\stepsize\popUup(\Ucoeff_t)\Ucoeff_t^{\top} \Umat^{\top} \Delta_t \Umat - \stepsize \trueUcoeff \Umat^{\top} \Delta_t \Umat}{2} 
    \leq  \eta \vecnorm{\popUup(\Ucoeff_t)\Ucoeff_t^{\top} - \trueUcoeff}{2} \vecnorm{\Delta_t}{2}, \\
    \text{IV \& VII:}& \quad \quad \quad \vecnorm{\stepsize\popUup(\Ucoeff_t)\Vcoeff_t^{\top} \Vmat^{\top} \Delta_t \Umat}{2} 
    \leq \eta \vecnorm{\popUup(\Ucoeff_t)\Vcoeff_t^{\top}}{2}\vecnorm{\Delta_t}{2}, \\
    \text{V \& VIII:}& \quad \quad \quad\vecnorm{\stepsize \trueUcoeff \Umat^{\top} \Delta_t \Umat }{2}  \leq \eta \vecnorm{\trueUcoeff}{2} \vecnorm{\Delta_t}{2}.
\end{align*}
Lastly,  consider the II term, we have the following bound:
\begin{equation}
\label{eqn:bounding-II-term}
\begin{split}
    \vecnorm{\stepsize^2 \Umat^{\top} \parenth{ \bG_t -  \bG_t^n}\parenth{ \bG_t -  \bG_t^n}^{\top} \Umat}{2}
    \leq &\stepsize^2 \vecnorm{\Delta_t \fitMat_t \fitMat_t^{\top} \Delta_t}{2} \\
    \leq & \stepsize^2 \parenth{\vecnorm{\Ucoeff_t \Ucoeff_t^{\top}}{2} + \vecnorm{\Vcoeff_t \Vcoeff_t^{\top}}{2} + 2\vecnorm{\Ucoeff_t \Vcoeff_t^{\top}}{2} } \vecnorm{\Delta_t}{2}^2\\
    \leq & \frac{1}{100}\eta \vecnorm{\Delta_t}{2}^2,
\end{split}
\end{equation}
where the last inequality holds by assuming $\rho \leq 0.1$.
Putting all the above results together, we obtain that
\begingroup
\allowdisplaybreaks
\begin{align*}
    & \hspace{- 4 em} \vecnorm{\Ucoeff_{t+1}\Ucoeff_{t+1}^\top - \trueUcoeff}{2} \\
    \overset{(1)}{\leq} & \underbrace{\vecnorm{\mathcal{M}(\Ucoeff_t) \mathcal{M}(\Ucoeff_t)^\top - \trueUcoeff}{2}}_{\text{I}}
    + \underbrace{2\eta \vecnorm{\trueUcoeff}{2} \vecnorm{\Delta_t}{2}}_{\text{V + VIII}}
    + \underbrace{\frac{1}{100}\eta \vecnorm{\Delta_t}{2}^2}_{\text{II}} \\
    & + \underbrace{2\eta \vecnorm{\popUup(\Ucoeff_t)\Ucoeff_t^T - \trueUcoeff}{2} \vecnorm{\Delta_t}{2}}_{\text{III +  VI }}
    + \underbrace{2\eta \vecnorm{\popUup(\Ucoeff_t)\Vcoeff_t^T}{2}\vecnorm{\Delta_t}{2}}_{\text{IV +  VII }}\\
    \overset{(2)}{\leq} & \underbrace{\parenth{1 - \frac{7}{10}\eta\sigma_r } \vecnorm{\Ucoeff_{t}\Ucoeff_{t}^\top - \trueUcoeff}{2}}_{\text{I}} + \underbrace{\frac{1}{50} \vecnorm{\Delta_t}{2}}_{\text{V + VIII}}
    + \underbrace{\frac{1}{100}\eta \vecnorm{\Delta_t}{2}^2}_{\text{II}}
    + \underbrace{4 \eta D_t \vecnorm{\Delta_t}{2}}_{\text{III +  VI + IV +  VII }} \\
    \overset{(3)}{\leq} & \parenth{1 - \frac{7}{10}\eta\sigma_r }\vecnorm{\Ucoeff_{t}\Ucoeff_{t}^\top - \trueUcoeff}{2} + \frac{1}{10} \vecnorm{\Delta_t}{2} \\
    \overset{(4)}{\leq} & \parenth{1 - \frac{7}{10}\eta\sigma_r} \vecnorm{\Ucoeff_{t}\Ucoeff_{t}^\top - \trueUcoeff}{2}
    + \sqrt{\frac{k d \log d}{n} }  D_t + \frac{4}{10} \sqrt{\frac{d \log d }{n}}\sigma
\end{align*}
\endgroup
where inequality $(2)$ is obtained by the non-expansion property of population update (cf. equations~\eqref{eqn:pop-non-expansion1} and~\eqref{eqn:pop-non-expansion2})
; inequality $(3)$ is obtained by the fact that $\frac{1}{100} \eta \vecnorm{\Delta_t}{2} < 0.0001$, and $4 \eta D_t < 0.04$;
inequality $(4)$ is obtained by plugging in the relaxation of $\vecnorm{\Delta_t}{2}$ (cf. equation~\eqref{eqn:tight-delta_bound}) and organizing according to $D_t$ and $\sigma$.

That is, we proved the equation~\eqref{eqn:contraction-SS-main-text} in the Lemma~\ref{lemma:sample-contraction-T-small}, namely, we have
\begin{align}
\label{eqn:contraction-SS}
    \vecnorm{\Ucoeff_{t+1}\Ucoeff_{t+1}^\top - \trueUcoeff}{2} \leq \parenth{1 - \frac{7}{10}\eta\sigma_r} \vecnorm{\Ucoeff_{t}\Ucoeff_{t}^\top - \trueUcoeff}{2}
    + \sqrt{\frac{k d \log d}{n} }  D_t + \frac{4}{10} \sqrt{\frac{d \log d }{n}}\sigma.
\end{align}
This indicates a contraction with respect to $D_t$. Since $\sqrt{\frac{k d \log d}{n} } \leq \frac{1}{1000\kappa}=0.1\eta\sigma_r$ by the choice of $n$, we find that
\begin{align}
\label{eqn:contraction-SS-wr-D_t}
    \vecnorm{\Ucoeff_{t+1}\Ucoeff_{t+1}^\top - \trueUcoeff}{2} \leq \parenth{1 - \frac{6}{10}\eta\sigma_r} D_t
     + \frac{4}{10} \sqrt{\frac{d \log d }{n}}\sigma.
\end{align}
From the above result, we can verify that
\begin{align}
\label{eqn:contraction-SS-wr-A_t}
    \vecnorm{\Ucoeff_{t+1}\Ucoeff_{t+1}^{\top} - \trueUcoeff}{2}  \leq \parenth{1 - \frac{6}{10} \eta \sigma_r}\parenth{D_t - 50 \kappa \sqrt{\frac{d \log d }{n}}\sigma} +  50 \kappa \sqrt{\frac{ d \log d }{n}}\sigma.
\end{align}
\paragraph{Upper bound for $\vecnorm{\Ucoeff_t \Vcoeff_t^{\top} }{2}$:} 
According to equation~\eqref{eqn:St+1Tt+1}, we have
\begin{align*}
    \Ucoeff_{t+1}\Vcoeff_{t+1}^{\top} 
    = & \underbrace{\popUup(\Ucoeff_t) \popVup(\Vcoeff_t)^{\top}}_{\text{I}'}
    + \underbrace{\stepsize^2 \Umat^{\top} \parenth{ \bG_t -  \bG_t^n}\parenth{ \bG_t -  \bG_t^n}^T \Vmat}_{\text{II}'} \\
    & + \stepsize\popUup(\Ucoeff_t)\parenth{ \bG_t -  \bG_t^n}^{\top} \Vmat + \stepsize\Umat^{\top} \parenth{ \bG_t -  \bG_t^n} \popVup(\Vcoeff_t)^{\top},
\end{align*}
where we can expand $\stepsize\popUup(\Ucoeff_t)\parenth{ \bG_t -  \bG_t^n}^{\top} \Vmat$ and $\stepsize\Umat^T \parenth{ \bG_t -  \bG_t^n} \popVup(\Vcoeff_t)^{\top}$ as follows:
\begin{align*}
    \stepsize\popUup(\Ucoeff_t)\parenth{ \bG_t -  \bG_t^n}^{\top} \Vmat = &  \stepsize\popUup(\Ucoeff_t)\parenth{\Umat \Ucoeff_t + \Vmat \Vcoeff_t}^{\top} \Delta_t \Vmat  \\
    = & \underbrace{\stepsize\popUup(\Ucoeff_t)\Ucoeff_t^{\top} \Umat^{\top} \Delta_t \Vmat - \eta \trueUcoeff \Umat^T \Delta_t \Vmat}_{\text{III}'} 
    + \underbrace{\stepsize\popUup(\Ucoeff_t)\Vcoeff_t^{\top} \Vmat^{\top} \Delta_t \Vmat}_{\text{IV}'}  \\
    + & \underbrace{\eta \trueUcoeff \Umat^{\top} \Delta_t \Vmat}_{\text{V}'}, \\
     \stepsize\Umat^{\top} \parenth{ \bG_t -  \bG_t^n} \popVup(\Vcoeff_t)^{\top} 
    = & \stepsize\Umat^{\top} \Delta_t \parenth{\Umat \Ucoeff_t + \Vmat \Vcoeff_t} \popVup(\Vcoeff_t)^{\top}  \\ 
    = & \underbrace{\stepsize\Umat^{\top}  \Delta_t \Umat \Ucoeff_t \popVup(\Vcoeff_t)^{\top}}_{\text{VI}'}
    + \underbrace{\stepsize\Umat^{\top}  \Delta_t \Vmat \Vcoeff_t \popVup(\Vcoeff_t)^{\top}}_{\text{VII}'}.
\end{align*}
Clearly, our target upper bound for $\vecnorm{\Ucoeff_{t+1}\Vcoeff_{t+1}^{\top}}{2}$ can be obtained by bounding the seven terms: I' to VII'. Specifically, direct application of inequalities with operator norms leads to
\begin{align*}
    \text{III}': & \quad \quad \quad \vecnorm{\stepsize\popUup(\Ucoeff_t)\Ucoeff_t^{\top} \Umat^{\top} \Delta_t \Vmat - \stepsize \trueUcoeff \Umat^{\top} \Delta_t \Vmat}{2} 
    \leq  \eta \vecnorm{\popUup(\Ucoeff_t)\Ucoeff_t^{\top} - \trueUcoeff}{2} \vecnorm{\Delta_t}{2}, \\
    \text{IV}': & \quad \quad \quad \vecnorm{\stepsize\popUup(\Ucoeff_t)\Vcoeff_t^{\top} \Vmat^{\top} \Delta_t \Vmat}{2} 
    \leq \eta \vecnorm{\popUup(\Ucoeff_t)\Vcoeff_t^{\top}}{2}\vecnorm{\Delta_t}{2}, \\
    \text{V}': & \quad \quad \quad \vecnorm{\stepsize \trueUcoeff \Umat^{\top} \Delta_t \Vmat}{2} 
    \leq \eta \vecnorm{\trueUcoeff}{2}\vecnorm{\Delta_t}{2}, \\
    \text{VI}': & \quad \quad \quad \vecnorm{\stepsize\Umat^{\top}  \Delta_t \Umat \Ucoeff_t \popVup(\Vcoeff_t)^{\top}}{2} 
    \leq \eta \vecnorm{\popVup(\Vcoeff_t)\Ucoeff_t^{\top}}{2}\vecnorm{\Delta_t}{2}, \\
    \text{VII}': & \quad \quad \quad \vecnorm{\stepsize\Umat^T  \Delta_t \Vmat \Vcoeff_t \popVup(\Vcoeff_t)^T}{2} 
    \leq \eta \vecnorm{\popVup(\Vcoeff_t)\Vcoeff_t^{\top}}{2}\vecnorm{\Delta_t}{2}.
\end{align*}
Lastly,  the II term is bounded as in Equation \eqref{eqn:bounding-II-term}, namely, we have
\begin{align*}
    \vecnorm{\stepsize^2 \Umat^{\top} \parenth{ \bG_t -  \bG_t^n}\parenth{ \bG_t -  \bG_t^n}^{\top} \Vmat}{2}
    \leq  \frac{1}{100}\eta \vecnorm{\Delta_t}{2}^2.
\end{align*}
Collecting the above results, we find that
\begin{align*}
    & \hspace{- 2 em} \vecnorm{\Ucoeff_{t+1}\Vcoeff_{t+1}^{\top} }{2} \\
    \leq & \underbrace{\vecnorm{\popUup(\Ucoeff_t) \popVup(\Vcoeff_t)^\top }{2}}_{\text{I}'}
    + \underbrace{\frac{1}{100}\eta \vecnorm{\Delta_t}{2}^2}_{\text{II}'}
    + \underbrace{\eta \vecnorm{\popUup(\Ucoeff_t)\Ucoeff_t^{\top} - \trueUcoeff}{2} \vecnorm{\Delta_t}{2} }_{\text{III}'} 
    + \underbrace{\eta \vecnorm{\trueUcoeff}{2}\vecnorm{\Delta_t}{2}}_{\text{V}'}\\
    & + \underbrace{\eta \vecnorm{\popUup(\Ucoeff_t)\Vcoeff_t^{\top}}{2}\vecnorm{\Delta_t}{2}}_{\text{IV}'}
    + \underbrace{\eta \vecnorm{\popVup(\Vcoeff_t)\Ucoeff_t^{\top}}{2}\vecnorm{\Delta_t}{2}}_{\text{VI}'}
    + \underbrace{\eta \vecnorm{\popVup(\Vcoeff_t)\Vcoeff_t^{\top}}{2}\vecnorm{\Delta_t}{2}}_{\text{VII}'}\\
    \overset{(1)}{\leq} & \underbrace{\parenth{1 - \eta \sigma_r } \vecnorm{\Ucoeff_t \Vcoeff_t^\top}{2} }_{\text{I}'}
    + \underbrace{\frac{1}{100}\eta \vecnorm{\Delta_t}{2}^2}_{\text{II}'}
    + \underbrace{\frac{1}{100} \vecnorm{\Delta_t}{2}}_{\text{V}}
    + \underbrace{5 \eta D_t \vecnorm{\Delta_t}{2}}_{\text{III}' + \text{IV}' + \text{VI}'+ \text{VII}'}\\
    \overset{(2)}{\leq} &  \parenth{1 - \eta \sigma_r } \vecnorm{\Ucoeff_t \Vcoeff_t^\top}{2}
    + \frac{1}{10} \vecnorm{\Delta_t}{2}\\
    \overset{(3)}{\leq} & \parenth{1 - \eta\sigma_r} \vecnorm{\Ucoeff_t \Vcoeff_t^\top}{2}
    + \sqrt{\frac{k d \log d}{n} }  D_t + \frac{4}{10} \sqrt{\frac{d \log d }{n}}\sigma,
\end{align*}
where inequality $(1)$ is obtained by the non-expansion property of population update (cf. equations~\eqref{eqn:pop-non-expansion1} and~\eqref{eqn:pop-non-expansion2});
inequality $(2)$ is obtained by the fact that $\frac{1}{100} \eta \vecnorm{\Delta_t}{2} < 0.001$ and $5 \eta D_t < 0.05$.
In summary, we have
\begin{align}
\label{eqn:contraction-ST}
    \vecnorm{\Ucoeff_{t+1}\Vcoeff_{t+1}^{\top} }{2} \leq \parenth{1 - \eta \sigma_r } \vecnorm{\Ucoeff_t \Vcoeff_t^\top}{2}
    + \sqrt{\frac{k d \log d}{n} }  D_t + \frac{4}{10} \sqrt{\frac{d \log d }{n}}\sigma.
\end{align}
With similar treatment as in $\vecnorm{\Ucoeff_t \Ucoeff_t^{\top} - \trueUcoeff}{2}$, we have the following contraction result with respect to $D_t$:
\begin{align}
\label{eqn:contraction-ST-wr-D_t}
    \vecnorm{\Ucoeff_{t+1}\Vcoeff_{t+1}^{\top} }{2} \leq \parenth{1 - \frac{9}{10}\eta\sigma_r} D_t 
    + \frac{4}{10} \sqrt{\frac{d \log d }{n}}\sigma.
\end{align}
Given the above result, we can verify that
\begin{align}
\label{eqn:contraction-ST-wr-A_t}
    \vecnorm{\Ucoeff_{t+1}\Vcoeff_{t+1}^{\top}}{2}  \leq \parenth{1 - \frac{9}{10} \eta \sigma_r}\parenth{D_t - 50 \kappa \sqrt{\frac{d \log d }{n}}\sigma} +  50 \kappa \sqrt{\frac{d \log d }{n}}\sigma.
\end{align}
\paragraph{Upper bound for $\vecnorm{\Vcoeff_t \Vcoeff_t^{\top} }{2}$:}
According to equation~\eqref{eqn:Tt+1Tt+1} and similar deductions as in previous bounds for $\vecnorm{\Ucoeff_t \Ucoeff_t^{\top} - \trueUcoeff}{2}$ and $\vecnorm{\Ucoeff_t \Vcoeff_t^{\top} }{2}$, we have
\begin{align*}
    \Vcoeff_{t+1}\Vcoeff_{t+1}^{\top} 
    = & \underbrace{\popVup(\Vcoeff_t) \popVup(\Vcoeff_t)^{\top}}_{\text{I}''}
    + \underbrace{\stepsize^2 \Vmat^{\top} \parenth{ \bG_t -  \bG_t^n}\parenth{ \bG_t -  \bG_t^n}^{\top} \Vmat}_{\text{II}''} \\
    & + \stepsize\popVup(\Vcoeff_t)\parenth{ \bG_t -  \bG_t^n}^{\top} \Vmat + \stepsize\Vmat^{\top} \parenth{ \bG_t -  \bG_t^n} \popVup(\Vcoeff_t)^{\top},
\end{align*}
where the following expansions hold:
\begin{align*}
     \stepsize\popVup(\Vcoeff_t)\parenth{ \bG_t -  \bG_t^n}^{\top} \Vmat 
    = & \underbrace{\stepsize\popVup(\Vcoeff_t)\Ucoeff_t^{\top} \Umat^{\top} \Delta_t \Vmat }_{\text{III}''}
    + \underbrace{\stepsize\popVup(\Vcoeff_t)\Vcoeff_t^{\top} \Vmat^{\top} \Delta_t \Vmat}_{\text{IV}''}, \\
    \stepsize\Vmat^{\top} \parenth{ \bG_t -  \bG_t^n} \popVup(\Vcoeff_t)^{\top}
    = & \underbrace{\stepsize\Vmat^{\top}  \Delta_t \Umat \Ucoeff_t \popVup(\Vcoeff_t)^{\top}}_{\text{V}''}
    + \underbrace{\stepsize\Vmat^{\top} \Delta_t \Vmat \Vcoeff_t \popVup(\Vcoeff_t)^{\top}}_{\text{VI}''}.
\end{align*}
Given the formulations of the terms $\text{I}''$-$\text{VI}''$, we find that
\begin{align*}
    \text{III}'' \& \text{V}'':& \quad \quad \quad \vecnorm{\stepsize\popVup(\Vcoeff_t)\Ucoeff_t^{\top} \Umat^{\top} \Delta_t \Vmat}{2} 
    \leq  \eta \vecnorm{\popVup(\Vcoeff_t)\Ucoeff_t^{\top} }{2} \vecnorm{\Delta_t}{2}, \\
    \text{IV}'' \& \text{VI}'':& \quad \quad \quad \vecnorm{\stepsize\popVup(\Vcoeff_t)\Vcoeff_t^{\top} \Vmat^{\top} \Delta_t \Vmat}{2} 
    \leq \eta \vecnorm{\popVup(\Vcoeff_t)\Vcoeff_t^{\top}}{2}\vecnorm{\Delta_t}{2}, \\
    \text{II}'':& \quad \quad \quad\vecnorm{\stepsize^2 \Umat^{\top} \parenth{ \bG_t -  \bG_t^n}\parenth{ \bG_t -  \bG_t^n}^{\top} \Umat}{2}  \leq \frac{1}{100}\eta \vecnorm{\Delta_t}{2}^2.
\end{align*}
Assume that $\vecnorm{\Vcoeff_{t}\Vcoeff_{t}^{\top}}{2} = z D_t$ for $0 < z \leq 1$. Note that, $z$ is not necessarily a constant. For notation simplicity we use the short hand that $\epsilon_{stat} = \sqrt{\frac{d \log d }{n}}\sigma$.
With the choice of $n$ and equation~\eqref{eqn:tight-delta_bound}, we have $\vecnorm{\Delta_t}{2} \leq \eta \sigma_r D_t + 4\epsilon_{stat}$. Therefore, we obtain that
\begin{align*}
    &\vecnorm{\Vcoeff_{t+1}\Vcoeff_{t+1}^{\top}}{2} \\
    \overset{(1)}{\leq} & \underbrace{\vecnorm{\popVup(\Vcoeff_t) \popVup(\Vcoeff_t)^{\top}}{2}}_{\text{I}''}
    + \underbrace{\frac{1}{100}\eta \vecnorm{\Delta_t}{2}^2}_{\text{II}''}
    + \underbrace{2\eta \vecnorm{\popVup(\Vcoeff_t)\Ucoeff_t^{\top} }{2} \vecnorm{\Delta_t}{2}}_{\text{III}'' +  \text{V}'' }
    + \underbrace{2\eta \vecnorm{\popVup(\Vcoeff_t)\Vcoeff_t^{\top}}{2}\vecnorm{\Delta_t}{2}}_{\text{IV}'' +  \text{VI}'' }\\
    \overset{(2)}{\leq} & \underbrace{\parenth{z - z^2\eta D_t  + 2z\eta \vecnorm{\trueVcoeff}{2} + 4 \eta \parenth{\eta \sigma_r D_t + 4\epsilon_{stat}}}D_t}_{\text{I}'' + \text{III}'' +  \text{V}'' + \text{IV}'' +  \text{VI}''} + \underbrace{\frac{1}{100}\eta \parenth{\eta \sigma_r D_t + 4\epsilon_{stat}}^2 }_{\text{II}}\\
    \overset{(3)}{\leq} & \parenth{1 - \eta D_t  + 2\eta \vecnorm{\trueVcoeff}{2} + 4 \eta \parenth{\eta \sigma_r D_t + 4\epsilon_{stat}}}D_t + \frac{1}{100}\eta \parenth{\eta \sigma_r D_t + 4\epsilon_{stat}}^2, 
\end{align*}
where inequality $(2)$ is obtained by the non-expansion property of population update (cf. equations~\eqref{eqn:pop-non-expansion1} and~\eqref{eqn:pop-non-expansion2}) and the assumption on $\vecnorm{\Delta_t}{2}$ (cf. equation~\eqref{eqn:tight-delta_bound}). For inequality $(3)$, observe that the above quantity is a quadratic formula with respect to $z$, and the maximum is taken when $z = \frac{1+2\eta\vecnorm{\trueVcoeff}{2}}{2\eta D_t} > 1$. Hence we can just safely plug-in $z = 1$. Now, we arrange by organizing according to $D_t$ and $\epsilon_{stat}$ and obtain that
\begin{align*}
    &\vecnorm{\Vcoeff_{t+1}\Vcoeff_{t+1}^{\top}}{2} \\
    \overset{(1)}{\leq} & \parenth{1 - \eta D_t  + 2\eta \epsilon_{stat} + 4 \eta^2 \sigma_r D_t + 16 \eta \epsilon_{stat}}D_t + \frac{1}{100}\eta \parenth{\eta^2 \sigma_r^2 D_t^2 + 16 \epsilon_{stat}^2 + 8 \epsilon_{stat}\eta \sigma_r D_t}\\
    \overset{(2)}{=} & \parenth{1 - \eta D_t  +  4 \eta^2 \sigma_r D_t + 0.01 \eta^3 \sigma_r^2 D_t}D_t +  \parenth{0.16 \epsilon_{stat} + 0.08 \eta \sigma_r D_t + 18 D_t} \eta \epsilon_{stat}\\
    \overset{(3)}{\leq} & \parenth{1 - 0.9 \eta D_t}D_t +  \parenth{0.16 \epsilon_{stat} + 19 D_t} \eta \epsilon_{stat},
\end{align*}
where inequality $(1)$ is obtained by $\vecnorm{\trueVcoeff}{2} \leq \epsilon_{stat}$, and inequality $(2)$ is obtained by organizing according to $D_t$ and $\sigma$.

For notation simplicity we introduce $A_t = D_t - 50 \epsilon_{stat}$, and hence $D_t = A_t + 50 \epsilon_{stat}$ where $\epsilon_{stat} = \kappa \sqrt{\frac{ d \log d }{n}}\sigma$. With some algebraic manipulations, we have
\begin{align*}
    \vecnorm{\Vcoeff_{t+1}\Vcoeff_{t+1}^{\top}}{2}  \leq \parenth{1 - 0.9 \eta A_t}A_t +  50 \kappa \sqrt{\frac{ d \log d }{n}}\sigma.
\end{align*}
Furthermore, from equations~\eqref{eqn:contraction-SS-wr-A_t} and~\eqref{eqn:contraction-ST-wr-A_t}, we have
\begin{align*}
    D_{t+1}  \leq \parenth{1 - 0.5 \eta A_t}A_t +  50 \kappa \sqrt{\frac{ d \log d }{n}}\sigma.
\end{align*}
Putting all these results together yields that
\begin{align*}
    A_{t+1}  \leq \parenth{1 - 0.5 \eta A_t}A_t.
\end{align*}
This completes the proof of the Lemma~\ref{lemma:sample-contraction-T-small}.
\end{proof}
\vspace{20pt}
In the lemma above we use $\sqrt{\frac{k d \log d}{n} }  D_t$ to demonstrate how sample iteration conerges to population iteration as  $n$ increases.
However we do not need $\vecnorm{\Delta_t}{2} \leq \sqrt{\frac{k d \log d}{n} }  D_t$ 
when $n$ is extremely large such that $\sqrt{\frac{k d \log d}{n} } \gg \eta \sigma_r$.
This will lead to a waste of sample and lead to sub-optimal sample complexity result.
Hence we introduce the following corollary in order to deal with the scenario .
\begin{corollary}
\label{corollary:extension-to-contraction-lemma-with-large-n}
Consider under the same setting and using the same notation as Lemma~\ref{lemma:sample-contraction-T-small}.
Suppose
$
\vecnorm{\Delta_t}{2} \leq \eta \sigma_r  D_t + 4 \sqrt{\frac{d \log d}{n}} \sigma,
$
Then
\begin{align*}
    \vecnorm{\Ucoeff_{t+1}\Ucoeff_{t+1}^\top - \trueUcoeff}{2} \leq&
    \parenth{1 - \frac{7}{10}\eta\sigma_r} \vecnorm{\Ucoeff_{t}\Ucoeff_{t}^\top - \trueUcoeff}{2}
    + 0.1 \eta \sigma_r  D_t 
    + \frac{4}{10} \sqrt{\frac{d \log d }{n}}\sigma,\\
    \vecnorm{\Ucoeff_{t+1}\Vcoeff_{t+1}^T }{2} \leq& 
    \parenth{1 - \eta \sigma_r } \vecnorm{\Ucoeff_t \Vcoeff_t^\top}{2}
    +  0.1 \eta \sigma_r  D_t 
    + \frac{4}{10} \sqrt{\frac{d \log d }{n}}\sigma.
\end{align*}
Moreover, denote $\epsilon_{stat} = \kappa \sqrt{\frac{d \log d }{n}}\sigma$. Then we have
\begin{align}
\label{eqn:sub-linear-contraction-in-corollary}
    \parenth{D_{t+1} -  50 \epsilon_{stat}}  \leq \bracket{1 - \frac{1}{2} \eta \parenth{D_t -  50 \epsilon_{stat}}}\parenth{D_t -  50 \epsilon_{stat}}.
\end{align}
\end{corollary}
\begin{proof}
As in the proof of Lemma~\ref{lemma:sample-contraction-T-small}, if we replace the condition 
$\vecnorm{\Delta_t}{2} \leq 10 \sqrt{\frac{k d \log d}{n} }  D_t + 4 \sqrt{\frac{d \log d}{n}} \sigma,$
by
$\vecnorm{\Delta_t}{2} \leq \eta \sigma_r  D_t + 4 \sqrt{\frac{d \log d}{n}} \sigma,$,
nothing changed except the notation, since equation \eqref{equation:condition-of-kdlogd/n} is satisfied directly.
Hence this corollory can be proved using the exactly same argument as in the proof in Lemma~\ref{lemma:sample-contraction-T-small}.
\end{proof}

\vspace{30pt}
Note that Lemma~\ref{lemma:sample-contraction-T-small} is established for $D_t > 50 \epsilon_{stat}$. 
To complete the proof of our main theorem, we want to make sure that $D_t$ do not expand too much after we reaches the statistical accuracy.

\begin{lemma}
\label{lemma:non-expansion-after-epsilon-stat}
Consider the same setting as Lemma \ref{lemma:sample-contraction-T-small}, except that $D_t \leq 50  \epsilon_{stat}$. We claim that $D_{t+1} \leq 100 \epsilon_{stat}$.
\end{lemma}
\begin{proof}
The proof of this Lemma is a simple extension using the proof of Lemma \ref{lemma:sample-contraction-T-small}.
As in the proof of Lemma \ref{lemma:sample-contraction-T-small}, we know that
\begin{align*}
    \vecnorm{\Ucoeff_{t+1}\Ucoeff_{t+1}^\top - \trueUcoeff}{2} 
    \leq \parenth{1 - \frac{7}{10}\eta\sigma_r }\vecnorm{\Ucoeff_{t}\Ucoeff_{t}^\top - \trueUcoeff}{2} + \frac{1}{10} \vecnorm{\Delta_t}{2}.
\end{align*}
From the hypothesis, $\vecnorm{\Delta_t}{2} \leq D_t + \epsilon_{stat} \leq 51 \epsilon_{stat}$. Hence, we have 
\begin{align*}
    \vecnorm{\Ucoeff_{t+1}\Ucoeff_{t+1}^\top - \trueUcoeff}{2} 
    \leq 100 \epsilon_{stat}.
\end{align*}
Similarly for $\vecnorm{\Ucoeff_{t+1}\Vcoeff_{t+1}^{\top} }{2}$, we have
\begin{align*}
    \vecnorm{\Ucoeff_{t+1}\Vcoeff_{t+1}^{\top} }{2} 
    \leq   \parenth{1 - \eta \sigma_r } \vecnorm{\Ucoeff_t \Vcoeff_t^\top}{2}
    + \frac{1}{10} \vecnorm{\Delta_t}{2} \leq 100 \epsilon_{stat}.
\end{align*}
Finally, for $\vecnorm{\Vcoeff_{t+1}\Vcoeff_{t+1}^{\top}}{2}$, we find that
\begin{align*}
    &\vecnorm{\Vcoeff_{t+1}\Vcoeff_{t+1}^{\top}}{2} \\
    \overset{(1)}{\leq} & \underbrace{\vecnorm{\popVup(\Vcoeff_t) \popVup(\Vcoeff_t)^{\top}}{2}}_{\text{I}}
    + \underbrace{\frac{1}{100}\eta \vecnorm{\Delta_t}{2}^2}_{\text{II}}
    + \underbrace{2\eta \vecnorm{\popVup(\Vcoeff_t)\Ucoeff_t^{\top}}{2} \vecnorm{\Delta_t}{2}}_{\text{III +  V }}
    + \underbrace{2\eta \vecnorm{\popVup(\Vcoeff_t)\Vcoeff_t^{\top}}{2}\vecnorm{\Delta_t}{2}}_{\text{IV +  VI }}\\
    \overset{(2)}{\leq} & \vecnorm{\Vcoeff_t \Vcoeff_t^\top}{2} \parenth{1 - \eta \vecnorm{\Vcoeff_t \Vcoeff_t^\top}{2} + 2\eta \vecnorm{\trueVcoeff}{2}} + 5 \eta \cdot 50\epsilon_{stat} \cdot 51 \epsilon_{stat} \\
    \overset{(3)}{\leq} & \parenth{1 +  300 \eta \epsilon_{stat}} 50\epsilon_{stat}\\
    \overset{(4)}{\leq} & 100 \epsilon_{stat}
\end{align*}
where inequality (1) is deducted in the proof of Lemma \ref{lemma:sample-contraction-T-small}; inequality (2) is by relaxing term I using Equation \eqref{eqn:pop_contraction}, relaxing $\vecnorm{\Delta_t}{2} \leq 51 \epsilon_{stat}$, and grouping all other terms; inequality (3) is by the assumption that $\vecnorm{\trueVcoeff}{2} \leq \epsilon_{stat}$; inequality (4) is by the choice of $n$ such that $\epsilon_{stat} \leq 0.1$.

Putting all the results together, we obtain the conclusion of Lemma~\ref{lemma:non-expansion-after-epsilon-stat}.
\end{proof}

\subsection{Proof of Theorem~\ref{theorem:sample-convergence-T-small}}
\label{sec:proof-main-theorem}
Our proof is divided into verifying claim (a) and claim (b).
\paragraph{Proof for claim (a) with the linear convergence:} From the result of Lemma~\ref{lemma:sample-contraction-T-small}, we know that
\begin{align*}
    \vecnorm{\Ucoeff_{t+1}\Ucoeff_{t+1}^\top - \trueUcoeff}{2} \leq \parenth{1 - \frac{7}{10}\eta\sigma_r} \vecnorm{\Ucoeff_{t}\Ucoeff_{t}^\top - \trueUcoeff}{2}
    + \sqrt{\frac{k d \log d}{n} }  D_t + \frac{4}{10} \sqrt{\frac{d \log d }{n}}\sigma.
\end{align*}
We consider $\vecnorm{\Ucoeff_{t}\Ucoeff_{t}^\top - \trueUcoeff}{2} > 1000 \sqrt{\frac{k \kappa^2 d \log d}{n} }  \sigma_r > 1000 \sqrt{\frac{k \kappa^2 d \log d}{n} }  D_t$. 
Note that $\eta \sigma_r = 0.01/\kappa$.
Then $\sqrt{\frac{k d \log d}{n} }  D_t \leq 0.1\eta \sigma_r \vecnorm{\Ucoeff_{t}\Ucoeff_{t}^\top - \trueUcoeff}{2}$ 
and $\frac{4}{10} \sqrt{\frac{d \log d }{n}}\sigma < 0.1\eta \sigma_r \vecnorm{\Ucoeff_{t}\Ucoeff_{t}^\top - \trueUcoeff}{2}$ by the choice of the constants in the lower bound of $n$.
Hence, when $\vecnorm{\Ucoeff_{t}\Ucoeff_{t}^\top - \trueUcoeff}{2} > 1000 \sqrt{\frac{k \kappa^2 d \log d}{n} }  \sigma_r$, we find that
\begin{equation}
\label{eqn:linear-contraction-in-proof-main-thm}
    \vecnorm{\Ucoeff_{t+1}\Ucoeff_{t+1}^\top - \trueUcoeff}{2} \leq \parenth{1 - \frac{5}{10}\eta\sigma_r} \vecnorm{\Ucoeff_{t}\Ucoeff_{t}^\top - \trueUcoeff}{2}.
\end{equation}
We now have constant contraction for one iteration.
We can invoke the Lemma  \ref{lemma:concentration-Aepsilon} for once,
Lemma \ref{lemma:concentration-standard} for $t$ iterations,
and take the union bounds,
to quantify the probability that equation \eqref{eqn:tight-delta_bound} holds for all iteration $t$.
Shortly we will show that this probability is at least $1 - d^{-c}$ for some universal constant $c$.
But first we need to know how large we need the number of iterations $t$ to be.
Suppose equation \eqref{eqn:tight-delta_bound} holds for all iteration $t$, then
$$
\vecnorm{\Ucoeff_{t}\Ucoeff_{t}^\top - \trueUcoeff}{2} 
\leq \parenth{1 - \frac{5}{10}\eta\sigma_r}^t \vecnorm{\Ucoeff_{0}\Ucoeff_{0}^\top - \trueUcoeff}{2} \leq \parenth{1 - \frac{5}{10}\eta\sigma_r}^t 0.1 \sigma_r,
$$
where the final inequality holds by simply plugging in the initialization condition. 
After at most $t = \frac{1}{\log \frac{1}{1-0.005/\kappa}} \cdot \log \frac{1}{10000 \sqrt{\frac{k \kappa^2 d \log d}{n} } }$ iterations, $\vecnorm{\Ucoeff_{t}\Ucoeff_{t}^\top - \trueUcoeff}{2} < 1000 \sqrt{\frac{k \kappa^2 d \log d}{n} }  \sigma_r$. 
Since $\frac{1}{\log \frac{1}{1-0.005/\kappa}} \leq 1.1$, we further simplify this to $t > \log \frac{n}{k \kappa^2 d \log d}$.
As a consequence, we claim that after $t = \ceil{\log \frac{n}{k \kappa^2 d \log d}}$ iterations,
$\vecnorm{\Ucoeff_{t}\Ucoeff_{t}^\top - \trueUcoeff}{2} < C \sqrt{\frac{k \kappa^2 d \log d}{n} }  \sigma_r$ for some universal constant $C$.

Now the remaining task is to show that equation \eqref{eqn:tight-delta_bound} holds for all iteration $t$ with probability at least $1 - d^{-c}$.
If $n$ is not too large, i.e. $n < d^{c_5}$ for some constant $c_5$, then we invoke equation~\ref{eqn:lemma-concentration-standard-AUA} in Lemma~\ref{lemma:concentration-standard} for $t$ iterations. This holds with probability at least $1 - t d^{-c} > 1 - d^{-c+1}$ for some universal constant $c$, since $t < \log n < C_5 \log d$.
If $n$ is large, i.e.,
$n^{z_1} > C_2 d \log^3 d k \kappa^2$ for some universal constant  $z_1 \in (0,1)$,
then we use Corollary \ref{corollary:extension-to-contraction-lemma-with-large-n} to establish equation~\eqref{eqn:linear-contraction-in-proof-main-thm}, 
and we invoke equation~\eqref{eqn:lemma-concentration-standard-AUA-large-n} in Lemma  \ref{lemma:concentration-standard} for $t$ iterations.
This holds with probability at least $1 - t/\exp(n^{z}) > 1 - d^{-c}$ for some universal constant $c$, since $t < \log n$.

Therefore equation \eqref{eqn:tight-delta_bound} holds with probability at least $1 - d^{-c}$ for all iteration $t$, and we complete our proof with $\vecnorm{\Ucoeff_{t}\Ucoeff_{t}^\top - \trueUcoeff}{2}$. 

With the same argument, we also obtain $\vecnorm{\Ucoeff_t \Vcoeff_t^{\top}}{2} < C \sqrt{\frac{k \kappa^2 d \log d}{n} }  \sigma_r$ after $t = \ceil{\log \frac{n}{k \kappa^2 d \log d}}$ iterations. Therefore, we obtain the conclusion of claim (a) in Theorem~\ref{theorem:sample-convergence-T-small}.

\paragraph{Proof for claim (b) with the sub-linear convergence:} For the sublinear convergence part in claim (b), we prove it by induction. 
We consider the base case. Since $n > C_1 \kappa^2 d \log^3 d \cdot \max(1, \sigma^2/\sigma_r^2)$, we have $50 \kappa \sqrt{\frac{d \log d }{n}}\sigma \leq 0.05 \sigma_r$ by choosing $\sqrt{C_1} = 1000$. Therefore the base case is correct by the definitions of $A_0$ and $D_0$.

The key induction step is proven in the Lemma \ref{lemma:sample-contraction-T-small}, as the equation \eqref{eqn:sub-linear-contraction-in-lemma}.
However, as the convergence rate is sub-linear ultimately, it is sub-optimal to directly invoke concentration result (Lemma \ref{lemma:concentration-standard}) to establish equation \eqref{eqn:tight-delta_bound} at each iteration and take union bound over all the iterations.
Hence, we adapt the standard localization techniques from empirical process theory to sharpen the rates. Note that, these techniques had also been used to study the convergence rates of optimization algorithms in mixture models settings~\citep{Raaz_Ho_Koulik_2018, kwon2020minimax}.

The key idea of the localization technique is that, instead of invoking the concentration result at each iteration, we only do so when $D_t$ is decreased by $2$. More precisely, we divide all the iterations into epochs, where $i$-th epoch starts at iteration $\alpha_i$, ends at iteration $\alpha_{i+1} - 1$, and $D_{\alpha_{i+1}} \leq 0.5 D_{\alpha_i}$.  We invoke Lemma~\ref{lemma:concentration-uniform} at $\alpha_i$ to establish equation~\eqref{eqn:tight-delta_bound} for all the iterations in $i$-th epoch. Finally, we take a union bound over all the epochs.

By definition, we have
\begin{align*}
    \Delta_t = \frac{1}{n} \sum_i^n  \angles{\senseMat_i, \fitMat_t \fitMat_t^T - \mathbf{X}^*} \senseMat_i   - (\fitMat_t \fitMat_t^T - \mathbf{X}^*) + \frac{1}{n} \sum_i^n \epsilon_i \senseMat_i .
\end{align*}
From Lemma~\ref{lemma:concentration-Aepsilon}, we know that with  probability at least $1 - \exp(-C)$,
\begin{align*}
    \frac{1}{n} \sum_i^n \epsilon_i \senseMat_i \leq \sqrt{\frac{d \log d}{n} }\sigma.
\end{align*}
We only have to invoke this concentration result once for the entire algorithm analysis.

At iteration ${\alpha_i}$, note that $\vecnorm{\fitMat_{\alpha_i} \fitMat_{\alpha_i}^T - \mathbf{X}^*}{2} \leq \vecnorm{\Ucoeff_{\alpha_i} \Ucoeff_{\alpha_i}^T - \trueUcoeff}{2} + \vecnorm{\Vcoeff_{\alpha_i} \Vcoeff_{\alpha_i}^T - \trueVcoeff}{2} + 2\vecnorm{\Ucoeff_{\alpha_i} \Vcoeff_{\alpha_i}^T}{2} \leq 4 D_{\alpha_i} + \vecnorm{\trueVcoeff}{2} < 5 D_{\alpha_i} $. By Lemma \ref{lemma:concentration-uniform} we know that, with  probability at least $1 - \exp(-C)$, we have
\begin{align*}
    \sup_{\vecnorm{\bX}{2}\leq 5  D_{\alpha_i}}
    \frac{1}{n} \sum_i^n  \angles{\senseMat_i, \bX} \senseMat_i   - \bX \leq 5\sqrt{\frac{k d \log d}{n} } D_{\alpha_i}.
\end{align*}
Therefore, we find that 
\begin{align*}
\vecnorm{\Delta_{\alpha_i}}{2} \leq 5\sqrt{\frac{k d \log d}{n} }  D_{\alpha_i} + \sqrt{\frac{d \log d}{n}} \sigma,
\end{align*}
and equation~\eqref{eqn:tight-delta_bound} is satisfied at iteration $\alpha_i$. For notation simplicity, we define $A_t = D_t - 50 \kappa \sqrt{\frac{d \log d }{n}}\sigma$.
Invoking Lemma \ref{lemma:sample-contraction-T-small}, we have
\begin{align*}
    D_{{\alpha_i}+1}  = A_{{\alpha_i}+1} +  50 \kappa \sqrt{\frac{d \log d }{n}}\sigma \leq \parenth{1 - \frac{1}{2} \eta A_{\alpha_i}}A_{\alpha_i} +  50 \kappa \sqrt{\frac{d \log d }{n}}\sigma \leq D_{{\alpha_i}},
\end{align*}
where the last inequality just comes from $D_{{\alpha_i}} = A_{{\alpha_i}} +  50 \kappa \sqrt{\frac{d \log d }{n}}\sigma$.
At iteration $t \in \parenth{{\alpha_i}, {\alpha_{i+1}} - 1}$, by induction $D_t = A_t +  50 \kappa \sqrt{\frac{d \log d }{n}}\sigma$, and $D_t \leq D_{t-1} \leq D_{\alpha_i}$. Furthermore, we also have $2 D_t > D_{\alpha_i}$. Therefore, the following bounds hold:
\begingroup
\allowdisplaybreaks
\begin{align*}
    \Delta_t &= \frac{1}{n} \sum_i^n  \angles{\senseMat_i, \fitMat_t \fitMat_t^T - \mathbf{X}^*} \senseMat_i   - (\fitMat_t \fitMat_t^T - \mathbf{X}^*) + \frac{1}{n} \sum_i^n \epsilon_i \senseMat_i \\
    &\leq 5 \sqrt{\frac{d \log d}{n} } D_{\alpha_i} + \sqrt{\frac{d \log d}{n}} \sigma\\
    &\leq 10 \sqrt{\frac{d \log d}{n} } D_t + \sqrt{\frac{d \log d}{n}} \sigma\\
    &\leq \eta \sigma_r D_t + \sqrt{\frac{d \log d}{n}} \sigma.
\end{align*}
\endgroup
Hence, equation~\eqref{eqn:tight-delta_bound} is satisfied for all iteration $t \in \parenth{{\alpha_i}, {\alpha_{i+1}} - 1}$.
Invoking Lemma \ref{lemma:sample-contraction-T-small}, we have
\begin{align*}
    D_{t+1}  = A_{t+1} +  50 \kappa \sqrt{\frac{d \log d }{n}}\sigma \leq \parenth{1 - \frac{1}{2} \eta A_t}A_t +  50 \kappa \sqrt{\frac{d \log d }{n}}\sigma
\end{align*}
with probability at least $1 - d^{-c}$ for a universal constant $c$.
This directly implies that
\begin{equation}
\label{eqn:sub-linear-contraction-in-proof-main-thm}
    A_{t+1} \leq \parenth{1 - \frac{1}{2} \eta A_t} A_t.
\end{equation}
We first assume that equation~\eqref{eqn:sub-linear-contraction-in-proof-main-thm} holds for all iterations $t$, and then show that this is true with probability at least $1 - d^{-c}$ for some constant $c$.
With this, we claim that $A_t \leq \frac{4}{\eta t + \frac{4}{A_0}}$.
To see this, we have
\begin{align*}
     A_{t+1} \leq \parenth{1 - \frac{1}{2} \eta A_t} A_t 
     \overset{(1)}{\leq} &\parenth{1 -   \frac{2}{ t + \frac{4}{\eta A_0}}} \frac{4}{\eta t + \frac{4}{A_0}} \\  
    = &\frac{\parenth{t + \frac{4}{\eta A_0}}-2}{t + \frac{4}{\eta A_0}} \frac{4}{\eta \parenth{t + \frac{4}{\eta A_0}}}\\
     \overset{(2)}{\leq} & \frac{4}{\eta \parenth{t + 1+ \frac{4}{\eta A_0}}}
\end{align*}
where inequality $(1)$ holds because $\parenth{1 - \frac{1}{2} \eta A_t} A_t$ is quadratic with respect to $A_t$ and we plug-in the optimal $A_t$; inequality $(2)$ holds because $\frac{\parenth{t + \frac{4}{\eta A_0}}-2}{\parenth{t + \frac{4}{\eta A_0}}^2} \leq \frac{1}{\parenth{t + \frac{4}{\eta A_0}} + 1}$.

Therefore, after $t \geq \Theta\parenth{\frac{1}{\eta \epsilon_{stat}}}$ number of iterations, $A_t = D_t - 50 \kappa \sqrt{\frac{d \log d }{n}}\sigma \leq \Theta\parenth{\epsilon_{stat}}$,
which indicates that
\begin{align}
\label{eqn:final-stat-error-in-proof-main-thm}
    \max \braces{ \vecnorm{\Ucoeff_t \Ucoeff_t^T - \trueUcoeff}{2},  \vecnorm{\Vcoeff_t \Vcoeff_t^T}{2}, \vecnorm{\Ucoeff_t \Vcoeff_t^T}{2} }  \leq \Theta\parenth{\epsilon_{stat}}.
\end{align}

Now what is left to be shown is that equation~\eqref{eqn:sub-linear-contraction-in-proof-main-thm} holds for all iterations $t$ with probability at least $1 - d^{-c}$ for some constant $c$.
We first consider $t = \Theta\parenth{\frac{1}{\eta \epsilon_{stat}}}$.
If $n$ is not too large, i.e., $n < d^{c_5}$ for some constant $c_5$, then we invoke equation~\eqref{eqn:lemma-concentration-standard-AUA} in Lemma~\ref{lemma:concentration-standard} for $t$ iterations.
This holds with probability at least $1 - t d^{-c} > 1 - d^{-c+1}$ for some universal constant $c$, since $t < \log n < C_5 \log d$.
If $n$ is large, i.e.,
$n^{z_1} > C_2 d \log^3 d k \kappa^2$ for some universal constant  $z_1 \in (0,1)$,
then we use Corollary~\ref{corollary:extension-to-contraction-lemma-with-large-n} to establish equation~\eqref{eqn:sub-linear-contraction-in-proof-main-thm}, 
and we invoke equation~\eqref{eqn:lemma-concentration-standard-AUA-large-n} in Lemma~\ref{lemma:concentration-standard} for $t$ iterations.
This holds with probability at least $1 - t/\exp(n^{z}) > 1 - d^{-c}$ for some universal constant $c$, since $t < \log n$.
If $t > \Theta\parenth{\frac{1}{\eta \epsilon_{stat}}}$, we can show using above argument that 
after $\Theta\parenth{\frac{1}{\eta \epsilon_{stat}}}$ number of iterations equation~\eqref{eqn:final-stat-error-in-proof-main-thm} holds.
After this, by Lemma~\ref{lemma:non-expansion-after-epsilon-stat} we know that $D_t = \Theta\parenth{\epsilon_{stat}}$. Then, we can invoke Lemma~\ref{lemma:sample-contraction-T-small} or Corollary~\ref{corollary:extension-to-contraction-lemma-with-large-n} without further invoking the concentration argument anymore, since the radius in the uniform concentration result does not change .

As a consequence, after $t \geq \Theta\parenth{\frac{1}{\eta \epsilon_{stat}}}$ number of iterations, by triangular inequality, and the assumption that $\vecnorm{\trueVcoeff}{2}\leq \epsilon_{stat}$, we have $\vecnorm{\fitMat_t \fitMat_t^\top - \trueMat}{2} \leq \Theta \parenth{\epsilon_{stat}}$.
Combined with Lemma \ref{lemma:non-expansion-after-epsilon-stat}, we complete the proof of Theorem~\ref{theorem:sample-convergence-T-small}.

\section{Supporting Lemma}
In this appendix, we provide proofs for supporting lemmas in the main text.
\subsection{Proof of Lemma \ref{lemma:init}}
\label{appendix:proof-init}

\begin{proof}
From the definition of operator norm, we have
\begin{align*}
    \vecnorm{\trueVcoeff - \Vcoeff_0 \Vcoeff_0^\top}{2} = \max_{\bx \in \mathbb{R}^{d-r}: \vecnorm{\bx}{2}\leq 1} \abss{\bx^\top \parenth{\trueVcoeff - \Vcoeff_0 \Vcoeff_0^\top}\bx }.
\end{align*}
Since $\Vmat\in\mathbb{R}^{d * (d-r)}$ is an orthonormal matrix, for any $\bx \in \mathbb{R}^{d-r}$, we can find a vector $\bz \in \mathbb{R}^{d}$ such that $\Vmat^\top \bz = \bx$. Hence, we find that
\begin{align*}
    \vecnorm{\trueVcoeff - \Vcoeff_0 \Vcoeff_0^\top}{2} = \vecnorm{\Vmat \parenth{\trueVcoeff - \Vcoeff_0 \Vcoeff_0^\top}\Vmat^\top}{2} = \max_{\bx \in \mathbb{R}^{d}: \vecnorm{\bx}{2}\leq 1} \abss{\bx^\top \Vmat \parenth{\trueVcoeff - \Vcoeff_0 \Vcoeff_0^\top}\Vmat^\top\bx }.
\end{align*}
Without loss of generality we can write any $\bx \in \mathbb{R}^{d}$ as $\bx = \bx_u + \bx_v$, such that $\Umat \bz = \bx_u $ for some $\bz \in \mathbb{R}^{r}$, and $\Vmat \bz' = \bx_v$ for some $\bz' \in \mathbb{R}^{d - r}$ since $\Umat$ and $\Vmat$ are perpendicular to each other and they together span $\mathbb{R}^d$. 
If $\bx^* = \arg \max_{\bx \in \mathbb{R}^{d}: \vecnorm{\bx}{2}\leq 1} \abss{\bx^\top \Vmat \parenth{\trueVcoeff - \Vcoeff_0 \Vcoeff_0^\top}\Vmat^\top\bx  }$ then $\bx^*_u$ is zero. It is because if $\bx^*_u \neq 0$, one can decrease $\bx^*_u$ to zero and increase $\bx^*_v$ to $\bx^*_v/\vecnorm{\bx^*_v}{2}$, which does make the target quantity smaller. Therefore, we have
\begin{align*}
    & \hspace{- 1 em} \vecnorm{\trueVcoeff - \Vcoeff_0 \Vcoeff_0^\top}{2} \\
    =& \max_{\bx \in \mathbb{R}^{d}: \vecnorm{\bx}{2}\leq 1} \abss{\bx^\top \Vmat \parenth{\trueVcoeff - \Vcoeff_0 \Vcoeff_0^\top}\Vmat^\top\bx }\\
    =& \max_{\substack{\bx: \vecnorm{\bx}{2}\leq 1,\\ \Umat^\top \bx = 0}} \abss{\bx^\top \Vmat \parenth{\trueVcoeff - \Vcoeff_0 \Vcoeff_0^\top}\Vmat^\top\bx + \bx^\top \Umat \parenth{\trueUcoeff - \Ucoeff_0 \Ucoeff_0^\top} \Umat^\top\bx + 2 \bx^\top \parenth{\Umat \Ucoeff_0 \Vcoeff_0^\top\Vmat^\top}\bx}\\
    \leq& \max_{\bx: \vecnorm{\bx}{2}\leq 1} \abss{\bx^\top \Vmat \parenth{\trueVcoeff - \Vcoeff_0 \Vcoeff_0^\top}\Vmat^\top\bx + \bx^\top \Umat \parenth{\trueUcoeff - \Ucoeff_0 \Ucoeff_0^\top} \Umat^\top\bx + 2 \bx^\top \parenth{\Umat \Ucoeff_0 \Vcoeff_0^\top\Vmat^\top}\bx}\\
    =& \vecnorm{\fitMat_0 \fitMat_0^\top - \trueMat}{2} \leq 0.7 \rho \sigma_r,
\end{align*}
where the final inequality is due to the Assumption~\ref{assumption:init}.
The same techniques can be applied to obtain
\begin{align*}
    \vecnorm{\trueUcoeff - \Ucoeff_0 \Ucoeff_0^\top}{2} \leq \vecnorm{\fitMat_0 \fitMat_0^\top - \trueMat}{2} \leq 0.7 \rho \sigma_r.
\end{align*}  

For $\vecnorm{ \Ucoeff_0 \Vcoeff_0^\top}{2}$, we claim that the following equations hold:
\begin{align*}
    \vecnorm{ \Ucoeff_0 \Vcoeff_0^\top}{2} = \vecnorm{ \Umat \Ucoeff_0 \Vcoeff_0^\top\Vmat^\top}{2}= 0.5 \vecnorm{\Umat \Ucoeff_0 \Vcoeff_0^\top \Vmat^\top + \Vmat \Vcoeff_0 \Ucoeff_0^\top \Umat^\top}{2}.
\end{align*} 
To see the last equality, let $\sigma_1$ be the largest eigen-value (in magnitude) of $\Umat \Ucoeff_0 \Vcoeff_0^\top\Vmat^\top$ and let $\bar{\bx}$ be the corresponding eigen-vector. 
For some $c \in (0, 1)$, let $\bar{\bx} = c \bar{\bx}_u + \sqrt{1-c^2} \bar{\bx}_v$ such that $\Umat \bz = \bar{\bx}_u $ for some $\bz \in \mathbb{R}^{r}$,  $\Vmat \bz' = \bar{\bx}_v$ for some $\bz' \in \mathbb{R}^{d - r}$, and $\vecnorm{\bar{\bx}_u}{2} = 1$ and $\vecnorm{\bar{\bx}_v}{2}=1$. 
Then, direct algebra leads to
\begin{align*}
    \sigma_1 = \parenth{\bar{\bx}}^\top \Umat \Ucoeff_0 \Vcoeff_0^\top\Vmat^\top \bar{\bx} = c \sqrt{1-c^2}\parenth{\bar{\bx}}^\top_u \Umat \Ucoeff_0 \Vcoeff_0^\top\Vmat^\top \bar{\bx}_v. 
\end{align*}
For the RHS of the above equation, the optimal choice of $c$ is $1/\sqrt{2}$. We already know that the largest eigen-value (in magnitude) of $ \Vmat\Vcoeff_0  \Ucoeff_0^\top\Umat^\top$ is also $\sigma_1$. 
Therefore, we obtain that
\begin{align*}
    \parenth{\bar{\bx}}^\top  \Vmat\Vcoeff_0  \Ucoeff_0^\top\Umat^\top \bar{\bx} = c \sqrt{1-c^2}\parenth{\bar{\bx}}^\top_u \Umat \Ucoeff_0 \Vcoeff_0^\top\Vmat^\top \bar{\bx}_v = \sigma_1.
\end{align*}
Collecting the above results, we have $\vecnorm{ \Umat \Ucoeff_0 \Vcoeff_0^\top\Vmat^\top}{2}= 0.5 \vecnorm{\Umat \Ucoeff_0 \Vcoeff_0^\top \Vmat^\top + \Vmat \Vcoeff_0 \Ucoeff_0^\top \Umat^\top}{2}$. Then, an application of triangular inequality yields that
\begin{align*}
    2 \vecnorm{ \Ucoeff_0 \Vcoeff_0^\top}{2} =& \vecnorm{\Umat \Ucoeff_0 \Vcoeff_0^\top \Vmat^\top + \Vmat \Vcoeff_0 \Ucoeff_0^\top \Umat^\top}{2} \\
    \leq& \vecnorm{\fitMat_0 \fitMat_0^\top - \trueMat}{2} + \vecnorm{\trueVcoeff - \Vcoeff_0 \Vcoeff_0^\top + \trueUcoeff - \Ucoeff_0 \Ucoeff_0^\top}{2}.
\end{align*}
We can check that $\vecnorm{\trueVcoeff - \Vcoeff_0 \Vcoeff_0^\top + \trueUcoeff - \Ucoeff_0 \Ucoeff_0^\top}{2} \leq 0.7\cdot \sqrt{2} \rho \sigma_r$ by decomposing the eigen-vector $\bx = c \bx_u + \sqrt{1-c^2} \bx_v$ as above. Therefore, $\vecnorm{ \Ucoeff_0 \Vcoeff_0^\top}{2} < \rho \sigma_r$.

As a consequence, we obtain the conclusion of the lemma.
\end{proof}

\section{Concentration bounds}
\label{appendix:concentration-proof}
In this appendix, we want establish the uniform concentration bound for the following term: 
$$\frac{1}{n} \sum_{i = 1}^n \left( \angles{\senseMat_i, \fitMat \fitMat^{\top} - \mathbf{X}^*} + \epsilon_i \right)\senseMat_i   - (\fitMat \fitMat^{\top} - \mathbf{X}^*),$$ for any matrix $F \in \mathbb{R}^{d * k}$ such that $\| \fitMat \fitMat^{\top} - \mathbf{X}^*\|_{2} \leq R$ for some radius $R > 0$. To do so, we have to bound the spectral norm of each random observation, and then take Bernstein/Chernoff type bound.
\begin{lemma}
\label{lemma:matrix-bernstein}
\textbf{(Matrix Bernstein, Theorem 1.4 in \cite{troppuserfriendly})}
Consider a finite sequence $\{\bX_k\}$ of independent, random, self-adjoint matrices with dimension $d$.
Assume that each random matrix satisfies
\begin{align*}
    \Exs[\bX_k] = \mathbf{0}, \quad \quad \rm{and} \quad \quad \lambda_{\rm{max}}(\bX_k) \leq R \quad \text{almost surely}.
\end{align*}
Then, for all $t \geq 0$,
\begin{align}
\label{eqn:mat-bernstein}
\Prob \parenth{\lambda_{\rm{max}}\parenth{\sum_k \bX_k} \geq t} \leq d \cdot \exp \parenth{\frac{-t^2/2}{\sigma^2 + Rt/3}} \quad where \quad \sigma^2 \vcentcolon= \vecnorm{\sum_k \Exs(\bX_k^2)}{2}.
\end{align}
\end{lemma}
\begin{lemma}
Let $\bA$ be a symmetric random matrix in $\mathbb{R}^{d*d}$, with the upper triangle entries ($i\geq j$) being independently sampled from an identical  sub-Gaussian distribution whose mean is $0$ and variance proxy is $1$. Let $\epsilon$ follows $N(0, \sigma)$. Then
\begin{align*}
    \Prob \parenth{\vecnorm{\epsilon \bA}{2} > C_1 \sigma \sqrt{d}}  \leq \exp \parenth{- C_2}.
\end{align*}
\end{lemma}

\begin{proof}

As $\epsilon$ is sub-Gaussian, we know that for all $t > 0$
\begin{align*}
    \Prob \parenth{|\epsilon| > t \sigma} \leq 2 \exp \parenth{- \frac{t}{2}}
\end{align*}
By standard $\epsilon$-net argument~\citep{troppuserfriendly, Vershynin_2018}, for some universal constant $C_1, C_2$, we have
\begin{align*}
    \Prob \parenth{\vecnorm{\bA}{2} > C_1 \sqrt{d}} \leq \exp \parenth{- \frac{d}{C_2}}.
\end{align*}
Applying the union bound to the above concentration results leads to
\begin{align*}
    \Prob \parenth{|\epsilon| > C_1 \sigma \text{ or } \vecnorm{\bA}{2} > C_2 \sqrt{d}} \leq 2 \exp \parenth{- \frac{C_1}{2}} + \exp \parenth{- \frac{d}{C_3}} \leq \exp \parenth{- C_4}.
\end{align*}
Note that,
$
\vecnorm{\bA \epsilon}{2} = |\epsilon| \vecnorm{\bA}{2}
$. Therefore, we have
\begin{align*}
    \Prob \parenth{\vecnorm{\epsilon \bA}{2} > C_1 \sigma \sqrt{d}}  \leq \exp \parenth{- C_2}.
\end{align*}
\end{proof}

\begin{lemma}
\label{lemma:concentration-Aepsilon}
\textbf{(Lemma \ref{lemma:concentration-Aepsilon-main-text} re-stated)}
Let $\bA_i$ be symmetric random matrices in $\mathbb{R}^{d*d}$, with the upper triangle entries ($i\geq j$) being independently sampled from an identical  sub-Gaussian distribution whose mean is $0$ and variance proxy is $1$. Let $\epsilon_i$ follows $N(0, \sigma)$. Then
\begin{align*}
\Prob \parenth{\vecnorm{\frac{1}{n}\sum_i^n \bA_i \epsilon_i}{2} \geq C \sqrt{\frac{d \sigma^2}{n}}} \leq \exp(-C).
\end{align*}
\end{lemma}
\begin{proof}
We prove the lemma by applying the matrix Bernstein bound. In fact, direct calculation shows that
\begin{align*}
    \Exs \parenth{\parenth{\bA_i \epsilon_i}^2 }= \sigma^2 \Exs \parenth{\bA_i^2} =  \sigma^2 d \mathbf{I}. 
\end{align*}
Hence, we obtain
\begin{align*}
    \vecnorm{\sum_i^n \Exs \parenth{\parenth{\bA_i \epsilon_i}^2 }}{2} \leq n \sigma^2 d. 
\end{align*}
From the matrix Bernstein bound~\citep{Wainwright_nonasymptotic}, we find that
\begin{align*}
\Prob \parenth{\vecnorm{\frac{1}{n}\sum_i^n \bA_i \epsilon_i}{2} \geq t} \leq &d \cdot \exp \parenth{ \frac{-3 t^2 n^2}{6 d n \sigma^2 + 2 C_1 \sigma \sqrt{d} t n}}
=  d \cdot \exp \parenth{ \frac{-3 t^2 n}{6 d \sigma^2 + 2 C_1 \sigma \sqrt{d} t}}.
\end{align*}
For any $\delta < 1/e$, let $t = \log \frac{1}{\delta} \sqrt{\frac{d  \sigma^2}{n}}$. Then, the above bound becomes
\begin{align*}
\Prob \parenth{\vecnorm{\frac{1}{n}\sum_i^n \bA_i \epsilon_i}{2} \geq \log \frac{1}{\delta} \sqrt{\frac{d \sigma^2}{n}}} \leq \delta.
\end{align*}
Or equivalently, for any $C > 1$, let $t = C \sqrt{\frac{d  \sigma^2}{n}}$,
\begin{align*}
\Prob \parenth{\vecnorm{\frac{1}{n}\sum_i^n \bA_i \epsilon_i}{2} \geq C \sqrt{\frac{d \sigma^2}{n}}} \leq \exp(-C).
\end{align*}
As a consequence, we reach the conclusion of the lemma.
\end{proof}


\begin{lemma}
\label{lemma:concentration-one-term}
Let $\bA$ be a symmetric random matrix in $\mathbb{R}^{d*d}$, with the upper triangle entries ($i\geq j$) being independently sampled from an identical  sub-Gaussian distribution whose mean is $0$ and variance proxy is $1$. Let $\bU$ be a deterministic symmetric matrix of the same dimension. Then, for some universal constant $C_1, C_2$, we have
\begin{align*}
    \Prob \parenth{\vecnorm{\langle \bA,  \bU \rangle \bA -  \bU}{2} \geq C_1 d  \vecnorm{\bU}{F}} \leq \exp\parenth{-d / C_2}.
\end{align*}
\end{lemma}
\begin{proof}
We show this by standard $\epsilon$-net argument. In particular, we have
\begin{align*}
    \vecnorm{\langle \bA,  \bU \rangle \bA -  \bU}{2}
    = & \max_{\bx \in \mathcal{S}^{d-1}} \bx^{\top} \parenth{\langle \bA,  \bU \rangle \bA -  \bU}\bx \\
    = & \max_{\bx \in \mathcal{S}^{d-1}}  \langle \bA,  \bU \rangle \langle \bA,  \bx \bx^{\top} \rangle - \parenth{\bx^{\top} \bU \bx}.
\end{align*}
Note that $\langle \bA,  \bU \rangle = \sum_{i,j} A_{ij} U_{ij}$ is sub-Gaussian with variance proxy $\vecnorm{\bU}{F}^2$, and $\langle \bA,  \bx \bx^{\top} \rangle = \sum_{i,j} A_{ij} x_i x_j$ is sub-Gaussian with variance proxy $1$. Therefore 
$\Prob \parenth{|\langle \bA,  \bU \rangle| > t \vecnorm{\bU}{F}} \leq \exp\parenth{-t^2}$
and 
$\Prob \parenth{|\langle \bA, \bx \bx^{\top} \rangle| > t } \leq \exp\parenth{-t^2}$.
By the union bound, 
$$\Prob \parenth{|\langle \bA,  \bU \rangle \langle \bA,  \bx \bx^{\top} \rangle| > t \vecnorm{\bU}{F}} \leq 2\exp\parenth{-t}. $$
Since $\parenth{\bx^{\top} \bU \bx} \leq \vecnorm{\bU}{2} \leq \vecnorm{\bU}{F}$, we have
\begin{align}
\label{eqn:concentration-one-x}
    \Prob \parenth{\bx^{\top} \parenth{\langle \bA,  \bU \rangle \bA -  \bU}\bx \geq t \vecnorm{\bU}{F}} \leq \exp\parenth{-\frac{t}{C_1}}.
\end{align}
By the standard $\epsilon$-net argument, let $\Vcal$ be the $\epsilon$ covering of 
$\mathcal{S}^{d-1}$. Then, we find that
$$\vecnorm{\langle \bA,  \bU \rangle \bA -  \bU}{2} \leq \frac{1}{1-2\epsilon}\max_{\bx \in \Vcal} \bx^T \parenth{\langle \bA,  \bU \rangle \bA -  \bU}\bx.$$
Now we fix $\epsilon$ to be $1/4$. Then, for equation~\eqref{eqn:concentration-one-x} we take union bound over $\Vcal$ and we have
\begin{align*}
    \Prob \parenth{\max_{\bx \in \Vcal} \bx^{\top} \parenth{\langle \bA,  \bU \rangle \bA -  \bU}\bx \geq t\vecnorm{\bU}{F}} \leq  |\Vcal| \exp\parenth{-\frac{t}{C_1}}, \quad \rm{for} \quad t > C_2
\end{align*}
and $|\Vcal| = e^{d \log 9}$. By choosing $t = C_1 d $ for reasonably large universal constant $C_1$ we have
\begin{align*}
    \Prob \parenth{\vecnorm{\langle \bA,  \bU \rangle \bA -  \bU}{2} \geq C_1 d  \vecnorm{\bU}{F}} \leq \exp\parenth{-d / C_2}.
\end{align*}
As a consequence, we obtain the conclusion of the lemma.
\end{proof}

\begin{lemma}
\label{lemma:concentration-standard}
Let $\bA_i$ be a symmetric random matrix of dimension $d$ by $d$, with the upper triangle entries ($i\geq j$) being independently sampled from an identical  sub-Gaussian distribution whose mean is $0$ and variance proxy is $1$.
Let $\bU$ be a deterministic symmetric matrix of the same dimension. Then  as long as $n > C_1  d \log^3 d $ for some universal $C_1, C_2 > 10$, we have
\begin{align}
\label{eqn:lemma-concentration-standard-AUA}
\Prob \parenth{\vecnorm{\frac{1}{n}\sum_i^n \parenth{\langle \bA_i, \bU \rangle \bA_i - \bU}}{2} \leq \sqrt{\frac{d \log d}{n}} \vecnorm{\bU}{F}} 
\geq  1 - \exp \parenth{ - C_2 \log d} .
\end{align}
Moreover when $n$ is larger than the order of $d$, that is, 
if there exists a constant $z_1 \in (0,1)$ such that 
$
    n^{z_1} > C_2 d \log^3 d k \kappa^2,
$
for some universal constant $z_2 \in (0,1)$ we have
\begin{align}
\label{eqn:lemma-concentration-standard-AUA-large-n}
\Prob \parenth{\vecnorm{\frac{1}{n}\sum_i^n \parenth{\langle \bA_i, \bU \rangle \bA_i - \bU}}{2} \leq \frac{1}{\kappa\log d\sqrt{ k C_2 }}\vecnorm{\bU}{F}} 
\geq  1 - \exp \parenth{ - n^{z_2}} .
\end{align}
\end{lemma}
\begin{proof}
Following Lemma \ref{lemma:matrix-bernstein}, we want to first bound the second order moment of the random matrices. Since $\Exs \langle \bA_i, \bU \rangle \bA_i = \bU$ and $\bU$ has no randomness, we have
\begin{align*}
    \Exs \parenth{\langle \bA_i, \bU \rangle \bA_i - \bU}^2 = \Exs \parenth{\langle \bA_i, \bU \rangle \bA_i}^2 - \bU^2.
\end{align*}
The $(m,n)$  entry  of  $\Exs \parenth{\langle \bA, \bU \rangle \bA}^2$ equals to
\begin{align*}
    \sum_{a,b,c,d, j=1}^d \Exs \parenth{A_{ab} A_{cd} U_{ab} U_{cd} A_{mj} A_{jn}}.
\end{align*}
For diagonal entries, i.e., $m=n$, the expectation is not zero if and only if $A_{ab} = A_{cd}$. Hence for diagonal entry $(m,m)$, its expectation is 
\begin{align*}
    \sum_{a,b}^d \Exs \parenth{ A_{ab}^2 A_{mm}^2 }  U_{ab}^2 = \sum_{a,b}^d  U_{ab}^2 + 2 U_{mm}^2 = \vecnorm{\bU}{F}^2 + 2 U_{mm}^2. 
\end{align*}
For off diagonal entries, i.e., $m\neq n$, the expectation is not zero for that entry when (1) $A_{ab} = A_{mj}$ and $A_{cd} = A_{jn}$, or when (2) $A_{ab} = A_{jn}$ and $A_{cd} = A_{mj}$. For both cases, the expectation equals the $(m,n)$ entry of $\bU^2$. Therefore, we obtain that
\begin{align*}
    \sum_{j=1}^d \Exs \parenth{ A_{mj}^2 A_{jn}^2 U_{mj} U_{jn}} = \sum_{j=1}^d   U_{mj} U_{jn}.
\end{align*}
Hence the $(m,n)$ entry of $\Exs \parenth{\langle \bA_i, \bU \rangle \bA_i - \bU}^2$ equals $0$ when $m\neq n$, and equals $\vecnorm{\bU}{F}^2 + 2 U_{mm}^2 - \sum_j U_{mj}^2$ when $m=n$. Hence $\vecnorm{\Exs \parenth{\langle \bA_i, \bU \rangle \bA_i - \bU}^2}{2} \leq 3 \vecnorm{\bU}{F}^2 $ and
\begin{align*}
    \vecnorm{\sum_i^n \Exs \parenth{\langle \bA_i, \bU \rangle \bA_i - \bU}^2}{2} \leq  3n\vecnorm{\bU}{F}^2. 
\end{align*}
Then, the following inequality holds:
\begin{align*}
\Prob \parenth{\lambda_{\rm{max}}\parenth{\sum_i^n \parenth{\langle \bA_i, \bU \rangle \bA_i - \bU}} \geq t} &\leq d \cdot \exp \parenth{\frac{-t^2/2}{3n \vecnorm{\bU}{F}^2 + \frac{C_1 d \log d \vecnorm{\bU}{F} t}{3}}}
\end{align*}
where $C_1$ is a universal constant inherited from Lemma \ref{lemma:concentration-one-term} and
\begin{align*}
\Prob \parenth{\lambda_{\rm{max}} \parenth{\frac{1}{n}\sum_i^n \parenth{\langle \bA_i, \bU \rangle \bA_i - \bU}} \geq t} &\leq d \cdot \exp \parenth{\frac{-3 t^2 n}{ 18 \vecnorm{\bU}{F}^2 + 2 C_1 d \log d \vecnorm{\bU}{F} t}}. 
\end{align*}

Let $t = \sqrt{\frac{d \log d}{n}} \vecnorm{\bU}{F}$, and as long as $n > C_2  d \log^3 d $ for some universal constant $C_4 > 1000$, we have
\begin{align*}
& \hspace{- 5 em} \Prob \parenth{\lambda_{\rm{max}} \parenth{\frac{1}{n}\sum_i^n \parenth{\langle \bA_i, \bU \rangle \bA_i - \bU}} \geq \sqrt{\frac{d \log d}{n}} \vecnorm{\bU}{F}} \\
\leq& d \cdot \exp \parenth{\frac{-3 \parenth{\sqrt{\frac{d \log d}{n}} \vecnorm{\bU}{F}}^2 n}{ 18 \vecnorm{\bU}{F}^2 + 2 C_1 d \log d \vecnorm{\bU}{F} \parenth{\sqrt{\frac{d \log d}{n}} \vecnorm{\bU}{F}}}}  \\
\leq &  d \cdot \exp \parenth{\frac{- d \log d}{ C_3 d }} \quad \quad \text{(for $\sqrt{\frac{d \log d}{n}} \cdot \log d < 1$)} \\
\leq & \exp \parenth{ - C_4 \log d}  \quad \quad \text{(for some universal constant $C_4$)}.
\end{align*}  
Hence we finish the proof for equation~\ref{eqn:lemma-concentration-standard-AUA}.


For the tightness of our statistical analysis, we need to consider the case when $n$ is larger than the order of polynomial of $d$.
If there exists a constant $z \in (0,1)$ such that 
\begin{align*}
    n^z > C_2 d \log^3 d k \kappa^2,
\end{align*}
then plugging in $t = \sqrt{\frac{d\log d}{C_2 d \log^3 d k \kappa^2}} \vecnorm{\bU}{F} =  \frac{1}{\kappa\log d\sqrt{ k C_2 }}\vecnorm{\bU}{F}$, we have
\begin{align*}
& \hspace{- 4 em} \Prob \parenth{\lambda_{\rm{max}} \parenth{\frac{1}{n}\sum_i^n \parenth{\langle \bA_i, \bU \rangle \bA_i - \bU}} \geq \frac{1}{\kappa\log d\sqrt{ k C_2 }}\vecnorm{\bU}{F}} \\
\leq& d \cdot \exp \parenth{\frac{-3 \parenth{\frac{1}{\kappa\log d\sqrt{ k C_2 }}\vecnorm{\bU}{F}}^2 n}{ 18 \vecnorm{\bU}{F}^2 + 2 C_1 d \log d \vecnorm{\bU}{F} \parenth{\frac{1}{\kappa\log d\sqrt{ k C_2 }}\vecnorm{\bU}{F}}}}  \\
= &  d \cdot \exp \parenth{\frac{-3 n}{ 18  \parenth{\kappa\log d\sqrt{ k C_2 }}^2 + 2 C_1 d \log d \parenth{\kappa\log d\sqrt{ k C_2 }}}} \\
\leq &  \exp \parenth{\frac{- n}{C_3 n^{z_1}}} \\
\leq & \exp \parenth{ - n^{z_2}}  \quad \quad \text{(for some universal constant $z_2 \in (0,1)$)}.
\end{align*}  
In summary, we reach the conclusion of the lemma.
\end{proof}

\begin{lemma}
\label{lemma:concentration-uniform}
\textbf{(Lemma \ref{lemma:concentration-uniform-main-text} re-stated)}
Let $\bA_i$ be a symmetric random matrix of dimension $d$ by $d$. Its upper triangle entries ($i\geq j$) are independently sampled from an identical  sub-Gaussian distribution whose mean is $0$ and variance proxy is $1$. If $\bU$ is of rank $k$ and is in a bounded spectral norm ball of radius $R$ (i.e. $\|\bU\|_2 \leq R$), then we have
\begin{align}
\label{eqn:lemma-concentration-uniform-AUA}
\Prob \parenth{\sup_{\bU: \vecnorm{\bU}{2} \leq R}\vecnorm{\frac{1}{n}\sum_i^n \parenth{\langle \bA_i, \bU \rangle \bA_i - \bU}}{2} \leq \sqrt{\frac{d \log d}{n}} \sqrt{k}R} 
\geq  1 - \exp \parenth{ - C_2 \log d} .
\end{align}
Moreover when $n$ is larger than the order of $d$, that is, 
if there exists a constant $z_1 \in (0,1)$ such that 
$
    n^{z_1} > C_2 d \log^3 d k \kappa^2,
$
for some universal constant $z_2 \ in (0,1)$ we have
\begin{align}
\label{eqn:lemma-concentration-uniform-AUA-large-n}
\Prob \parenth{\sup_{\bU: \vecnorm{\bU}{2} \leq R}\vecnorm{\frac{1}{n}\sum_i^n \parenth{\langle \bA_i, \bU \rangle \bA_i - \bU}}{2} \leq \frac{1}{\kappa\log d\sqrt{ k C_2 }}R} 
\geq  1 - \exp \parenth{ - n^{z_2}} .
\end{align}
\end{lemma}
\begin{proof}
To show this uniform convergence result, we use the standard discretization techniques (i.e. $\epsilon$-net). In particular, we have
\begin{align*}
    & \hspace{- 4 em} \sup_{\bU: \vecnorm{\bU}{2} \leq R}\vecnorm{\frac{1}{n}\sum_i^n \parenth{\langle \bA_i, \bU \rangle \bA_i - \bU}}{2}\\
    = &\sup_{\bU: \vecnorm{\bU}{2} \leq R}\sup_{\bx: \vecnorm{\bx}{2}\leq 1}\frac{1}{n} \abss{\sum_i^n \parenth{\inprod{\bA_i}{\bU} \inprod{\bA_i}{\bx \bx^\top} - \inprod{\bU}{\bx \bx^\top} }}.
\end{align*}
Since the above quantity is symmetric, we can take off the absolute value and only look at the one-side deviation.
The crux is how to construct the $\epsilon$-net for $\bU$.
We decompose $\bU$. 
Since $\bU$ is of rank $k$ and $\vecnorm{\bU}{2}\leq R$, we can write $\bU = \sum_i^k \bu_i \bu_i^\top$ where $\bu_i$  are vectors, with $\vecnorm{\bu_i}{2}\leq \sqrt{R}$ and $\bu_i^\top \bu_j = 0$ for $i \neq j$. Therefore
\begin{align*}
    & \hspace{- 4 em} \frac{1}{n} \sum_i^n \parenth{\inprod{\bA_i}{\bU} \inprod{\bA_i}{\bx \bx^\top} - \inprod{\bU}{\bx \bx^\top} }\\
    =&\frac{1}{n} \sum_i^n \parenth{\inprod{\bA_i}{\sum_i^k \bu_i \bu_i^\top} \inprod{\bA_i}{\bx \bx^\top} - \inprod{\sum_i^k\bu_i \bu_i^\top}{\bx \bx^\top} }.
\end{align*}
Now we can construct a standard $\epsilon$-net for each $\bu_i \in \mathbb{R}^d$, and in total we construct $k$ such epsilon net for $\bU$.
Hence we invoke equation~\eqref{eqn:lemma-concentration-standard-AUA} in Lemma \ref{lemma:concentration-standard} for $1/4$ $\epsilon$-net on these $k$ norm balls:  $\vecnorm{\bu_i}{2}\leq R$ and take an union bound, we have
\begin{align*}
\Prob \parenth{\sup_{\bU: \vecnorm{\bU}{2} \leq R}\vecnorm{\frac{1}{n}\sum_i^n \parenth{\langle \bA_i, \bU \rangle \bA_i - \bU}}{2} \leq \sqrt{\frac{d \log d}{n}} \sqrt{k}R} 
\geq  1 - \exp \parenth{ - C_2 \log d} .
\end{align*}
Similarly if we invoke the equation~\eqref{eqn:lemma-concentration-standard-AUA-large-n} for $1/4$ $\epsilon$-net on these $k$ norm balls:  $\vecnorm{\bu_i}{2}\leq R$ and take an union bound, we also have
\begin{align*}
\Prob \parenth{\sup_{\bU: \vecnorm{\bU}{2} \leq R}\vecnorm{\frac{1}{n}\sum_i^n \parenth{\langle \bA_i, \bU \rangle \bA_i - \bU}}{2} \leq \frac{1}{\kappa\log d\sqrt{ k C_2 }}R} 
\geq  1 - \exp \parenth{ - n^{z_2}} .
\end{align*}
As a consequence, the conclusion of the lemma follows.
\end{proof}

\bibliography{Nhat}

\begin{thebibliography}{34}
\providecommand{\natexlab}[1]{#1}
\providecommand{\url}[1]{\texttt{#1}}
\expandafter\ifx\csname urlstyle\endcsname\relax
  \providecommand{\doi}[1]{doi: #1}\else
  \providecommand{\doi}{doi: \begingroup \urlstyle{rm}\Url}\fi

\bibitem[Balakrishnan et~al.(2017)Balakrishnan, Wainwright, and Yu]{Siva_2017}
S.~Balakrishnan, M.~J. Wainwright, and B.~Yu.
\newblock Statistical guarantees for the {EM} algorithm: From population to
  sample-based analysis.
\newblock \emph{Annals of Statistics}, 45:\penalty0 77--120, 2017.

\bibitem[Bhojanapalli et~al.(2016{\natexlab{a}})Bhojanapalli, Kyrillidis, and
  Sanghavi]{bhojanapalli2016dropping}
S.~Bhojanapalli, A.~Kyrillidis, and S.~Sanghavi.
\newblock Dropping convexity for faster semi-definite optimization.
\newblock In \emph{Conference on Learning Theory}, pages 530--582,
  2016{\natexlab{a}}.

\bibitem[Bhojanapalli et~al.(2016{\natexlab{b}})Bhojanapalli, Neyshabur, and
  Srebro]{bhojanapalli2016global}
S.~Bhojanapalli, B.~Neyshabur, and N.~Srebro.
\newblock Global optimality of local search for low rank matrix recovery.
\newblock \emph{arXiv preprint arXiv:1605.07221}, 2016{\natexlab{b}}.

\bibitem[Boyd et~al.(2004)Boyd, Boyd, and Vandenberghe]{boyd2004convex}
S.~Boyd, S.~P. Boyd, and L.~Vandenberghe.
\newblock \emph{Convex optimization}.
\newblock Cambridge university press, 2004.

\bibitem[Burer and Monteiro(2003)]{burer2003nonlinear}
S.~Burer and R.~D. Monteiro.
\newblock A nonlinear programming algorithm for solving semidefinite programs
  via low-rank factorization.
\newblock \emph{Mathematical Programming}, 95\penalty0 (2):\penalty0 329--357,
  2003.

\bibitem[Burer and Monteiro(2005)]{burer2005local}
S.~Burer and R.~D. Monteiro.
\newblock Local minima and convergence in low-rank semidefinite programming.
\newblock \emph{Mathematical Programming}, 103\penalty0 (3):\penalty0 427--444,
  2005.

\bibitem[Candes and Plan(2011)]{candes2011tight}
E.~J. Candes and Y.~Plan.
\newblock Tight oracle inequalities for low-rank matrix recovery from a minimal
  number of noisy random measurements.
\newblock \emph{IEEE Transactions on Information Theory}, 57\penalty0
  (4):\penalty0 2342--2359, 2011.

\bibitem[Cand{\`e}s et~al.(2011)Cand{\`e}s, Li, Ma, and
  Wright]{candes2011robust}
E.~J. Cand{\`e}s, X.~Li, Y.~Ma, and J.~Wright.
\newblock Robust principal component analysis?
\newblock \emph{Journal of the ACM (JACM)}, 58\penalty0 (3):\penalty0 1--37,
  2011.

\bibitem[Chen and Wainwright(2015)]{chen2015fast}
Y.~Chen and M.~J. Wainwright.
\newblock Fast low-rank estimation by projected gradient descent: General
  statistical and algorithmic guarantees.
\newblock \emph{arXiv preprint arXiv:1509.03025}, 2015.

\bibitem[Chen et~al.(2013)Chen, Jalali, Sanghavi, and Caramanis]{chen2013low}
Y.~Chen, A.~Jalali, S.~Sanghavi, and C.~Caramanis.
\newblock Low-rank matrix recovery from errors and erasures.
\newblock \emph{IEEE Transactions on Information Theory}, 59\penalty0
  (7):\penalty0 4324--4337, 2013.

\bibitem[Chi et~al.(2019)Chi, Lu, and Chen]{chi2019nonconvex}
Y.~Chi, Y.~M. Lu, and Y.~Chen.
\newblock Nonconvex optimization meets low-rank matrix factorization: An
  overview.
\newblock \emph{IEEE Transactions on Signal Processing}, 67\penalty0
  (20):\penalty0 5239--5269, 2019.

\bibitem[Dwivedi et~al.(2020{\natexlab{a}})Dwivedi, Ho, Khamaru, Wainwright,
  Jordan, and Yu]{Raaz_Ho_Koulik_2018}
R.~Dwivedi, N.~Ho, K.~Khamaru, M.~J. Wainwright, M.~I. Jordan, and B.~Yu.
\newblock Singularity, misspecification, and the convergence rate of {EM}.
\newblock \emph{Annals of Statistics}, 48:\penalty0 3161--3182,
  2020{\natexlab{a}}.

\bibitem[Dwivedi et~al.(2020{\natexlab{b}})Dwivedi, Ho, Khamaru, Wainwright,
  Jordan, and Yu]{dwivedi2019challenges}
R.~Dwivedi, N.~Ho, K.~Khamaru, M.~J. Wainwright, M.~I. Jordan, and B.~Yu.
\newblock Sharp analysis of {E}xpectation-{M}aximization for weakly
  identifiable models.
\newblock In \emph{AISTATS}, 2020{\natexlab{b}}.

\bibitem[Ge et~al.(2016)Ge, Lee, and Ma]{ge2016matrix}
R.~Ge, J.~D. Lee, and T.~Ma.
\newblock Matrix completion has no spurious local minimum.
\newblock In \emph{Advances in Neural Information Processing Systems}, pages
  2973--2981, 2016.

\bibitem[Gross et~al.(2010)Gross, Liu, Flammia, Becker, and
  Eisert]{gross2010quantum}
D.~Gross, Y.-K. Liu, S.~T. Flammia, S.~Becker, and J.~Eisert.
\newblock Quantum state tomography via compressed sensing.
\newblock \emph{Physical review letters}, 105\penalty0 (15):\penalty0 150401,
  2010.

\bibitem[Hardt(2014)]{hardt2014understanding}
M.~Hardt.
\newblock Understanding alternating minimization for matrix completion.
\newblock In \emph{2014 IEEE 55th Annual Symposium on Foundations of Computer
  Science}, pages 651--660. IEEE, 2014.

\bibitem[Jain et~al.(2013)Jain, Netrapalli, and Sanghavi]{jain2013low}
P.~Jain, P.~Netrapalli, and S.~Sanghavi.
\newblock Low-rank matrix completion using alternating minimization.
\newblock In \emph{Proceedings of the forty-fifth annual ACM symposium on
  Theory of computing}, pages 665--674, 2013.

\bibitem[Kalev et~al.(2015)Kalev, Kosut, and Deutsch]{kalev2015quantum}
A.~Kalev, R.~L. Kosut, and I.~H. Deutsch.
\newblock Quantum tomography protocols with positivity are compressed sensing
  protocols.
\newblock \emph{npj Quantum Information}, 1\penalty0 (1):\penalty0 1--6, 2015.

\bibitem[Koltchinskii et~al.(2011)Koltchinskii, Lounici, Tsybakov,
  et~al.]{koltchinskii2011nuclear}
V.~Koltchinskii, K.~Lounici, A.~B. Tsybakov, et~al.
\newblock Nuclear-norm penalization and optimal rates for noisy low-rank matrix
  completion.
\newblock \emph{The Annals of Statistics}, 39\penalty0 (5):\penalty0
  2302--2329, 2011.

\bibitem[Kwon et~al.(2020)Kwon, Ho, and Caramanis]{kwon2020minimax}
J.~Kwon, N.~Ho, and C.~Caramanis.
\newblock On the minimax optimality of the {EM} algorithm for learning
  two-component mixed linear regression.
\newblock \emph{arXiv preprint arXiv:2006.02601}, 2020.

\bibitem[Li et~al.(2018)Li, Ma, and Zhang]{li2018algorithmic}
Y.~Li, T.~Ma, and H.~Zhang.
\newblock Algorithmic regularization in over-parameterized matrix sensing and
  neural networks with quadratic activations.
\newblock In \emph{Conference On Learning Theory}, pages 2--47. PMLR, 2018.

\bibitem[Negahban and Wainwright(2011)]{negahban2011estimation}
S.~Negahban and M.~J. Wainwright.
\newblock Estimation of (near) low-rank matrices with noise and
  high-dimensional scaling.
\newblock \emph{The Annals of Statistics}, pages 1069--1097, 2011.

\bibitem[Negahban and Wainwright(2012)]{negahban2012restricted}
S.~Negahban and M.~J. Wainwright.
\newblock Restricted strong convexity and weighted matrix completion: Optimal
  bounds with noise.
\newblock \emph{The Journal of Machine Learning Research}, 13\penalty0
  (1):\penalty0 1665--1697, 2012.

\bibitem[Recht et~al.(2010)Recht, Fazel, and Parrilo]{recht2010guaranteed}
B.~Recht, M.~Fazel, and P.~A. Parrilo.
\newblock Guaranteed minimum-rank solutions of linear matrix equations via
  nuclear norm minimization.
\newblock \emph{SIAM review}, 52\penalty0 (3):\penalty0 471--501, 2010.

\bibitem[Tropp(2012)]{troppuserfriendly}
J.~A. Tropp.
\newblock User-friendly tail bounds for sums of random matrices.
\newblock \emph{Foundations of computational mathematics}, 12\penalty0
  (4):\penalty0 389--434, 2012.

\bibitem[Tu et~al.(2016)Tu, Boczar, Simchowitz, Soltanolkotabi, and
  Recht]{tu2016low}
S.~Tu, R.~Boczar, M.~Simchowitz, M.~Soltanolkotabi, and B.~Recht.
\newblock Low-rank solutions of linear matrix equations via procrustes flow.
\newblock In \emph{International Conference on Machine Learning}, pages
  964--973. PMLR, 2016.

\bibitem[van~der Vaart and Wellner(2000)]{Vaart_Wellner_2000}
A.~W. van~der Vaart and J.~A. Wellner.
\newblock \emph{Weak {C}onvergence and {E}mpirical {P}rocesses: {W}ith
  {A}pplications to {S}tatistics}.
\newblock Springer-Verlag, New York, NY, 2000.

\bibitem[Vershynin(2018)]{Vershynin_2018}
R.~Vershynin.
\newblock \emph{High Dimensional Probability. An Introduction with Applications
  in Data Science}.
\newblock Cambridge University Press, 2018.

\bibitem[Wainwright(2019)]{Wainwright_nonasymptotic}
M.~J. Wainwright.
\newblock \emph{High-Dimensional Statistics: A Non-Asymptotic Viewpoint}.
\newblock Cambridge University Press, 2019.

\bibitem[Waters et~al.(2011)Waters, Sankaranarayanan, and
  Baraniuk]{waters2011sparcs}
A.~E. Waters, A.~C. Sankaranarayanan, and R.~Baraniuk.
\newblock Sparcs: Recovering low-rank and sparse matrices from compressive
  measurements.
\newblock In \emph{Advances in neural information processing systems}, pages
  1089--1097, 2011.

\bibitem[Zhang and Zhang(2020)]{zhang2020many}
J.~Zhang and R.~Zhang.
\newblock How many samples is a good initial point worth in low-rank matrix
  recovery?
\newblock \emph{Advances in Neural Information Processing Systems}, 33, 2020.

\bibitem[Zhang et~al.(2019)Zhang, Sojoudi, and Lavaei]{Zhang_Sharp_2019}
R.~Y. Zhang, S.~Sojoudi, and J.~Lavaei.
\newblock Sharp restricted isometry bounds for the inexistence of spurious
  local minima in nonconvex matrix recovery.
\newblock \emph{Journal of Machine Learning Research}, 20:\penalty0 1--34,
  2019.

\bibitem[Zheng and Lafferty(2015)]{zheng2015convergent}
Q.~Zheng and J.~Lafferty.
\newblock A convergent gradient descent algorithm for rank minimization and
  semidefinite programming from random linear measurements.
\newblock \emph{Advances in Neural Information Processing Systems},
  28:\penalty0 109--117, 2015.

\bibitem[Zheng and Lafferty(2016)]{zheng2016convergence}
Q.~Zheng and J.~Lafferty.
\newblock Convergence analysis for rectangular matrix completion using
  {B}urer-{M}onteiro factorization and gradient descent.
\newblock \emph{arXiv preprint arXiv:1605.07051}, 2016.

\end{thebibliography}



\end{document}